\documentclass{article}
\usepackage{amsfonts,amsmath,amsthm,amssymb}

\date{}

\usepackage{algorithm}
\usepackage{algorithmic}
\usepackage{amsthm}
\usepackage{subfigure}
\usepackage{epsfig}
\usepackage{graphicx}
\usepackage{url}
\usepackage{multirow}
\usepackage{amssymb}
\usepackage{mathrsfs}
\usepackage{epstopdf}
\usepackage{multicol}
\usepackage{amsmath}
\usepackage{mathtools}
\usepackage{bm}
\usepackage{color}
\usepackage{multicol}
\usepackage{wrapfig}
\usepackage{wrapfig,lipsum,booktabs}
\usepackage{enumitem}
\usepackage{multirow}
\usepackage{natbib}

\usepackage{pifont}
\newcommand{\cmark}{\ding{51}}%
\newcommand{\xmark}{\ding{55}}%

\allowdisplaybreaks[4]

\newtheorem{theorem}{Theorem}[section]

\newtheorem{lemma}[theorem]{Lemma}
\newtheorem{corollary}[theorem]{Corollary}
\newtheorem{definition}[theorem]{Definition}
\newtheorem{assumption}[theorem]{Assumption}
\newtheorem{remark}[theorem]{Remark}

\numberwithin{equation}{section}

\theoremstyle{plain}

\theoremstyle{definition}

\title{Smoothed Gradient Method for Nonconvex Federated Stochastic Bilevel Optimization}

\author{
	Xinwen Zhang \\ Temple University \\
	 \and Peiran Yu \\ University of Texas at Arlington \\
	 \and Zhaosong Lu \\ University of Minnesota  \\
	 \and 
	Hongchang Gao\thanks{{\tt hongchang.gao@temple.edu}}  \\ Temple University 
	}
\begin{document}
\maketitle

\begin{abstract}
	In recent years, federated stochastic bilevel optimization has attracted increasing attention due to its wide range of applications in machine learning. To reduce the computational overhead associated with second-order Hessian and Jacobian matrices, several first-order methods have been proposed. However, existing methods typically impose restrictive assumptions on the lower-level function, suffer from a strong dependence on the condition number in their convergence rates, and require different learning-rate scales for variables across the upper- and lower-level problems, limiting their practical applicability and complicating hyperparameter tuning. To address these challenges, we propose a stochastic doubly smoothed gradient method for nonconvex federated stochastic bilevel optimization problems, which decouples the learning rates of upper- and lower-level variables and does not require a strongly-convex lower-level loss function. We establish rigorous theoretical guarantees for the proposed algorithm, demonstrating an improved convergence rate of $O(\kappa^{15/2}/\epsilon^5)$ and a communication complexity of $O(\kappa^{4}/\epsilon^3)$, where $\kappa$ denotes the condition number and $\epsilon$ represents the solution accuracy. Notably, these bounds exhibit significantly better dependence on the condition number $\kappa$ than those of existing methods. Extensive experiments validate the effectiveness of our algorithm.
\end{abstract}

\section{Introduction}
Federated bilevel optimization has emerged as an important paradigm for large-scale machine learning, with applications spanning model-agnostic meta-learning~\citep{finn2017model, rajeswaran2019meta}, hyperparameter optimization~\citep{feurer2019hyperparameter}, neural architecture search~\citep{liu2018darts, zhang2021idarts}, and reinforcement learning~\citep{li2024getting, shen2025principled}. In many modern scenarios, data are distributed across clients due to privacy and communication constraints, motivating the following federated stochastic bilevel optimization problem:
\begin{align}\label{eq:loss}
	& \min_{x\in \mathbb{R}^{d_x}}  \frac{1}{K}\sum_{k=1}^{K}f^{(k)}(x, y^*(x)) \quad  s.t. \quad {y}^*(x) = \arg\min_{y\in \mathbb{R}^{d_y}} \frac{1}{K}\sum_{k=1}^{K} g^{(k)}(x, y) \ , 
\end{align}
where $K$ denotes the number of clients, $f^{(k)}(x, y^*(x)) = \mathbb{E}_{\xi}[f^{(k)}(x, y^*(x); \xi^{(k)})]$ is the upper-level objective, and $g^{(k)}(x, y) = \mathbb{E}_{\xi}[g^{(k)}(x, y; \xi^{(k)})]$ is the lower-level objective on the $k$-th client.

Solving Eq.~(\ref{eq:loss}) in federated learning setting is challenging because each client's upper-level objective depends on the  global lower-level solution.
Classical federated bilevel methods~\citep{huang2022fast, gao2022convergence, tarzanagh2022fednest, huang2023achieving, li2023communicationefficient} typically rely on implicit differentiation with costly Jacobian- and Hessian-related computations, resulting in substantial computational and communication overhead. Recent works~\citep{yang2025first, zhang2025federated} alleviate this issue through fully first-order reformulations that avoid explicit second-order computations. A representative strategy~\citep{kwon2023fully, shen2023penalty,lu2026solving,lu2024first,liu2024moreau} uses a penalty-based formulation to transform the bilevel problem into an unconstrained minimax problem solvable using only first-order stochastic gradients.




Despite this progress, existing first-order federated bilevel methods have several limitations. First, they require different learning-rate scales for the upper- and lower-level variables, making hyperparameter tuning difficult and potentially destabilizing the coupled optimization dynamics. Second, their theoretical guarantees generally rely on strong convexity of the lower-level objective, excluding a broad class of nonconvex problems. Third, their convergence and communication complexities have unfavorable dependence on the condition number $\kappa$. These limitations motivate more efficient and robust first-order algorithms for federated bilevel optimization.

\begin{table*}[t] 
	\small
	\begin{center}
		\caption{{Comparison of our algorithm with existing stochastic first-order methods for bilevel optimization. NC-SC refers to the setting where the upper-level problem is nonconvex and the lower-level problem is strongly convex, while NC-PL refers to the setting where the lower-level problem satisfies the PL condition. $\epsilon$ measures the accuracy for stationary point, $\kappa$ denotes the condition number; ``$\cdot$'' indicates that the cited work does not report an explicit $\kappa$-dependence in closed form.}}
		\begin{tabular}{l|c|c|c|c}
			\toprule
			\multirow{2}{*}{\textbf{Algorithms}} & \multirow{2}{*}{\textbf{Assumption}} 
			& \textbf{Convergence} & \textbf{Communication} & \textbf{Unified} \\
			&   & \textbf{Rate} & \textbf{Complexity} & \textbf{LR Scale} \\
			\hline
			\multicolumn{5}{c}{\textbf{Single-machine Setting}} \\
			\hline
			F$^2$BA\citep{kwon2023fully} & NC-SC  & $O(\cdot/\epsilon^{7})$, $O(\cdot/\epsilon^{5})$ & -- & \xmark \\
			Prox-F$^2$BA\citep{kwon2024penalty} & NC-PL  & $O(\cdot/\epsilon^{7})$, $O(\cdot/\epsilon^{5})$ & --  & \xmark \\
			F$^2$BA\citep{chen2024finding} & NC-PL  & $O(\kappa^{12}/\epsilon^{6})$ & -- & \xmark\\
			
			\hline
			\multicolumn{5}{c}{\textbf{Federated Learning Setting}} \\
			\hline
			MemFBO~\citep{yang2025first} & NC-SC & $O(\cdot/\epsilon^7)$ & $O(\cdot/\epsilon^3)$ & \xmark \\
			FedSVRBGD-FO~\citep{zhang2025federated} & NC-SC & $O(\kappa^{10}/\epsilon^5)$ & $O(\kappa^{9}/\epsilon^4)$ & \xmark \\
			FedDS-BVRFO(Corollary~\ref{corollary}) & NC-PL & $O(\kappa^{15/2}/\epsilon^5)$ & $O(\kappa^{4}/\epsilon^3)$ & \cmark \\
			
			\bottomrule
		\end{tabular}
	\end{center}
	\label{table:comparison}
	\vspace{-10pt}
\end{table*}



To address these challenges, we revisit stochastic smoothing techniques~\citep{zhang2020single, yang2022faster, zheng2023universal}, which can stabilize optimization, enable unified learning-rate scaling, and improve condition-number dependence in stochastic minimax optimization. Recent works extend smoothing to federated stochastic minimax optimization~\citep{shen2024stochastic}, stochastic multi-level compositional minimax problems~\citep{zhang2025convergence}, and constrained bilevel problems~\citep{lu2023slm, lu2025tsp}. However, stochastic smoothing in federated bilevel optimization, particularly doubly smoothed schemes, remains largely unexplored. This raises the following question:

\textbf{Can stochastic smoothing be effectively incorporated into federated bilevel optimization to simultaneously improve theoretical guarantees and practical performance?}

However, incorporating smoothing into federated bilevel optimization is nontrivial. Unlike standard minimax problems, federated bilevel optimization involves tightly coupled upper- and lower-level variables together with a global lower-level problem. Moreover, the penalty-based first-order reformulation introduces additional interactions between the penalty and smoothing parameters, making the appropriate smoothing strategy unclear. Stochastic gradient estimation and periodic communication further complicate the algorithm design. Theoretically, establishing convergence under the relaxed PL condition while improving condition-number dependence poses additional challenges.

To address these challenges, we propose a stochastic \textit{doubly smoothed} framework for federated bilevel optimization. Our method integrates smoothing into a penalty-based first-order  reformulation for \textit{both the upper- and lower-level variables}, yielding a fully first-order algorithm for solving Eq.~(\ref{eq:loss}). This formulation enables unified learning-rate scaling without explicit Jacobian or Hessian computations. Moreover, our analysis allows lower-level objectives satisfying the PL condition with improved condition-number dependence. Our main contributions are summarized as follows:

\begin{itemize}[leftmargin=20pt, itemsep=0pt]
	\vspace{-10pt}
	\item \textbf{Algorithmic contribution:} We propose a stochastic \emph{doubly smoothed} framework for federated bilevel optimization and develop a fully first-order algorithm, termed \textbf{FedDS-BVRFO}. By smoothing both the upper- and lower-level updates, our method enables the same order of learning rates with respect to the condition number, simplifying hyperparameter tuning. To the best of our knowledge, this is the first stochastic doubly smoothed method for federated bilevel optimization with unified learning-rate scaling.
	
	\item \textbf{Theoretical contribution:} We establish convergence guarantees under the PL condition on the lower-level objective, relaxing the strong-convexity assumption of existing federated bilevel methods. Our method achieves a sample complexity of $O(\kappa^{15/2}/(K\epsilon^5))$ and  communication complexity of $O(\kappa^4/\epsilon^3)$, improving the condition-number dependence over existing fully first-order federated bilevel methods while preserving linear speedup in the number of clients.
\end{itemize}
\vspace{-5pt}
Finally, extensive experimental results validate the effectiveness of the proposed algorithm.

\section{Related Work}
\subsection{Federated and Stochastic Bilevel Optimization}
Federated stochastic bilevel optimization has recently attracted growing attention due to its applications in privacy-preserving meta-learning and hyperparameter optimization~\citep{chen2018federated, jiang2019improving, fallah2020personalized,  yang2023personalized}. Compared with classical single-machine bilevel optimization~\citep{ghadimi2018approximation, ji2021bilevel, chen2021tighter}, the federated setting introduces additional challenges due to the coupling between local upper-level objectives and a global lower-level problem. A primary line of work is based on approximate implicit differentiation techniques, such as Neumann-series-based estimators, leading to representative methods including FedNEST and its variants~\citep{tarzanagh2022fednest, huang2023achieving, yang2023simfbo}. These approaches are further improved by incorporating momentum or variance-reduction strategies~\citep{gao2022convergence, li2023communicationefficient} to enhance communication efficiency and training stability. More recently, fully first-order methods have been proposed via penalty-based Lagrangian reformulations~\citep{yang2025first, zhang2025federated}, which enable gradient-based updates without explicit second-order computations. In addition, several works have investigated nonconvex lower-level problems, primarily in the single-machine setting~\citep{shen2023penalty, kwon2024penalty, chen2024finding,  liu2024moreau,lu2023slm,lu2025tsp}. These studies relax strong convexity assumptions on the lower-level problem, but their extension to federated settings remains largely unexplored.

\subsection{Smoothed Techniques}
Stochastic smoothing has been widely studied for stabilizing optimization dynamics and improving convergence behavior. \cite{zhang2020single} first introduced it into stochastic minimax optimization to mitigate gradient oscillations and ensure convergence to stationary solutions. Under the PL condition,~\cite{yang2022faster} further showed that smoothing improves the condition-number dependence from $O(\kappa^3)$ to $O(\kappa^2)$. Subsequent works extend smoothing to more complex settings. In particular,~\cite{zheng2023universal} proposed doubly smoothed schemes for nonconvex–nonconcave problems, while~\cite{shen2024stochastic} incorporated smoothing into federated stochastic minimax optimization, achieving an $O(\kappa^2)$ dependence. More recently,~\cite{zhang2025convergence} extended smoothing to stochastic multi-level compositional minimax problems, achieving an improved $O(\kappa^{3/2})$ dependence when combined with variance reduction. Beyond minimax optimization,~\cite{lu2023slm,lu2025tsp} incorporated smoothing into Lagrangian methods for constrained bilevel optimization. Note that our reformulation differs from those in recent works~\citep{lu2023slm,lu2025tsp}. Specifically, \cite{lu2023slm,lu2025tsp} apply the smoothing technique only for the primal variable, while our method apply it to both the primal and dual variables in the reformulated minimax problem.  Despite these advances, the role of stochastic smoothing in federated bilevel optimization, particularly the use of doubly smoothed schemes, remains largely unexplored.

\section{Method}\label{sec:method}
\subsection{Doubly Smoothed Bilevel Objective}
Following the fully first-order reformulation in \cite{yang2025first, zhang2025federated}, we introduce a penalty parameter $\rho > 0$ to transform the federated bilevel problem in Eq.~(\ref{eq:loss}) into the following unconstrained minimax formulation:
\begin{align}\label{eq:first_loss}
	&  \min_{x,y} \max_{z} \underbrace{\frac{1}{K}\sum_{k=1}^{K}f_{\rho}^{(k)} (x, y, z)}_{f_{\rho} (x, y, z)} \ ,  \quad f_{\rho}^{(k)} (x, y, z) =  f^{(k)} (x, y) + \frac{1}{\rho}(g^{(k)}(x, y)-g^{(k)}(x, z))    \ .
\end{align}
At optimality, the maximizer with respect to $z \in \mathbb{R}^{d_y}$ coincides with the lower-level solution $y^*(x)$. Importantly, the gradients with respect to $x$, $y$, and $z$ involve only first-order information, thereby avoiding explicit Jacobian and Hessian computations and enabling more efficient optimization in federated learning settings.

Motivated by smoothing techniques in classical stochastic minimax optimization \citep{zhang2020single, yang2022faster, zheng2023universal}, we introduce auxiliary proximal variables $p\in \mathbb{R}^{d_x}$ and $q\in \mathbb{R}^{d_y}$ for the  primal updates $x$ and $y$, respectively, and an auxiliary proximal variables $r\in \mathbb{R}^{d_y}$ for the  dual maximizer $z$. We then  construct the following doubly smoothed objective:
\begin{align}
	& F(x, y, z ; p, q, r)  = \rho \frac{1}{K}\sum_{k=1}^{K}f_{\rho}^{(k)} (x, y, z)   + \underbrace{\frac{\gamma_1}{2}(\| x - p \|^2+ \| y - q\|^2)}_{\textbf{primal smoothing on } (x, y)}  -\underbrace{\frac{\gamma_2}{2} \|z - r\|^2}_{\textbf{dual smoothing on }z} \ , 
\end{align}
where $\gamma_1 > 0$ and $\gamma_2 > 0$ are the primal and dual smoothing parameters.

The term $\rho f_{\rho}^{(k)}(x, y, z)$ is $L_{\rho}$-smooth with respect to $x$ and $y$, where $L_{\rho} \triangleq \rho L_f + 2L_g$. By choosing $\gamma_1 = O(L_{\rho})$ with $\gamma_1 > L_{\rho}$, the function $F(x, y, z; p, q, r)$ is $(\gamma_1 - L_{\rho})$-strongly convex and $(\gamma_1 + L_{\rho})$-smooth in $x$ and $y$. In particular, the strong convexity and smoothness constants are given by $\mu_{F} = (\gamma_1 - L_{\rho})$ and $L_F = (\gamma_1 + L_{\rho})$, respectively. Similarly, by smoothing the dual variable $z$ with $\gamma_2>L_{\rho}$, $F(x,y,z;p,q,r)$ is $(\gamma_2-L_{\rho})$-strongly concave and $(\gamma_2+L_{\rho})$-smooth with respect to $z$.

By rearranging the above expression, the doubly smoothed objective can be written in the following decomposed form:
\begin{align}\label{eq:smooth_loss}
	&  F(x, y, z; p, q, r) = \frac{1}{K}\sum_{k=1}^{K} \Big( \rho f^{(k)} (x, y)  + g_{\gamma_1}^{(k)} (x, y; p, q) - g_{\gamma_2}^{(k)}(x, z; r) \Big) \ , 
\end{align}
where the smoothed lower-level component is defined as
\begin{align}
	& g_{\gamma_1}^{(k)} (x, y; p, q)  = g^{(k)} (x, y)  + \frac{\gamma_1}{2}(\| x - p \|^2 + \| y - q\|^2)  \ , \notag \\
	& g_{\gamma_2}^{(k)} (x, z; r)  = g^{(k)} (x, z)  + \frac{\gamma_2}{2}\| z - r \|^2 \ . 
\end{align}
Here, $g_{\gamma_1}^{(k)}$ smooths the primal minimization updates $(x,y)$, whereas $g_{\gamma_2}^{(k)}$ smooths the dual maximization update $z$. This simultaneous primal-dual smoothing is the key feature of our doubly smoothed formulation, and yields a well-conditioned strongly-convex-strongly-concave surrogate that supports unified learning-rate scaling.


\begin{algorithm}[ht]
	\caption{FedDS-BVRFO}
	\label{alg:fedss-bvrfo}
	\begin{algorithmic}[1]
		\REQUIRE $\alpha_x>0$, $\alpha_y>0$, $\alpha_z>0$, $\alpha_p>0$, $\alpha_q>0$, $\alpha_r>0$,
		$\beta_{x}>0$, $\beta_{y}>0$, $\beta_{z}>0$, $\eta>0$.  \\
		Initialization on each worker $k\in\{1, \cdots, K\}$: 
		${x}_{0}^{(k)}={x}_{0}$, ${y}_{0}^{(k)}={y}_{0}$, ${z}_{0}^{(k)}={z}_{0}$, 
		\FOR{$t=0,\cdots, T-1$} 
		\STATE Update variables $x$, $y$, and $z$ via Eq.~(\ref{eq:xyz_update})
		
		\STATE Update variables $p$, $q$, and $r$ via Eq.~(\ref{eq:pq_update})
		
		\STATE Update gradients $u$, $v$, and $w$ via Eq.~(\ref{eq:uv_update})
		\IF {$\text{mod}(t+1, \tau)==0$}
		\STATE Communication via Eq.~(\ref{eq:comm})
		\ENDIF
		\ENDFOR
	\end{algorithmic}
\end{algorithm}

\subsection{Proposed Algorithm}
To solve Eq.~(\ref{eq:smooth_loss}), we propose a new algorithm termed Federated doubly smoothed  Bilevel Variance-Reduced First-Order (FedDS-BVRFO), summarized in Algorithm~\ref{alg:fedss-bvrfo}. 
The full algorithmic description is provided in the Appendix~\ref{app:alg}.

For the smoothed objective in Eq.~(\ref{eq:smooth_loss}), the partial gradients are given by
\begin{align}\label{eq:partial_grad}
	&  \nabla_{x} F(x_{t}, y_{t}, z_{t}; p_{t}, q_{t}, r_{t}) =  \rho\nabla_1 f(x_{t}, y_{t})  + (\nabla_1 g_{\gamma_1}(x_{t}, y_{t}; p_{t}, q_{t}) - \nabla_1 g_{\gamma_2}(x_{t}, z_{t}; r_{t})) \ , \notag \\
	& \nabla_{y} F(x_{t}, y_{t}, z_{t}; p_{t}, q_{t}, r_{t}) = \rho\nabla_2 f(x_{t}, y_{t}) + \nabla_2 g_{\gamma_1}(x_{t}, y_{t}; p_{t}, q_{t}) \ , \notag \\
	& \nabla_{z} F(x_{t}, y_{t}, z_{t}; p_{t}, q_{t}, r_{t}) = -\nabla_2 g_{\gamma_2}(x_{t}, z_{t}; r_{t}) \ ,
\end{align}
where
\begin{align}
	& \nabla_1 g_{\gamma_1}(x_{t}, y_{t}; p_{t}, q_{t}) = \nabla_1 g(x_{t}, y_{t}) + \gamma_1(x_{t} - p_{t}) \ ,  \notag \\
	& \nabla_2 g_{\gamma_1}(x_{t}, y_{t}; p_{t}, q_{t}) = \nabla_2 g(x_{t}, y_{t}) + \gamma_1(y_{t} - q_{t})  \ , \notag \\
	& \nabla_1 g_{\gamma_2}(x_{t}, z_{t}; r_{t}) = \nabla_1 g(x_{t}, z_{t}) \ , \notag \\
	& \nabla_2 g_{\gamma_2}(x_{t}, z_{t}; r_{t}) = \nabla_2 g(x_{t}, z_{t}) + \gamma_2 (z_{t} - r_{t}) \ . 
\end{align}
To construct variance-reduced estimators, we decompose the gradient into multiple components and maintain separate momentum variables. Specifically, $u_{1, t}$, $u_{2, t}$, and $u_{3, t}$ track the three components of the $x$-gradient, corresponding to $\nabla_1 f(x_{t}, y_{t})$, $\nabla_1 g_{\gamma}(x_{t}, y_{t}, p_{t}, q_{t})$ and $\nabla_1 g(x_{t}, z_{t})$, respectively. Similarly, $v_{1, t} $ and $v_{2, t}$ track the two components of the $y$-gradient, while $w_{t}$ tracks the $z$-gradient.

On each worker $k$, we adopt STORM-style variance-reduced updates~\citep{cutkosky2019momentum} for each gradient component, including $u^{(k)}_{1, t+1}$, $u^{(k)}_{2, t+1}$, $u^{(k)}_{3, t+1}$, $v^{(k)}_{1, t+1}$, $v^{(k)}_{2, t+1}$, $w^{(k)}_{ t+1}$. For illustration, we present the update for $u^{(k)}_{1, t+1}$ as follows:
\begin{align}
	& u^{(k)}_{1, t+1} =  (1-\beta_x\eta^2)(u^{(k)}_{1, t} - \nabla_1 f^{(k)}(x^{(k)}_{t}, y^{(k)}_{t}; \xi^{(k)}_{t+1})) + \nabla_1 f^{(k)}(x^{(k)}_{t+1}, y^{(k)}_{t+1}; \xi^{(k)}_{t+1}) \ . 
\end{align}
Thus, the final variance-reduced momentum estimators for the $x$-, $y$-, and $z$-updates are constructed as:
\begin{align}\label{eq:uv_update}
	& u^{(k)}_{t+1} = \rho u^{(k)}_{1, t+1} +  (u^{(k)}_{2, t+1}  - u^{(k)}_{3, t+1}) \ , \quad  v^{(k)}_{t+1}  =  \rho v^{(k)}_{1, t+1}   +  v^{(k)}_{2, t+1} \ ,  \quad  w^{(k)}_{t+1}  =  w^{(k)}_{t+1} \ .  
\end{align}
After obtaining these estimators, the local model variables are updated as 
\begin{align}\label{eq:xyz_update}
	&	x^{(k)}_{t+1} = x^{(k)}_{t} -  \alpha_{x} \eta u^{(k)}_{t}  \ , \quad  y^{(k)}_{t+1} = y^{(k)}_{t} -  \alpha_{y} \eta v^{(k)}_{t} \ ,  \quad  z^{(k)}_{t+1} = z^{(k)}_{t} -  \alpha_{z} \eta w^{(k)}_{t} \ .  
\end{align}
In addition, the auxiliary smoothing variables $p^{(k)}_{t+1}$, $q^{(k)}_{t+1}$, and $r^{(k)}_{t+1}$ are updated as:
\begin{align}\label{eq:pq_update}
	& p^{(k)}_{t+1} = p^{(k)}_{t} + \alpha_{p} \eta (x^{(k)}_{t+1} - p^{(k)}_{t}) \ , \quad  q^{(k)}_{t+1} = q^{(k)}_{t} + \alpha_{q} \eta (y^{(k)}_{t+1} - q^{(k)}_{t}) \ ,  \notag \\
	& r^{(k)}_{t+1} = r^{(k)}_{t} + \alpha_{r} \eta (z^{(k)}_{t+1} - r^{(k)}_{t}) \ . 
\end{align}
To improve communication efficiency, we adopt a periodic synchronization strategy. Specifically, every $\tau$ iterations, each local variable is replaced by the average across all participating workers. This includes the model variables $x^{(k)}_{t+1}$, $y^{(k)}_{t+1}$, $z^{(k)}_{t+1}$; the auxiliary variables $p^{(k)}_{t+1}$, $q^{(k)}_{ t+1}$, and $r^{(k)}_{ t+1}$; and the variance-reduced momentum estimators $u^{(k)}_{t+1}$, $v^{(k)}_{t+1}$,  $w^{(k)}_{t+1}$. Specifically, for any $S\in\{x,y,z,p,q,r,u,v,w\}$, the synchronization step is given by
\begin{align}\label{eq:comm}
	& S^{(k)}_{t+1}=\frac{1}{K}\sum_{k'=1}^{K}S^{(k')}_{ t+1}  \ .    
\end{align}

\section{Theoretical Analysis}
\subsection{Assumptions}\label{sec:assumption}
We adopt the following standard assumptions commonly used in federated stochastic bilevel optimization~\citep{kwon2023fully, kwon2024penalty, chen2024finding, zhang2025federated, yang2025first}.
\begin{assumption} \label{assumption:f-smooth}
	For $k\in\{1, \cdots, K\}$,  $f^{(k)}(\cdot, \cdot)$ satisfies the mean-squared smoothness with the constant $L_{f}>0$ and its second-order gradient  satisfies the Lipschitz continuity with the constant $\ell_{f}>0$.  In addition,  the partial gradient with respect to the second argument is bounded as  $\|\nabla_2 f^{(k)}(\cdot, \cdot)\|\leq C_f$, where $C_f>0$ is a constant. 
\end{assumption}

\begin{assumption}\label{assumption:g-smooth}
	For $k\in\{1, \cdots, K\}$, $g^{(k)}(\cdot, \cdot)$ satisfies the mean-squared smoothness with the constant $L_{g}>0$, and its second-order gradient satisfies the Lipschitz continuity with the constant $\ell_{g}>0$. 
\end{assumption}

\begin{assumption}\label{assumption:g-pl}
	For $k\in\{1, \cdots, K\}$, $g^{(k)}(\cdot, \cdot)$ satisfies the $\mu$-PL condition with respect to the second argument with the constant $\mu>0$.  In addition,  $h_{\rho}(\cdot, \cdot)= \rho f(\cdot, \cdot) + g(\cdot, \cdot)$ also satisfies the $\mu$-PL condition with respect to the second argument. 
\end{assumption}

Note that most federated bilevel optimization methods in the literature~\citep{tarzanagh2022fednest, huang2023achieving, gao2022convergence, li2023communicationefficient, zhang2025federated, yang2025first} rely on the strong convexity of $g^{(k)}(\cdot, \cdot)$ with respect to the second argument. In contrast, our federated analysis assumes only a $\mu$-PL condition on the lower-level objective, which has recently been studied in single-machine bilevel optimization. Since the PL condition is strictly weaker than strong convexity, our framework applies to a broader range of practical scenarios.

\begin{assumption}\label{assumption:variance}
	For $k\in\{1, \cdots, K\}$, the stochastic gradients of  $f^{(k)}(\cdot, \cdot)$  and $g^{(k)}(\cdot, \cdot)$  have variances upper bounded by $\sigma^2$ where  $\sigma>0$ is a constant.
\end{assumption}

\begin{assumption} \label{assumption:heterogeneous}
	The data distribution between workers is heterogeneous and the upper-level and lower-level gradients satisfy the following conditions:
	{\small\begin{align}
			& \frac{1}{K}\sum_{k=1}^{K}\|\nabla  f^{(k)}(x, y) - \frac{1}{K}\sum_{k'=1}^{K}\nabla  f^{(k')}(x, y) \|^2 \leq \delta^2  \ , \notag \\
			& \frac{1}{K}\sum_{k=1}^{K}\|\nabla  g^{(k)}(x, y) - \frac{1}{K}\sum_{k'=1}^{K}\nabla  g^{(k')}(x, y) \|^2 \leq \delta^2 \  , 
	\end{align}}
	where $\delta>0$ is a constant, $x\in \mathbb{R}^{d_x}$, and $y\in \mathbb{R}^{d_y}$. 
\end{assumption}

\begin{definition}
	Given Assumptions~\ref{assumption:f-smooth}-\ref{assumption:g-pl}, we define $\ell=\max\{C_f, L_f, \ell_{f}, L_{g}, \ell_{g}\}$ and the condition number $\kappa=\ell/\mu$. 
\end{definition}

\begin{definition} \label{def:measures}\citep{yang2022faster}
	$(\alpha, \theta)$ is an $(\epsilon_1, \epsilon_2)$-stationary point of $f(\cdot, \cdot)$, if $\|\nabla_{\alpha} f(\alpha, \theta)\|\leq \epsilon_1$ and $\|\nabla_{\theta} f(\alpha, \theta)\|\leq \epsilon_2$. $\alpha$ is an $\epsilon$-stationary point of $\Phi$ if $\|\nabla \Phi(\alpha)\|\leq \epsilon$, where $\Phi(\alpha) = f(\alpha, \theta^*(\alpha))$ and $\theta^*(\alpha) =\arg\max_{\theta\in\mathbb{R}^{d_2}} f(\alpha, \theta)$. 
\end{definition}

\subsection{Convergence Rate}
Based on Assumptions~\ref{assumption:f-smooth}-\ref{assumption:heterogeneous}, we establish the convergence rate of our algorithm below:
\begin{theorem}\label{theorem:main}
	Given Assumptions~\ref{assumption:f-smooth}-\ref{assumption:heterogeneous}, when $\beta_x>0$, $\beta_y>0$, $\beta_z>0$, $\gamma_1=O(L_{\rho})$, $\gamma_2=O(L_{\rho})$, the condition about $\alpha_x$ in Eq.~(\ref{eq:upper-bound-beta-x}), and
	{\small\begin{align}
			& \alpha_y = \alpha_x \ , \quad  \alpha_z = \alpha_x\underbrace{\frac{(\gamma_1 - L_{\rho})^2}{72L^2_{g}}}_{c_{\alpha_z}=O(1)} \ , \quad \alpha_{p} = \alpha_{q} = \alpha_{x}\underbrace{\frac{(\gamma_1-L_{\rho})^4(\gamma_2-L_\rho)^2\mu}{72\times768\gamma_1L_{g}^2(5L_g^2\mu+16\gamma_2^2(\gamma_1-L_\rho))}}_{c_{\alpha_{pq}} = O(1/\kappa)} \ , \notag \\
			& \alpha_r = \alpha_{x}\underbrace{\frac{(\gamma_1-L_{\rho})^4(\gamma_2-L_\rho)^4}{72\times24\times640\gamma_1\gamma_2^2 L_g^4}}_{{c_{\alpha_r}=O(1)}}  \ , \notag
	\end{align}}
	Algorithm~\ref{alg:fedss-bvrfo} has the following convergence upper bound:
	{\small\begin{align}
			&  \quad  \frac{1}{T}\sum_{t=0}^{T-1} \Bigg(  \mathbb{E}[\| \nabla_{x} f_{\rho}(\bar{x}_{t}, \bar{y}_{t}, \bar{z}_{t}) \|^2] + \mathbb{E}[\|\nabla_{y}   f_{\rho}(\bar{x}_{t}, \bar{y}_{t}, \bar{z}_{t}) \|^2] +  \kappa\mathbb{E}[\|\nabla_{z}  f_{\rho}(\bar{x}_{t}, \bar{y}_{t}, \bar{z}_{t}) \|^2 ] \Bigg) \notag \\
			& \leq O\Big(\frac{  \kappa}{\alpha_{x}\eta \rho T}\Big) + O\Big( \frac{\kappa}{\beta \eta^2\rho^2 BKT }\Big) + O\Big( \frac{\kappa\beta\eta^2}{\rho^2K}\Big) + O\Big(\frac{\kappa\tau^2\beta\eta^4\alpha_{x}^2 }{\rho^2K}\Big) + O\Big(\frac{\kappa\tau^4\alpha_{x}^2\beta^2\eta^6}{\rho^2}\Big)  \ .
	\end{align}}
	
\end{theorem}
Note that the convergence rate is established with respect to the original federated bilevel objective in Eq.~(\ref{eq:first_loss}), rather than the smoothed objective in Eq.~(\ref{eq:smooth_loss}).

\begin{corollary}\label{corollary}
	Given Assumptions~\ref{assumption:f-smooth}-\ref{assumption:heterogeneous}, by setting
	{\small\begin{align}
			& \rho = O\left(\frac{\epsilon}{\kappa^3}\right) \ , \quad  \alpha_x = O\left(1\right)  \ , \quad  \beta_x = O\left(\frac{1}{K}\right) \ , \quad \beta_y = O\left(\frac{1}{K}\right) \ , \quad  \beta_z = O\left(\frac{1}{K}\right) \ , \notag \\
			& \eta = O\left(\frac{K\epsilon^2}{\kappa^{7/2}}\right)\ , \quad  \tau =  O\left(\frac{\kappa^{7/2}}{K\epsilon^2}\right)  \ , \quad B = O\left(\frac{\kappa^{13/2}}{K\epsilon^3}\right) \ , \quad T = O\left(\frac{\kappa^{15/2}}{K\epsilon^5}\right) \ .  
	\end{align}}
	we obtain an $(\epsilon, \epsilon/\sqrt{\kappa})$-stationary point, in the sense that 
	{\small\begin{align}
			&    \frac{1}{T}\sum_{t=0}^{T-1} \Bigg(  \rho^2\mathbb{E}[\| \nabla_{x} f_{\rho}(\bar{x}_{t}, \bar{y}_{t}, \bar{z}_{t}) \|^2] +  \rho^2\mathbb{E}[\|\nabla_{y}   f_{\rho}(\bar{x}_{t}, \bar{y}_{t}, \bar{z}_{t}) \|^2] +  \kappa\rho^2\mathbb{E}[\|\nabla_{z}  f_{\rho}(\bar{x}_{t}, \bar{y}_{t}, \bar{z}_{t}) \|^2 ] \Bigg) \leq O(\epsilon^2) \ .  \notag 
	\end{align}}
\end{corollary}
Note that $B$ denotes only the initial batch size used in the first iteration; throughout the subsequent training process, the algorithm operates with an $O(1)$ batch size.

\begin{remark}\label{remark:tran}
	\textbf{(Stationarity translation.)} {Following~\cite{yang2022faster}, we need to translate the $(\epsilon, \epsilon/\sqrt{\kappa})$-stationary solution to the $\epsilon$-stationary solution. The complete translation procedure and analysis are provided in Appendix~\ref{app:translation}. Importantly, after this translation, the iteration complexity to find the $\epsilon$-stationary solution is still $T=O\left(\frac{\kappa^{15/2}}{K\epsilon^5}\right)$. }
\end{remark}

\begin{remark}
	\textbf{(Better convergence rate under milder assumption.)} The iteration complexity $T = O\left(\frac{\kappa^{15/2}}{K\epsilon^5}\right) $ indicates that the proposed method enjoys linear speedup with respect to the number of clients $K$. Compared with existing first-order federated bilevel methods, our result improves upon the $O(\frac{1}{K\epsilon^7})$ complexity established in~\cite{yang2025first} and matches the optimal $O(\frac{1}{K\epsilon^5})$ dependence reported in \cite{zhang2025federated}. In terms of condition-number dependence $\kappa$, our method achieves a complexity of $O(\kappa^{15/2})$, which significantly improves over the $O(\kappa^{10})$ bounds in \cite{zhang2025federated}, and $O(\kappa^{12})$ bounds in \cite{chen2024finding}, respectively. Moreover, this result is established under a milder assumption in federated bilevel optimization, requiring only a PL condition on the lower-level function rather than assuming strong convexity as in~\cite{yang2025first,zhang2025federated}.
\end{remark}

\begin{remark}
	\textbf{(Better communication complexity.)}
	Regarding communication efficiency, Algorithm~\ref{alg:fedss-bvrfo} requires $T/ \tau = O\left( \kappa^{4}/\epsilon^3\right)$ communication rounds. This significantly improves upon the $O\left( \kappa^{9}/\epsilon^4\right)$ communication complexity reported in~\cite{zhang2025federated}, in terms of both the condition-number dependence $\kappa$ and accuracy $\epsilon$.
\end{remark}

\begin{remark}
	\textbf{(Unified learning rate scale.)}
	{Existing first-order federated bilevel methods \citep{yang2025first, zhang2025federated} typically require multi-scale learning rates for different variables in the upper- and lower-level problems, with the step sizes tightly coupled by the condition number $\kappa$. For instance, in \cite{zhang2025federated}, the prescribed learning-rate ratios between $x$ and $y$, as well as between $x$ and $z$, are on the order of $O(1 / \kappa^2)$. By contrast, the proposed smoothed formulation enables the updates of $x$, $y$, and $z$ to operate under the same order of step sizes with respect to $\kappa$. This unified scaling substantially simplifies hyperparameter tuning and improves practical robustness.}
\end{remark}

The proof sketch and full proof are provided in Appendices~\ref{app:sketch} --~\ref{app:proof}.

\section{Experiments}\label{sec:exp}
\subsection{Nonconvex Hyperparameter Optimization}
We consider a nonconvex hyperparameter optimization problem with objective defined in Eq.~(\ref{eq:exp-hyperparameter-opt}). In this bilevel formulation, the lower-level problem trains a two-layer fully connected neural network, while the upper-level problem tunes the hyperparameters. Although the lower-level objective is nonconvex, prior work shows that overparameterized neural networks can satisfy the PL condition~\citep{liu2022loss}. Specifically, the lower-level problem optimizes network parameters $y = \{y_1, y_2\}$ using training data across clients, whereas the upper-level problem updates hyperparameters $x = \{x_1, x_2\}$ based on validation loss. The resulting bilevel formulation is shown as follows:
\begin{align}\label{eq:exp-hyperparameter-opt}
	& \qquad \quad \min_{x=\{x_1, x_2\}} \quad \frac{1}{K} \sum_{k=1}^{K} \ \mathcal{L}^{(k)}(y^*(x); \mathcal{D}^{(k)}_{val}) \notag  \\
	& s.t. \quad   y^*(x) =  \arg\min_{y=\{y_1, y_2\}}\frac{1}{K} \sum_{k=1}^{K} \mathcal{L}^{(k)}(y; \mathcal{D}^{(k)}_{tr})  + \mathcal{R}_1(x)  + \mathcal{R}_2(x)\ . 
\end{align}
Here, $y_1=[y_{1, pq}]\in\mathbb{R}^{d_1\times d_2}$ and $y_2=[y_{2, pq}]\in\mathbb{R}^{d_2\times d_3}$ denote the weight matrices of the first and second network layers. The hyperparameters $x_1=[x_{1, q}]\in\mathbb{R}^{d_2}$ and $x_2=[x_{2, q}]\in\mathbb{R}^{d_3}$ control the regularization terms: $\mathcal{R}_1(x)=\frac{1}{d_2}\sum_{q=1}^{d_2} \exp(x_{1, q})\frac{1}{d_1}\sum_{p=1}^{d_1}y_{1, pq}^2$ and  $\mathcal{R}_2(x)= \frac{1}{d_3}\sum_{q=1}^{d_3} \exp(x_{2, q})\frac{1}{d_2}\sum_{p=1}^{d_2}y_{2, pq}^2$. In our experiments, $d_1$ is set to the number of input features, $d_2=20$, and $d_3=1$ for binary classification. 

Since the lower-level objective in Eq.~(\ref{eq:exp-hyperparameter-opt}) is nonconvex, existing federated bilevel methods do not have convergence guarantees and cannot directly apply to this setting. In particular, both second-order methods~\citep{tarzanagh2022fednest, huang2023achieving, gao2022convergence, li2023communicationefficient} and fully first-order methods~\citep{zhang2025federated, yang2025first} rely critically on the strong convexity assumption for the lower-level problem. To enable a fair and meaningful comparison, we therefore benchmark against algorithms specifically designed for nonconvex bilevel optimization, most notably F$^2$BSA~\citep{kwon2024penalty}, adapted to the federated environment. We evaluate both the double-loop and single-loop variants. For the double-loop version, we set the number of inner iterations to 1, 5, and 10, denoted as \textbf{Fed-F$^2$BSA-1}, \textbf{Fed-F$^2$BSA-5}, and \textbf{Fed-F$^2$BSA-10}, respectively. We also include the variance-reduced single-loop variant, denoted as \textbf{Fed-F$^2$BSA-VR}, which incorporates a STORM-type estimator. All baseline learning rates follow the guidelines in Corollaries 5.2 and 5.5 of~\cite{kwon2024penalty}, which require different learning-rate scales across variables.

\begin{figure*}[ht]
	\vspace{-5pt}
	\begin{center}
		\centerline{\includegraphics[scale=0.25]{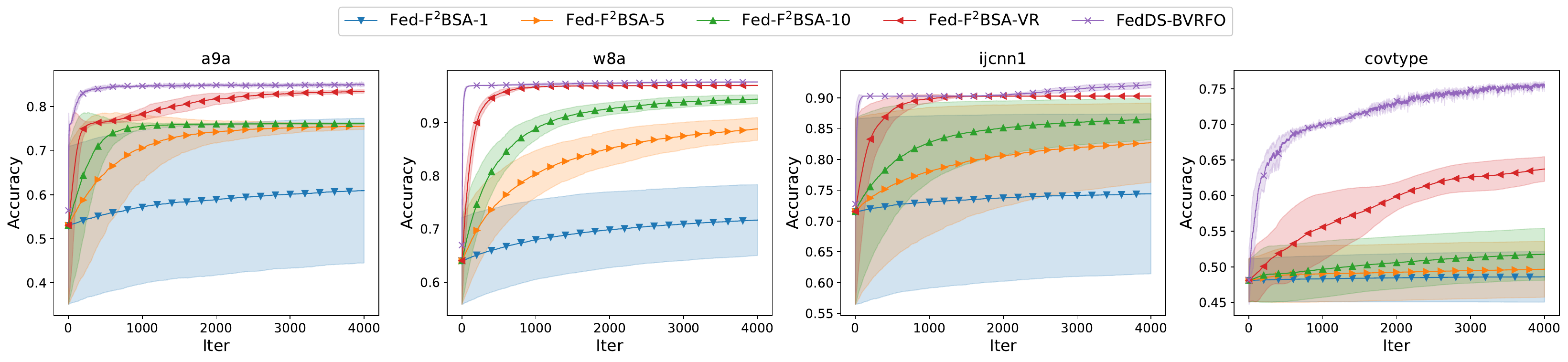}}
		\vspace{-0.1in}
		\caption{The test accuracy versus
			the training epochs, $\tau = 4$, nonconvex. }
		\label{fig:p_4_acc}
	\end{center}
	\vspace{-5pt}
\end{figure*}

\begin{figure*}[ht]
	\vspace{-0.25in}
	\begin{center}
		\centerline{\includegraphics[scale=0.25]{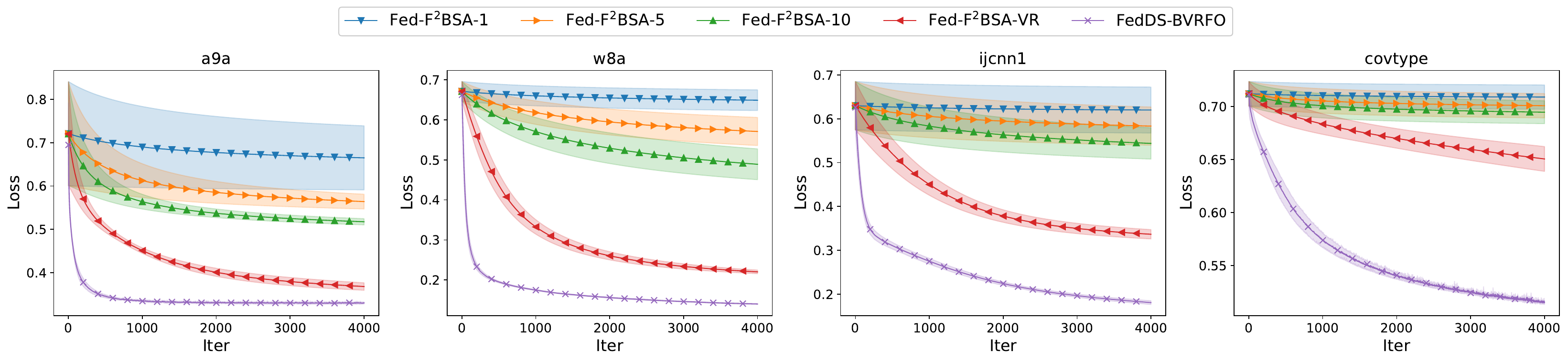}}
		\vspace{-5pt}
		\caption{The upper-level loss
			function value versus
			the training epochs, $\tau = 4$, nonconvex.}
		\label{fig:p_4_loss}
	\end{center}
	\vspace{-15pt}
\end{figure*}

We evaluate our method on four standard LIBSVM datasets: a9a, w8a, ijcnn1, and covtype\footnote{\url{https://www.csie.ntu.edu.tw/~cjlin/libsvmtools/datasets/}}. For each dataset, we randomly reserve $10\%$ of the samples for testing, and split the remaining data into training and validation sets with a $70\%/ 30\%$ ratio. To simulate the federated setting, both training and validation data are uniformly distributed across eight workers. Each worker uses a mini-batch size of $10$. For our method, we set $\epsilon = 0.1$, $\rho = \epsilon / 5$. The learning rates are chosen according to Corollary~\ref{corollary}, using a unified scale for both upper- and lower-level variables. For all methods, we conduct experiments over five random seeds and report the mean and standard deviation.

When the communication period is set to $\tau = 4$, the results in Figure~\ref{fig:p_4_acc} and Figure~\ref{fig:p_4_loss} show that the proposed method consistently achieves higher test accuracy and lower upper-level objective values than all baselines, indicating faster and more stable convergence. Moreover, its smaller standard deviations across five random seeds indicate reduced sensitivity to stochastic variation. We also observe that the double-loop variants of Fed-F$^2$BSA are sensitive to the number of inner iterations: increasing the inner-loop length improves performance but incurs higher computational cost. Although the variance-reduced single-loop variant Fed-F$^2$BSA-VR alleviates this issue, it remains inferior to our method. These results highlight the benefit of the proposed method in achieving efficient and stable convergence without additional inner-loop updates. Additional results for $\tau=16$ are provided in Appendix~\ref{app:exp}.

\begin{figure*}[!h]
	\centering 
	\vspace{-5pt}
	\hspace{-16pt}
	\subfigure[Different $\eta$]{\label{fig:diff_eta}
		\includegraphics[scale=0.225]{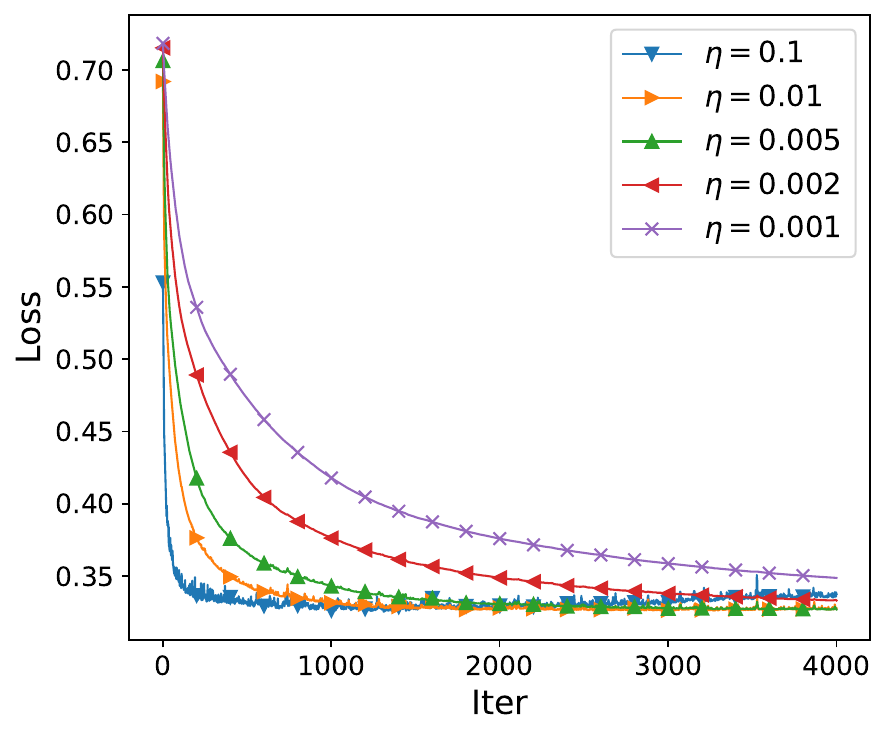}
	}
	\hspace{-5pt}
	\subfigure[Different $\rho$]{
		\label{fig:diff_rho}
		\includegraphics[scale=0.225]{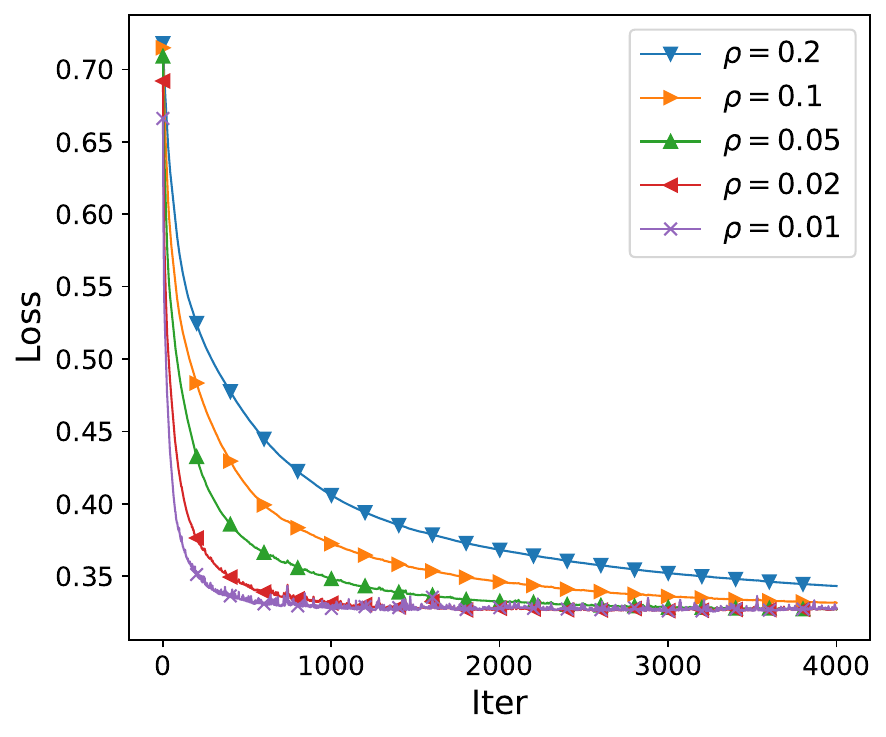}
	}
	\hspace{-5pt}
	\subfigure[Different $\gamma$]{
		\label{fig:diff_gamma}
		\includegraphics[scale=0.225]{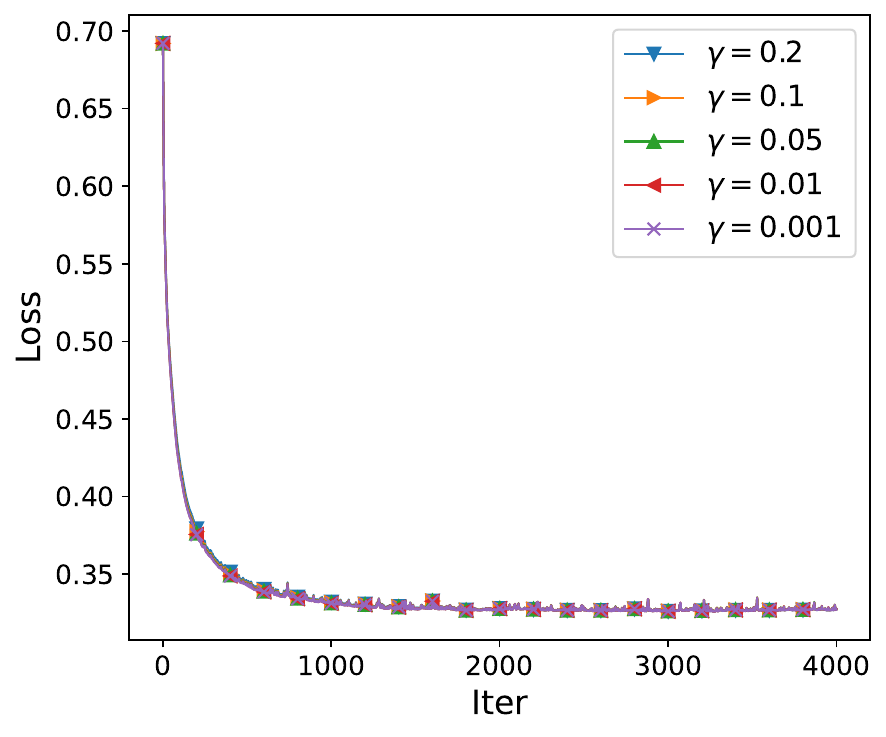}
	} 
	\hspace{-5pt}
	\subfigure[Different $\beta\eta^2$]{
		\label{fig:diff_beta}
		\includegraphics[scale=0.225]{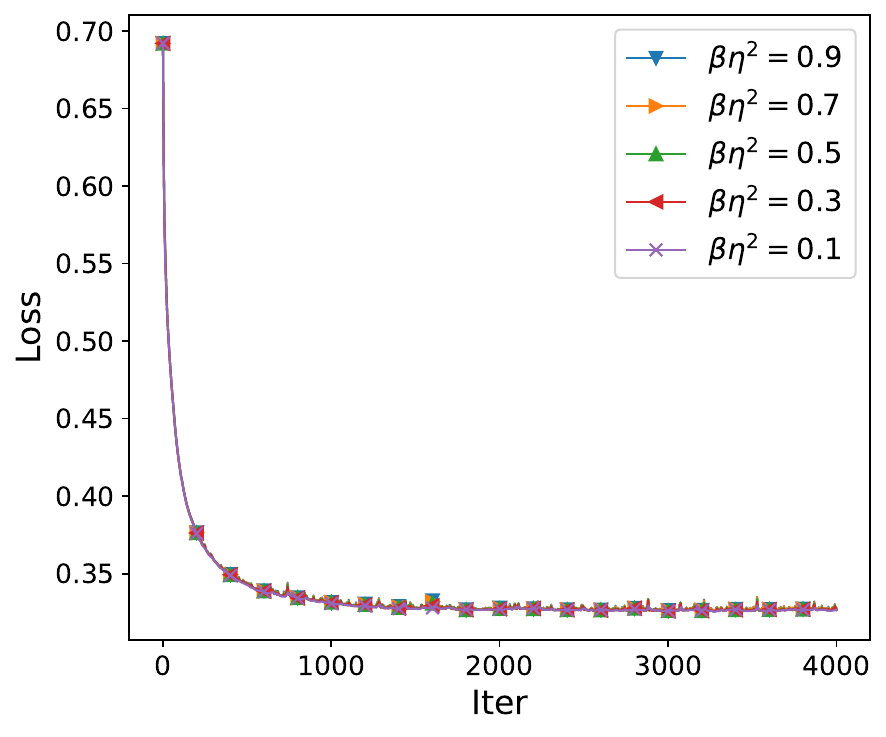}
	} 
	\vspace{-0.1in}
	\caption{The upper-level loss
		function value versus
		the training epochs on a9a, $\tau = 16$, nonconvex.}
	\label{fig:diff}
	\vspace{-0.1in}
\end{figure*}

We further study the sensitivity of our method to several hyperparameters, including the learning rate $\eta$, penalty parameter $\rho$, smoothing parameter $\gamma: = \gamma_1 = \gamma_2$, and the momentum coefficient $\beta \eta^2$, where $\beta := \beta_x = \beta_y = \beta_z$, as shown in Figure~\ref{fig:diff}. Figure~\ref{fig:diff_eta} indicates that a larger $\eta$ accelerates convergence but may cause instability, while Figure~\ref{fig:diff_rho} shows that a smaller $\rho$ (i.e., a larger penalty $1/\rho$) improves convergence speed. Moreover, Figures~\ref{fig:diff_gamma} and~\ref{fig:diff_beta} indicate that our method is relatively robust to $\gamma$ and $\beta$, suggesting that the proposed smoothing and momentum parameters do not require extensive tuning in practice.

Finally, additional comparisons with federated bilevel baselines on nonconvex–strongly-convex tasks are reported in Appendix~\ref{app:exp}.

\subsection{Model Pruning}
To provide additional empirical evidence, we evaluate our method on a structured parameter reduction problem, namely model pruning \citep{zhang2022advancing}. This task aims to identify a sparse subnetwork by selectively disabling redundant weights via a binary selection mask. In this framework, the sparsity pattern and the parameter retraining process are coupled through a bilevel formulation: the lower-level problem optimizes the network weights under a fixed selection pattern, while the upper-level problem determines which parameters should be retained and which should be pruned. The corresponding formulation is given by
\begin{align}\label{eq:model-pruning}
	& \min_{x} \underbrace{\frac{1}{K}\sum_{k=1}^{K}\ell^{(k)}(x \odot y^*(x))}_{\text{Pruning task}} \qquad s.t. \quad \underbrace{y^*(x) = \arg\min_{y}  \frac{1}{K}\sum_{k=1}^{K}\ell^{(k)}(x \odot y)}_{\text{Sparsity-fixed model retraining}}  \ , 
\end{align}
where $y\in\mathbb{R}^{d}$ represents the network weights and $x\in \{0, 1\}^{d}$ encodes the selection pattern, with zero entries indicating removed connections.

To solve Eq.~(\ref{eq:model-pruning}), we instantiate this setup using a ResNet20 architecture trained on CIFAR10 and CIFAR100. Due to the overparameterized nature of deep neural networks relative to the complexity of these datasets, the lower-level objective satisfies the PL condition under standard assumptions \citep{liu2022loss}. Our algorithm is then applied to jointly determine the model parameters and the sparsity structure, with a target sparsity level of $80\%$ (i.e., retaining $20\%$ of the weights).

All experiments are conducted under the same setup as in the previous section, except that we consider a heterogeneous data distribution across clients ($\delta > 0$). Specifically, the data are partitioned according to a Dirichlet distribution \citep{hsu2019measuring} with $Dir(0.5)$ , resulting in a non-IID setting. The testing accuracy and upper-level loss over iterations are reported in Figure~\ref{fig:mp}. As shown in the figure, our method consistently achieves higher accuracy and lower objective values at a faster rate compared to competing approaches, demonstrating improved optimization efficiency in this setting.

\begin{figure*}[ht]
	\vspace{-0.1in}
	\begin{center}
		\centerline{\includegraphics[scale=0.25]{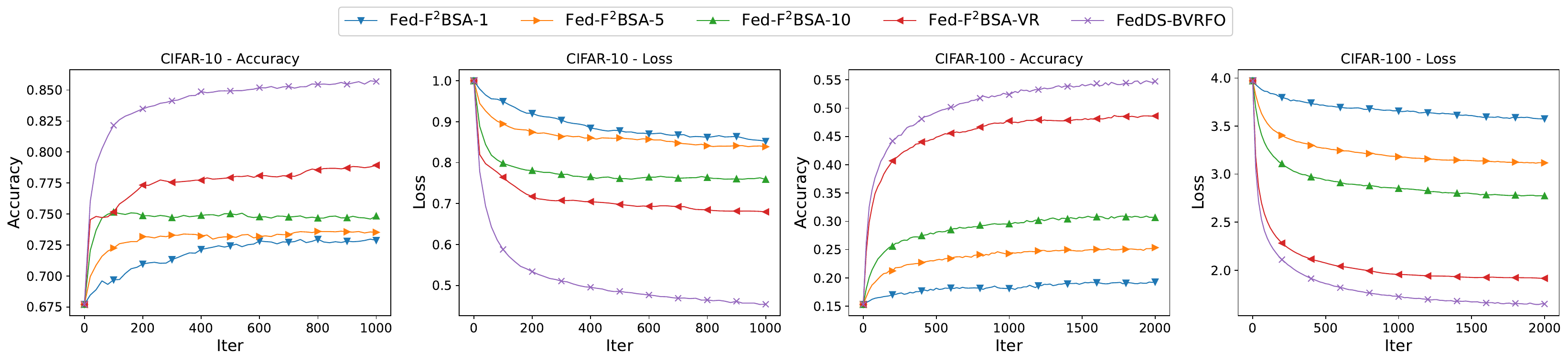}}
		\vspace{-0.1in}
		\caption{Testing accuracy and upper-level loss for model pruning, $\tau = 4$, nonconvex.}
		\label{fig:mp}
	\end{center}
	\vspace{-15pt}
\end{figure*}

\section{Conclusion}
In this paper, we proposed a stochastic doubly smoothed framework for federated bilevel optimization, leading to a fully first-order algorithm with unified learning-rate scaling. By integrating smoothing into a Lagrangian reformulation, the proposed method effectively stabilizes the coupled dynamics across variables and avoids expensive second-order computations. We further established convergence guarantees under relaxed nonconvex lower-level conditions and derived improved sample and communication complexities with favorable dependence on the condition number. Extensive experiments on benchmark tasks, including hyperparameter optimization and model pruning, demonstrate the efficiency and robustness of our approach compared to existing methods.

\bibliographystyle{abbrv}
\bibliography{sample-base}

\appendix
\section{Additional Experiments}
\label{app:exp}
\vspace{-5pt}
\subsection{Nonconvex Hyperparameter Optimization}
\vspace{-5pt}
We further evaluate all methods under a larger communication interval ($\tau = 16$), as shown in Figure~\ref{fig:p_16_acc} and Figure~\ref{fig:p_16_loss}.  A similar performance pattern is observed, where our method continues to achieve superior accuracy and lower upper-level loss, demonstrating its robustness and efficiency under reduced communication frequency.

\begin{figure*}[ht]
\vspace{-0.1in}
\begin{center}
\centerline{\includegraphics[scale=0.25]{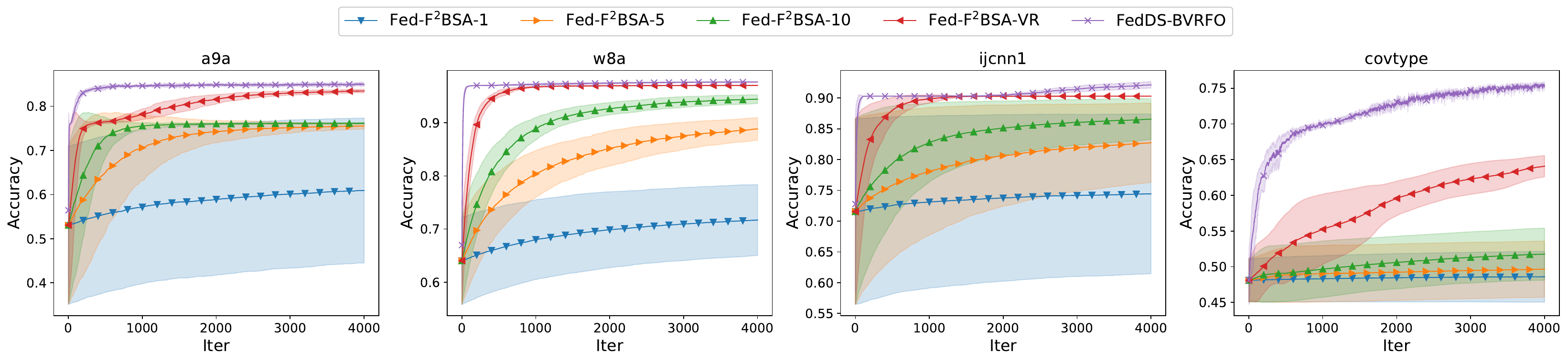}}
\vspace{-0.1in}
\caption{The test accuracy versus
the training epochs, $\tau = 16$, nonconvex.}
\label{fig:p_16_acc}
\end{center}
\vspace{-5pt}
\end{figure*}

\begin{figure*}[ht]
\vspace{-0.3in}
\begin{center}
\centerline{\includegraphics[scale=0.25]{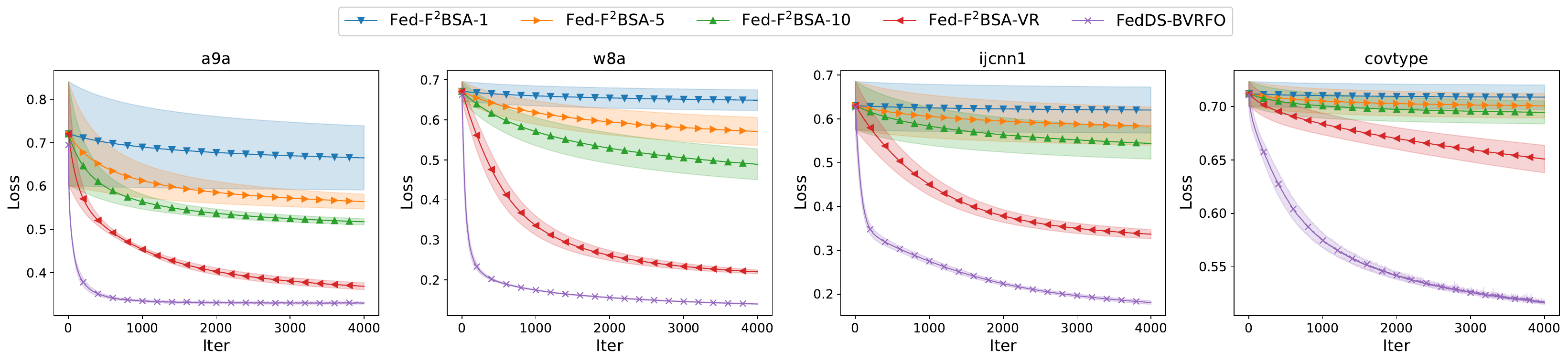}}
\vspace{-0.1in}
\caption{The upper-level loss
function value versus
the training epochs, $\tau = 16$, nonconvex.}
\label{fig:p_16_loss}
\end{center}
\vspace{-5pt}
\end{figure*}

\subsection{Hyperparameter Optimization on Nonconvex-Strongly-Convex Bilevel Problem}
\vspace{-5pt}
To provide a more comprehensive comparison with existing federated stochastic bilevel baselines, we conduct additional experiments on a hyperparameter optimization task under a nonconvex–strongly-convex bilevel formulation. Specifically, 
	{\begin{align}\label{eq:exp-hyperparameter-opt_strongly}
		& \qquad \min_{x\in \mathbb{R}^d} \frac{1}{K} \sum_{k=1}^{K}\frac{1}{m^{(n)}}\hat{\mathcal{L}}^{(k)}(y^*(x)^T; \mathcal{D}^{(k)}_{val})  \\
        & s.t. \quad  y^*(x) =  \arg\min_{y\in \mathbb{R}^{d\times c}}\frac{1}{K} \sum_{k=1}^{K}  \hat{\mathcal{L}}^{(k)}(y;\mathcal{D}^{(k)}_{tr}) + \hat{\mathcal{R}}(x) \ , \notag
	\end{align}
    }
where the lower-level problem learns  the parameter $y\in \mathbb{R}^{d\times c}$ of a logistic regression model using the training data, while the upper-level problem optimizes the hyperparameter $x\in \mathbb{R}^d$ with regularization term $\hat{\mathcal{R}}(x) =\frac{1}{d}\sum_{q=1}^{d}\exp(x_q) \frac{1}{c}\sum_{p=1}^{c} y_{pq}^2 $ based on the validation data. 

We compare Algorithm~\ref{alg:fedss-bvrfo} with four state-of-the-art second-order methods: FedNEST~\citep{tarzanagh2022fednest}, FedMBO~\citep{huang2023achieving}, LocalBSGVRM~\citep{gao2022convergence}, and FedBiOAcc~\citep{li2023communicationefficient}, as well as two first-order methods: MemFBO~\citep{yang2025first} and FedSVRBGD-FO~\citep{zhang2025federated}. For a fair comparison, we fix the target test accuracy for each dataset and set hyperparameters according to their theoretical guidelines. We report the upper-level objective value required to reach this accuracy versus running time in Figure~\ref{fig:p_4_loss_strongly}. The results show that first-order methods converge faster than second-order methods, as they avoid the computation of Jacobian and Hessian matrices. Compared with MemFBO and FedSVRBGD-FO, our method achieves comparable or better convergence behavior. These results indicate that the proposed smoothing approach effectively enhances convergence speed in nonconvex–strongly-convex bilevel settings.

\begin{figure*}[ht]
\vspace{-0.1in}
\begin{center}
\centerline{\includegraphics[scale=0.26]{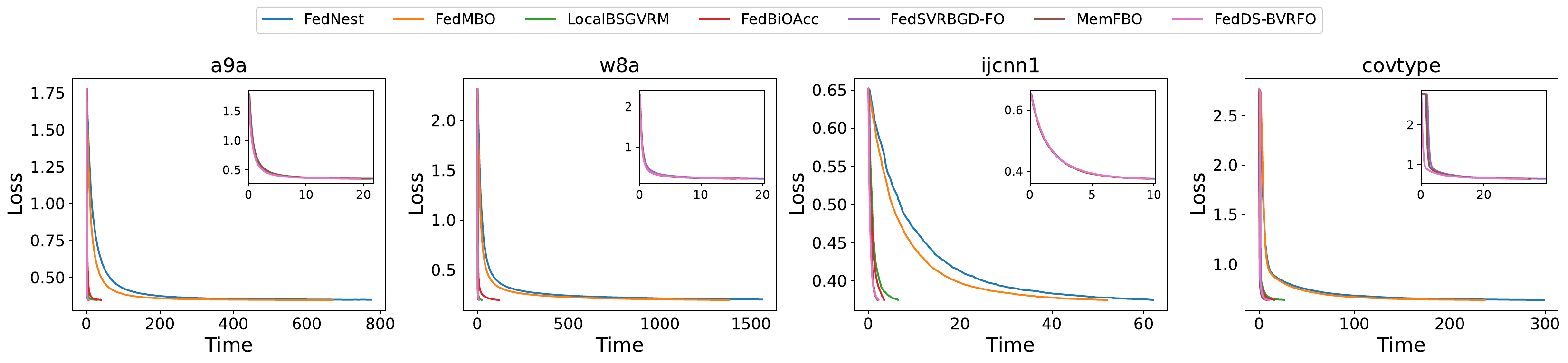}}
\vspace{-0.1in}
\caption{The upper-level loss
function value versus the running time, $\tau = 4$, strongly-convex.}
\label{fig:p_4_loss_strongly}
\end{center}
\vspace{-5pt}
\end{figure*}

\section{Complete Algorithm}\label{app:alg}
\vspace{-5pt}
The full algorithm is presented  in Algorithm~\ref{alg:2}.
\begin{algorithm*}[ht]
	\small
	\caption{FedDS-BVRFO}
	\label{alg:2}
	\begin{algorithmic}[1]
		\REQUIRE $\alpha_x>0$, $\alpha_y>0$, $\alpha_z>0$, $\alpha_p>0$, $\alpha_q>0$, $\alpha_r>0$, 
		$\beta_{x}>0$, $\beta_{y}>0$, $\beta_{z}>0$, $\eta>0$.  \\
		Initialization on each worker $k\in\{1, \cdots, K\}$: 
		${x}_{0}^{(k)}={x}_{0}$, ${y}_{0}^{(k)}={y}_{0}$, ${z}_{0}^{(k)}={z}_{0}$, 
		\FOR{$t=0,\cdots, T-1$} 
		
		%
		%
		\STATE Update variables $x$, $y$, and $z$: 
		\vspace{-5pt}
		\begin{align}
			&	x^{(k)}_{t+1} = x^{(k)}_{t} -  \alpha_{x} \eta u^{(k)}_{t}  \ , \qquad y^{(k)}_{t+1} = y^{(k)}_{t} -  \alpha_{y} \eta v^{(k)}_{t} \ , \qquad z^{(k)}_{t+1} = z^{(k)}_{t} -  \alpha_{z} \eta w^{(k)}_{t} \ , \notag 
		\end{align}
		
		\STATE Update variables $p$,  $q$, and $r$: 
		\vspace{-5pt}
		\begin{align}
			& p^{(k)}_{t+1} = p^{(k)}_{t} + \alpha_{p} \eta (x^{(k)}_{t+1} - p^{(k)}_{t}) \ , \qquad q^{(k)}_{t+1} = q^{(k)}_{t} + \alpha_{q} \eta (y^{(k)}_{t+1} - q^{(k)}_{t}) \ , \qquad r^{(k)}_{t+1} = r^{(k)}_{t} + \alpha_{r} \eta (z^{(k)}_{t+1} - r^{(k)}_{t}) \ , \notag 
		\end{align}

		\STATE Update gradient $u$:  
		\vspace{-5pt}
		\begin{align}
			& u^{(k)}_{1, t+1} =  (1-\beta_x\eta^2)(u^{(k)}_{1, t} - \nabla_1 f^{(k)}(x^{(k)}_{t}, y^{(k)}_{t}; \xi^{(k)}_{t+1})) 
			+ \nabla_1 f^{(k)}(x^{(k)}_{t+1}, y^{(k)}_{t+1}; \xi^{(k)}_{t+1}) \ , \notag \\
			& u^{(k)}_{2, t+1} =  (1-\beta_x\eta^2)(u^{(k)}_{2, t} - \nabla_1 g_{\gamma_1}^{(k)}(x^{(k)}_{t}, y^{(k)}_{t}, p^{(k)}_{t}, q^{(k)}_{t}; \zeta^{(k)}_{t+1}))+ \nabla_1 g_{\gamma_1}^{(k)}(x^{(k)}_{t+1}, y^{(k)}_{t+1}, p^{(k)}_{t+1}, q^{(k)}_{t+1}; \zeta^{(k)}_{t+1}) \ ,  \notag \\
			& u^{(k)}_{3, t+1}  =  (1-\beta_x\eta^2)(u^{(k)}_{3, t} - \nabla_1 g_{\gamma_2}^{(k)}(x^{(k)}_{t}, z^{(k)}_{t}, r^{(k)}_{t}; \zeta^{(k)}_{t+1}) )
			+ \nabla_1 g_{\gamma_2}^{(k)}(x^{(k)}_{t+1}, z^{(k)}_{t+1}, r^{(k)}_{t+1}; \zeta^{(k)}_{t+1}) \ ,  \notag \\
			& u^{(k)}_{t+1} = \rho u^{(k)}_{1, t+1} +  (u^{(k)}_{2, t+1}  - u^{(k)}_{3, t+1}) \ , \notag  
		\end{align}
		
		\STATE Update gradient $v$:  
		\vspace{-5pt}
		\begin{align}
			& v^{(k)}_{1, t+1}  =  (1-\beta_y\eta^2)(v^{(k)}_{1, t} - \nabla_2 f^{(k)}(x^{(k)}_{t}, y^{(k)}_{t}; \xi^{(k)}_{t+1}))+ \nabla_2 f^{(k)}(x^{(k)}_{t+1}, y^{(k)}_{t+1}; \xi^{(k)}_{t+1})\ , \notag \\
			& v^{(k)}_{2, t+1}  =  (1-\beta_y\eta^2)(v^{(k)}_{2, t} - \nabla_2 g_{\gamma_1}^{(k)}(x^{(k)}_{t}, y^{(k)}_{t}, p^{(k)}_{t}, q^{(k)}_{t}; \zeta^{(k)}_{t+1})) + \nabla_2 g_{\gamma_1}^{(k)}(x^{(k)}_{t+1}, y^{(k)}_{t+1}, p^{(k)}_{t+1}, q^{(k)}_{t+1}; \zeta^{(k)}_{t+1})\ , \notag \\
			& v^{(k)}_{t+1}  =  \rho v^{(k)}_{1, t+1}   +  v^{(k)}_{2, t+1} \ , \notag  
		\end{align}

		\STATE Update gradient $w$:  
		\vspace{-5pt}
		\begin{align}
			& w^{(k)}_{ t+1}  =  (1-\beta_z\eta^2)(w^{(k)}_{t} - \nabla_2 g_{\gamma_2}^{(k)}(x^{(k)}_{t}, z^{(k)}_{t}, r^{(k)}_{t}; \zeta^{(k)}_{t+1})) + \nabla_2 g_{\gamma_2}^{(k)}(x^{(k)}_{t+1}, z^{(k)}_{t+1}, r^{(k)}_{t+1}; \zeta^{(k)}_{t+1}) \ , \notag  
		\end{align}
		
		\IF {$\text{mod}(t+1, \tau)==0$}
		\STATE Communication:
		\vspace{-5pt}
		\begin{align}
			& x^{(k)}_{t+1}=\frac{1}{K}\sum_{k'=1}^{K}x^{(k')}_{ t+1} \ , \quad y^{(k)}_{ t+1}=\frac{1}{K}\sum_{k'=1}^{K}y^{(k')}_{ t+1} \ , \quad z^{(k)}_{ t+1}=\frac{1}{K}\sum_{k'=1}^{K}z^{(k')}_{ t+1} \ , \notag \\
            & p^{(k)}_{t+1}=\frac{1}{K}\sum_{k'=1}^{K}p^{(k')}_{ t+1} \ , \quad q^{(k)}_{ t+1}=\frac{1}{K}\sum_{k'=1}^{K}q^{(k')}_{ t+1} \ , \quad r^{(k)}_{ t+1}=\frac{1}{K}\sum_{k'=1}^{K}r^{(k')}_{ t+1} \ ,  \notag \\
			&  u^{(k)}_{t+1} =\frac{1}{K}\sum_{k'=1}^{K}u^{(k')}_{ t+1} \ , \quad  v^{(k)}_{t+1} =\frac{1}{K}\sum_{k'=1}^{K}v^{(k')}_{t+1} \ , \quad w^{(k)}_{t+1} =\frac{1}{K}\sum_{k'=1}^{K}w^{(k')}_{t+1}  \ ,  \notag 
		\end{align}
		\ENDIF
		\ENDFOR
	\end{algorithmic}
\end{algorithm*}

\section{Proof Sketch}\label{app:sketch}
To facilitate the convergence analysis, we introduce the following auxiliary quantities, following the constructions in \cite{zhang2020single, yang2022faster, zheng2023universal}:
\begin{align}\label{eq:auxiliary-fun-def}
    &  F_d({z};  p, q, r)= \min_{x,y} F(x, y, z; p, q, r) = F(x^*({z};  p, q, r), y^*({z};  p, q, r), {z}; p, q, r) \ ,\notag \\
    & (x^*({z};  p, q, r), y^*({z};  p, q, r)) =  \arg \min_{x, y} F(x, y, z; p, q, r) \notag \\
    & F_p(x, y; p, q, r)= \max_{z} F(x, y, z; p, q, r) = F(x, y, z^*({x}, {y};  p, q, r); p, q, r) \ ,\notag \\
    & z^*({x}, {y};  p, q, r) = \arg \max_{z} F(x, y, z; p, q, r) \ , \notag \\
    & H(p, q, r) =  \max_{z} F_d(z; p, q, r)  = \min_{x, y}\max_{z} F(x, y, z; p, q, r)   \ , \notag \\
    & z^*(p, q, r) =  \arg \max_{z} F_{d} (z; p, q, r) \ , \notag \\
    & x^*(p, q, r) = x^*(z^*(p, q, r); p, q, r)  \ , \qquad  y^*(p, q, r) = y^*(z^*(p, q, r); p, q, r)  \ , \notag \\
    & H_{z}(p, q) = \max_{r} H(p, q, r) \ , \qquad r^*(p, q) = \arg \max_{r} H(p, q, r) \ , \notag \\
    & z^+ = z_{t} + \alpha_{z}\eta\nabla_{z} F_{d}(z_{t}; p_{t}, q_{t}, r_{t}) \ , \qquad  r^+ = \bar{r}_t + \alpha_r\eta(z^*(\bar{p}_t,\bar{q}_t,\bar{r}_t)-\bar{r}_t) \ .
\end{align}

We next provide a proof sketch for the convergence result. The analysis consists of two main components: (i) the smoothing technique (Appendix~\ref{app:smooth}), and (ii) the federated analysis (Appendix~\ref{app:fed}).

\paragraph{Smoothing Analysis.}
We begin by establishing several auxiliary results that characterize the properties of the optimal solution mappings. Specifically, Lemmas~\ref{lemma:x_y_star_z_continuity}–\ref{lemma:xy_z_star_xy_star_r} provide continuity and regularity properties of the associated optimal functions.

Building on these results, we derive descent properties for the smoothed objective components in Lemmas~\ref{lemma:F_smooth}–\ref{lemma:h_smooth}. These results enable us to introduce the following smoothing potential function (Lemma~\ref{lemma:potential_psi}):
\begin{align}
     \Psi_{t} & = \underbrace{F(\bar{x}_{t}, \bar{y}_{t}, \bar{z}_{t}; \bar{p}_{t}, \bar{q}_{t}, \bar{r}_{t})}_{\text{Lemma~\ref{lemma:F_smooth}}} - 2 \underbrace{F_d(\bar{z}_{t}; \bar{p}_{t}, \bar{q}_{t}, \bar{r}_{t})}_{\text{Lemma~\ref{lemma:F_d_smooth}}} + 2\underbrace{H(\bar{p}_{t}, \bar{q}_{t})}_{\text{Lemma~\ref{lemma:h_smooth}}} \ . \notag
\end{align}
This potential function captures the joint progress of all variables and serves as the key tool for the convergence analysis. Based on this construction, we show that the step sizes of $x$, $y$, and $z$ can be chosen at the same order regarding $\kappa$:
\begin{align}
    \alpha_y = \alpha_x \ , \quad  \alpha_z = \alpha_x\frac{(\gamma_1 - L_{\rho})^2}{72L^2_{g}} \ , 
\end{align}
which highlights the stabilizing effect of the smoothing technique.

\paragraph{Federated Analysis.}
We begin by establishing two useful lemmas, Lemmas~\ref{lemma:incremental-x-y-z}–\ref{lemma:consensus-error-variables}, which characterize the evolution of local variables and their consensus errors. We then analyze the property of controlling gradient estimation errors and consensus errors.

First, we bound the six gradient estimation errors using Lemmas~\ref{lemma:grad_orig} and~\ref{lemma:grad_new}. In particular, Lemma~\ref{lemma:grad_new} establishes three novel bounds associated with the smoothed gradients of $g_{\gamma_1}$ and $g_{\gamma_2}$, which are essential for handling the doubly-smoothed terms. 

Next, we control the corresponding consensus errors in Lemmas~\ref{lemma:consensus-error-momentum-each-step-org} and \ref{lemma:consensus-error-momentum-each-step-new}. Among them,  Lemma~\ref{lemma:consensus-error-momentum-each-step-new} captures the additional error terms introduced by doubly-smoothed. Finally, these bounds are aggregated in Lemma~\ref{lemma:consensus_total}, which provides an overall consensus error bound across all six components.

\paragraph{Final Proof.} We construct a novel Lyapunov function $\mathcal{P}_{t}$ in Appendix~\ref{app:proof}:
\begin{align}
    & \mathcal{P}_{t} = \Psi_{t} + c_1\rho^2\mathcal{G}_{u, 1, t}  + c_2\mathcal{G}_{u, 2, t}   + c_3\mathcal{G}_{u, 3, t}   + c_4\rho^2\mathcal{G}_{v, 1, t} + c_5 \mathcal{G}_{v, 2, t}  + c_6\mathcal{G}_{w, t}  \ , \notag 
\end{align}
where we denote
\begin{align}
    & \mathcal{G}_{u, 1, t} = \mathbb{E}[\| \frac{1}{K}\sum_{k=1}^{K} \nabla_1 f^{(k)}(x^{(k)}_{t}, y^{(k)}_{t}) - \bar{u}_{1, t} \|^2] \ , \notag \\
    & \mathcal{G}_{u, 2, t} = \mathbb{E}[\| \frac{1}{K}\sum_{k=1}^{K} \nabla_1 g_{\gamma_1}^{(k)}(x^{(k)}_{t}, y^{(k)}_{t}; p^{(k)}_{t}, q^{(k)}_{t}) - \bar{u}_{2, t} \|^2] \ , \notag \\
    & \mathcal{G}_{u, 3, t} = \mathbb{E}[\| \frac{1}{K}\sum_{k=1}^{K} \nabla_1 g_{\gamma_2}^{(k)}(x^{(k)}_{t}, z^{(k)}_{t}; r^{(k)}_{t}) - \bar{u}_{3, t} \|^2] \ , \notag  \\
    & \mathcal{G}_{v, 1, t} = \mathbb{E}[\| \frac{1}{K}\sum_{k=1}^{K} \nabla_2 f^{(k)}(x^{(k)}_{t}, y^{(k)}_{t}) - \bar{v}_{1, t} \|^2] \ , \notag \\
    & \mathcal{G}_{v, 2, t} = \mathbb{E}[\| \frac{1}{K}\sum_{k=1}^{K} \nabla_2 g_{\gamma_1}^{(k)}(x^{(k)}_{t}, y^{(k)}_{t}; p^{(k)}_{t}, q^{(k)}_{t}) - \bar{v}_{2, t} \|^2] \ , \notag \\
    & \mathcal{G}_{w,t} = \mathbb{E}[\| \frac{1}{K}\sum_{k=1}^{K} \nabla_2 g_{\gamma_2}^{(k)}(x^{(k)}_{t}, z^{(k)}_{t}; r^{(k)}_{t}) - \bar{w}_{t} \|^2] \ .
\end{align}

This Lyapunov function integrates the smoothing potential $\Psi_{t}$ from Appendix~\ref{app:smooth} with the six gradient error terms from Appendix~\ref{app:fed}. By carefully choosing the coefficients $\{c_j\}_{j=1}^6$ and properly selecting the hyperparameter $\alpha_{x}$, we establish a descent property for $\mathcal{P}_{t}$, which leads to the desired convergence result.

\section{Smoothed Lemmas}\label{app:smooth}
\begin{lemma}\label{lemma:x_y_star_z_continuity}
	 Given Assumptions~\ref{assumption:f-smooth}-\ref{assumption:variance}, the following inequality holds:
	\begin{align}\label{eq:x_y_star_z_continuity}
		& (i). \ \|x^*(z; p, q, r) -x^*(z; p', q', r) \|^2+ \|y^*(z; p, q, r) -y^*(z; p', q', r) \|^2 \leq \omega_{1}^2(\| p - p' \|^2 + \| q - q' \|^2) \  , \notag \\
        & (ii).\  \|x^*(p, q, r) -x^*(p', q', r) \|^2 +  \|y^*(p, q, r) -y^*( p', q', r) \|^2  \leq \omega_{1}^2(\| p - p' \|^2 + \| q - q' \|^2) \ , \notag \\
        & (iii).\  \|x^*(z; p, q, r) -x^*(z'; p, q, r) \|^2+ \|y^*(z; p, q, r) -y^*(z'; p, q, r) \|^2 \leq \omega_{2}^2\| z - z' \|^2\ .  
	\end{align}
	where $\omega_{1} = \frac{\gamma_1}{(\gamma_1-L_{\rho})}, \omega_{2} = \frac{L_g}{(\gamma_1-L_{\rho})} $.
\end{lemma}

\begin{proof}
	\textbf{Dependence of $(x^*, y^*)$ on $(p, q)$.} To prove $(i)$, on the one hand, we split
	\begin{align}
		 F(x^*(z; p, q, r), y^*(z; p, q, r), z; p', q', r) -  F(x^*(z; p', q', r), y^*(z; p', q', r), z; p', q', r) \notag 
    \end{align}
    into three terms
    \begin{align}
		& T_1 : = F(x^*(z; p, q, r), y^*(z; p, q, r), z; p', q', r) -  F(x^*(z; p, q, r), y^*(z; p, q, r), z; p, q, r) \ , \notag \\
		& T_2 : = F(x^*(z; p, q, r), y^*(z; p, q, r), z; p, q, r) - F(x^*(z; p', q', r), y^*(z; p', q', r), z; p, q, r) \ ,  \notag \\
		& T_3: =F(x^*(z; p', q', r), y^*(z; p', q', r), z; p, q, r) -  F(x^*(z; p', q', r), y^*(z; p', q', r), z; p', q', r) \ . \notag 
    \end{align}
    For $T_2$, due to the $(\gamma_1 - L_{\rho})$-strong convexity of $F$ with respect to $x$ and $y$, we obtain
    \begin{align}
        T_2 \leq -\frac{(\gamma_1-L_{\rho})}{2} \|x^*(z; p, q, r) -x^*(z; p', q', r) \|^2-\frac{(\gamma_1-L_{\rho})}{2} \|y^*(z; p, q, r) -y^*(z; p', q', r) \|^2 \ . \notag 
    \end{align}
    For $T_1$ and $T_3$ , we derive
    \begin{align}
        T_1 & \leq \frac{\gamma_1}{2}\|x^*(z; p, q, r) - p'\|^2 +  \frac{\gamma_1}{2}\|y^*(z; p, q, r) - q'\|^2 \notag \\
        & \quad -  \frac{\gamma_1}{2}\|x^*(z; p, q, r) - p\|^2 -  \frac{\gamma_1}{2}\|y^*(z; p, q, r) - q\|^2  \notag \\
        & \leq \frac{\gamma_1}{2}(2\langle x^*(z; p, q, r) ,  p - p' \rangle+\| p'\|^2- \| p\|^2) + \frac{\gamma_1}{2}(2\langle y^*(z; p, q, r) ,  q - q' \rangle+\| q'\|^2- \| q\|^2) \ . \notag \\
        T_3 & \leq \frac{\gamma_1}{2}\|x^*(z; p', q', r) - p\|^2 +  \frac{\gamma_1}{2}\|y^*(z; p', q', r) - q\|^2 \notag \\
        & \quad -  \frac{\gamma_1}{2}\|x^*(z; p', q', r) - p'\|^2 -  \frac{\gamma_1}{2}\|y^*(z; p', q', r) - q'\|^2 \notag \\
        & \leq \frac{\gamma_1}{2}(2\langle x^*(z; p', q', r) ,  p' - p \rangle+\| p\|^2- \| p'\|^2) + \frac{\gamma_1}{2}(2\langle y^*(z; p', q', r) ,  q' - q \rangle+\| q\|^2- \| q'\|^2) \ . \notag
    \end{align}
    Combining $T_1$ and $T_3$ yields
    \begin{align}
        T_1 + T_3 \leq \gamma_1 \langle x^*(z; p, q, r) -x^*(z; p', q', r),  p - p' \rangle +\gamma_1\langle y^*(z; p, q, r) -y^*(z; p', q', r),  q - q' \rangle \ . \notag
    \end{align}
    Thus, combining $T_1 + T_2 + T_3$ yields
    \begin{align}
        & \quad F(x^*(z; p, q, r), y^*(z; p, q, r), z; p', q', r) -  F(x^*(z; p', q', r), y^*(z; p', q', r), z; p', q', r) \notag \\
        & \leq -\frac{(\gamma_1-L_{\rho})}{2} \|x^*(z; p, q, r) -x^*(z; p', q', r) \|^2-\frac{(\gamma_1-L_{\rho})}{2} \|y^*(z; p, q, r) -y^*(z; p', q', r) \|^2 \notag \\
		& \quad + \gamma_1 \langle x^*(z; p, q, r) -x^*(z; p', q', r),  p - p' \rangle  +\gamma_1\langle y^*(z; p, q, r) -y^*(z; p', q', r),  q - q' \rangle   \ . \notag
    \end{align}
	On the other hand, due to  the $(\gamma_1-L_{\rho})$-strong convexity of $F$ with respect to $x$ and $y$, we have
	\begin{align}
		& \quad F(x^*(z; p, q, r), y^*(z; p, q, r), z; p', q', r) -  F(x^*(z; p', q', r), y^*(z; p', q', r), z; p', q', r) \notag \\ 
		& \geq \frac{(\gamma_1-L_{\rho})}{2} \|x^*(z; p, q, r) -x^*(z; p', q', r) \|^2+\frac{(\gamma_1-L_{\rho})}{2} \|y^*(z; p, q, r) -y^*(z; p', q', r) \|^2  \  . \notag 
	\end{align}
	By combining the above two inequalities, we have
	\begin{align}
		& \quad  (\gamma_1-L_{\rho}) \|x^*(z; p, q, r) -x^*(z; p', q', r) \|^2+ (\gamma_1-L_{\rho}) \|y^*(z; p, q, r) -y^*(z; p', q', r) \|^2 \notag \\
		& \leq \gamma_1 \langle x^*(z; p, q, r) -x^*(z; p', q', r),  p - p' \rangle  +\gamma_1\langle y^*(z; p, q) -y^*(z; p', q'),  q - q' \rangle   \ . \notag
	\end{align}
	By denoting $a=[x^*(z; p, q, r) -x^*(z; p', q', r), y^*(z; p, q, r) -y^*(z; p', q', r)]^T$ and $b=[p - p', q - q']^T$, the above inequality can be represented as follows:
	\begin{align}
		& (\gamma_1-L_{\rho}) \|a\|^2 \leq \gamma_1 \langle a, b \rangle \leq \gamma_1 \|a\| \|b\| \ . \notag
	\end{align}
	It indicates
	\begin{align}
		& (\gamma_1-L_{\rho}) \|a\| \leq \gamma_1 \|b\| \ . \notag
	\end{align}
	As a result, we have
	\begin{align}
		&   \|x^*(z; p, q, r) -x^*(z; p', q', r) \|^2+ \|y^*(z; p, q, r) -y^*(z; p', q', r) \|^2  \leq \omega_{1}^2(\| p - p' \|^2 + \| q - q' \|^2) \ . \notag
	\end{align}
    where $\omega_{1} = \frac{\gamma_1}{(\gamma_1-L_{\rho})}$.
	The second inequalities $(ii)$ in Eq.~(\ref{eq:x_y_star_z_continuity}) can be proved in a similar manner.

    \textbf{Dependence of $(x^*, y^*)$ on $z$.}  To prove $(iii)$, due to $(\gamma_1-L_\rho)$-strong convexity of $F$ with respect to $x$ and $y$, we obtain
	\begin{align}
		& \quad (\gamma_1-L_\rho)\|(x^*(z'; p, q, r),y^*(z'; p, q, r)) - (x^*(z; p, q, r),y^*(z; p, q, r))\| \notag \\
        & \leq \|\nabla_{x,y} F(x^*(z'; p, q, r),y^*(z'; p, q, r),z;p,q,r)\| \notag \\
		& = \|\nabla_{x,y}F(x^*(z'; p, q, r),y^*(z'; p, q, r),z;p,q,r) - \nabla_{x,y}F(x^*(z'; p, q, r),y^*(z'; p, q, r),z';p,q,r)\| \ , \notag
	\end{align}
	where the last equality uses $\nabla_{x,y}F(x^*(z'; p, q, r),y^*(z'; p, q, r),z';p,q,r)=0$. Recalling $\nabla_x F(x,y,z;p,q,r) = \rho\nabla_1 f(x,y) + \nabla_1 g(x,y) - \nabla_1 g(x,z) + \gamma_1(x-p)$ and $\nabla_y F(x,y,z;p,q,r) = \rho\nabla_2 f(x,y) + \nabla_2 g(x,y) + \gamma_1(y-q)$, only the $x$-component of $\nabla_{x,y}F$ depends on $z$. Hence
	\begin{align}
		& \quad \|\nabla_{x,y}F(x^*(z'; p, q, r),y^*(z'; p, q, r),z;p,q,r) - \nabla_{x,y}F(x^*(z'; p, q, r),y^*(z'; p, q, r),z';p,q,r)\|  \notag \\
        & = \|\nabla_1 g(x^*(z'), z') - \nabla_1 g(x^*(z'), z)\| \leq L_g\|z-z'\| \ . \notag
	\end{align}
	Combining the two displays, we obtain
	\begin{align}
		\|x^*(z;p,q,r) - x^*(z';p,q,r)\|^2+ \|y^*(z;p,q,r) - y^*(z';p,q,r)\|^2 \leq  \omega_2^2\|z-z'\|^2 \ , \notag
	\end{align}
	where $\omega_{2} = \frac{L_g}{(\gamma_1-L_{\rho})} $.
\end{proof}

\begin{lemma}\label{lemma:z_star_continuity}
	Given Assumptions~\ref{assumption:f-smooth}-\ref{assumption:variance}, the following inequalities hold:
	\begin{align}\label{eq:z_star_continuity}
        & (i) . \ \|z^*(x, y; p, q, r) - z^*(x, y; p, q, r') \|^2 \leq \omega_{3}^2 \|r - r'\|^2 \  , \notag \\
        & (ii) . \ \|z^*(p, q, r) - z^*(p, q, r') \|^2 \leq \omega_{3}^2 \|r - r'\|^2 \ , \notag \\
        & (iii) . \  \|z^*(x, y; p, q, r) - z^*(x', y; p, q, r) \|^2 \leq \omega_{4}^2 \|x - x'\|^2 \  , \notag \\
        & (iv). \ \|z^*(p, q, r) - z^*(p', q', r) \|^2 \leq \omega_{5}^2 \big(\|p-p'\|^2+\|q-q'\|^2\big) \ , 
	\end{align}
	where $\omega_{3} = \frac{\gamma_2}{(\gamma_2-L_{\rho})}$, $\omega_{4} = \frac{L_g}{(\gamma_2-L_{\rho})}$, and $\omega_{5} = \frac{\gamma_1\omega_2}{(\gamma_2-L_\rho)} + \frac{\omega_1}{\omega_2}$. 
\end{lemma}

\begin{proof}
    \textbf{Dependence of $z^*$ on $r$.} We split
	\begin{align}
		F(x, y, z^*(x, y; p, q, r); p, q, r') -  F(x, y, z^*(x, y; p, q, r'); p, q, r') \notag
	\end{align}
	into three terms
	\begin{align}
		& S_1 : = F(x, y, z^*(x, y; p, q, r); p, q, r') -  F(x, y, z^*(x, y; p, q, r); p, q, r) \ , \notag \\
		& S_2 : = F(x, y, z^*(x, y; p, q, r); p, q, r) - F(x, y, z^*(x, y; p, q, r'); p, q, r) \ ,  \notag \\
		& S_3: =F(x, y, z^*(x, y; p, q, r'); p, q, r) -  F(x, y, z^*(x, y; p, q, r'); p, q, r') \ . \notag
	\end{align}
	Since $(x, y, p, q)$ are held fixed, only the term $-\frac{\gamma_2}{2}\|z-r\|^2$ in $F$ contributes to $S_1$ and $S_3$; a direct expansion gives
	\begin{align}
		S_1 + S_3 = -\gamma_2 \langle z^* - z^{*\prime}, \ r - r' \rangle \ . \notag
	\end{align}
	For $S_2$, since $z^*$ maximizes $F(x,y,\cdot;p,q,r)$ and $F$ is $(\gamma_2-L_\rho)$-strongly concave in $z$, we have $S_2 \geq \frac{\gamma_2-L_\rho}{2}\|z^*-z^{*\prime}\|^2$. 
    
    On the other hand, since $z^{*\prime}$ maximizes $F(x,y,\cdot;p,q,r')$, strong concavity also gives $F(x,y,z^*;p,q,r') \leq F(x,y,z^{*\prime};p,q,r') - \frac{\gamma_2-L_\rho}{2}\|z^*-z^{*\prime}\|^2$, i.e., $S_1+S_2+S_3 \leq -\frac{\gamma_2-L_\rho}{2}\|z^*-z^{*\prime}\|^2$. Combining the two bounds on $S_1+S_2+S_3$,
	\begin{align}
		\frac{\gamma_2-L_\rho}{2}\|z^*-z^{*\prime}\|^2 - \gamma_2\langle z^*-z^{*\prime}, r-r'\rangle \leq S_1+S_2+S_3 \leq -\frac{\gamma_2-L_\rho}{2}\|z^*-z^{*\prime}\|^2 \ , \notag
	\end{align}
	which yields $(\gamma_2-L_\rho)\|z^*-z^{*\prime}\|^2 \leq \gamma_2 \langle z^*-z^{*\prime}, r-r'\rangle \leq \gamma_2\|z^*-z^{*\prime}\|\|r-r'\|$, and hence
	\begin{align}
		\|z^*(x, y; p, q, r) - z^*(x, y; p, q, r') \|^2 \leq  \omega_3^2\|r-r'\|^2 \ , \notag
	\end{align}
    where $\omega_3 = \frac{\gamma_2}{\gamma_2-L_\rho}$. The second inequalities $(ii)$ in Eq.~(\ref{eq:z_star_continuity}) can be proved in a similar manner.

    \textbf{Dependence of $z^*$ on $x$.} We first note that $z^*(x,y;p,q,r)$ does not depend on $y$ at all. By the $(\gamma_2-L_\rho)$-strong concavity of $F$ in $z$,
	\begin{align}
		&  \quad (\gamma_2-L_\rho)\|z^*(x,y;p,q,r) - z^*(x',y;p,q,r)\| \notag \\
        & \leq \| \nabla_z F(x,y,z^*(x',y;p,q,r); p, q, r) \| \notag \\
		& = \| \nabla_z F(x,y,z^*(x',y;p,q,r); p, q, r) - \nabla_z F(x',y,z^*(x',y;p,q,r); p, q, r) \| \ , \notag \\
        & \leq L_{g}\|x-x'\| \ , \notag 
	\end{align}
	where the equality uses $\nabla_z F(x',y,z^*(x',y;p,q,r);p,q,r)=0$. Thus. we obtain
	\begin{align}
		\|z^*(x, y; p, q, r) - z^*(x', y; p, q, r) \|^2 \leq  \omega_4^2\|x-x'\|^2 \ , \notag
	\end{align}
    where $\omega_4 = \frac{L_{g}}{\gamma_2-L_\rho} $.

    \textbf{Dependence of $z^*$ on $p, q$.} By strong concavity of $F_d$, we obtain
	\begin{align}
		F_d(z^*(p,q,r); p,q,r) - F_d(z^*(p',q',r); p,q,r) &\geq \frac{\gamma_2-L_\rho}{2}\|z^*(p,q,r) - z^*(p',q',r)\|^2 \ , \notag \\
		F_d(z^*(p,q,r); p',q',r) - F_d(z^*(p',q',r); p',q',r) &\leq -\frac{\gamma_2-L_\rho}{2}\|z^*(p,q,r) - z^*(p',q',r)\|^2 \ . \notag
	\end{align}
	Combining these two inequalities derives
	\begin{align}
		& \quad (\gamma_2-L_\rho)\|z^*(p,q,r) - z^*(p',q',r)\|^2 \notag \\
        & \leq \underbrace{\big[F_d(z^*(p,q,r);p,q,r) - F_d(z^*(p,q,r);p',q',r)\big]}_{=:U_1} \notag \\
        & \quad + \underbrace{\big[F_d(z^*(p',q',r);p',q',r) - F_d(z^*(p',q',r);p,q,r)\big]}_{=:U_2} \ . \notag
	\end{align}
	For $U_1$, we obtain
	\begin{align}
		U_1 & \leq \gamma_1\langle x^*(z^*(p,q,r);p',q',r), p-p'\rangle + \gamma_1\langle y^*(z^*(p,q,r);p',q',r), q-q'\rangle \notag \\
        & \quad + \frac{\gamma_1}{2}(\|p'\|^2-\|p\|^2+\|q'\|^2-\|q\|^2) \ . \notag
	\end{align}
	Symmetrically,
	\begin{align}
		U_2 & \leq \gamma_1\langle x^*(z^*(p',q',r);p,q,r), p'-p\rangle + \gamma_1\langle y^*(z^*(p',q',r);p,q,r), q'-q\rangle \notag \\
        & \quad + \frac{\gamma_1}{2}(\|p\|^2-\|p'\|^2+\|q\|^2-\|q'\|^2) \ . \notag
	\end{align}
	Adding $U_1$ and $U_2$, the quadratic terms cancel, leaving
	\begin{align}
		U_1+U_2 & \leq \gamma_1\langle x^*(z^*(p,q,r);p',q',r)-x^*(z^*(p',q',r);p,q,r), p-p'\rangle \notag \\
        & \quad + \gamma_1\langle y^*(z^*(p,q,r);p',q',r)-y^*(z^*(p',q',r);p,q,r), q-q'\rangle \ . \notag
	\end{align}
	Denoting 
    \begin{align}
        a := & [x^*(z^*(p,q,r);p',q',r)-x^*(z^*(p',q',r);p,q,r), \notag \\
        & \ y^*(z^*(p,q,r);p',q',r)-y^*(z^*(p',q',r);p,q,r)]^T \ , \notag \\
        b := & [p-p',q-q']^T \ , \notag 
    \end{align}
    the display above gives 
    \begin{align}
         & \quad (\gamma_2-L_\rho)\|z^*(p,q,r) - z^*(p',q',r)\|^2  \leq \gamma_1\langle a,b\rangle \leq \gamma_1\|a\|\|b\| \notag \\
         & \leq \gamma_1(\omega_1\|b\| + \omega_2\|z^*(p,q,r) - z^*(p',q',r)\|)\|b\|  \ , \notag 
    \end{align}
    where the last step from Lemma~\ref{lemma:x_y_star_z_continuity}.
    
	Setting $\xi := \frac{\|z^*(p,q,r) - z^*(p',q',r)\|}{\|b\| }$, we have $\xi^2 \leq \frac{\gamma_1\omega_2}{(\gamma_2-L_\rho)}\xi + \frac{\gamma_1\omega_1}{(\gamma_2-L_\rho)}$. Following ~\cite{zheng2023universal}, using AM-GM inequality, we obtain
	\begin{align}
		\|z^*(p,q,r) - z^*(p',q',r)\|^2 \leq \omega_5^2\big(\|p-p'\|^2+\|q-q'\|^2\big) \ , \notag
	\end{align}
    where $\omega_5 = \frac{\gamma_1\omega_2}{(\gamma_2-L_\rho)} + \frac{\omega_1}{\omega_2}$.
\end{proof}

\begin{lemma}\label{lemma:xy_xy_star_r}
    Given Assumptions~\ref{assumption:f-smooth}-\ref{assumption:variance}, by defining $\hat{r} := r^*(\bar{p}_{t+1}, \bar{q}_{t+1})$, $r^+ = \bar{r}_t + \alpha_r\eta(z^*(\bar{p}_t,\bar{q}_t,\bar{r}_t)-\bar{r}_t)$, $z^+ = \bar{z}_t + \alpha_z\eta\nabla_zF_d(\bar{z}_t;\bar{p}_t,\bar{q}_t,\bar{r}_t)$, and $C_r = \frac{4\gamma_2^2}{\mu(\gamma_1-L_\rho)}\big(\omega_3^2+\frac{(1-\alpha_r\eta)^2}{\alpha_r^2\eta^2}\big)$, the following inequality holds:
    \begin{align}
        & \quad \| x^*( \bar{p}_{t+1}, \bar{q}_{t+1}, r^+) - x^*(\bar{p}_{t+1}, \bar{q}_{t+1}, \hat{r}) \|^2  + \| y^*( \bar{p}_{t+1}, \bar{q}_{t+1}, r^+) - y^*( \bar{p}_{t+1}, \bar{q}_{t+1}, \hat{r}) \|^2 \notag \\
        & \leq \frac{2\omega_5^2\gamma_2^2}{\mu(\gamma_1-L_\rho)}\big(\|\bar{p}_{t+1}-\bar{p}_t\|^2+\|\bar{q}_{t+1}-\bar{q}_t\|^2\big) + 2C_r\|\bar{r}_{t+1}-\bar{r}_t\|^2 \notag \\
        &  \quad  + 4C_r\alpha_r^2\eta^2\Big(\alpha_z^2\eta^2+ \frac{2}{(\gamma_2-L_\rho)^2}\Big) \|\bar{w}_t\|^2 + \frac{16C_r\alpha_r^2\eta^2L_g^2}{(\gamma_1-L_\rho)^2(\gamma_2-L_\rho)^2}\|\nabla_xF(\bar{x}_t,\bar{y}_t,\bar{z}_t;\bar{p}_t,\bar{q}_t,\bar{r}_t)\|^2 \notag \\
        & \quad  + \frac{16C_r\alpha_r^2\eta^2}{(\gamma_2-L_\rho)^2}\|\nabla_zF(\bar{x}_t,\bar{y}_t,\bar{z}_t;\bar{p}_t,\bar{q}_t,\bar{r}_t)-\bar{w}_t\|^2  \ .  
    \end{align}
\end{lemma}

\begin{proof}
    By the $(\gamma_1-L_\rho)$-strong convexity of $F_p$ in $(x,y)$, denoting $x^*_+: = x^*(\bar{p}_{t+1},\bar{q}_{t+1},r^+)$, $y^*_+: = y^*(\bar{p}_{t+1},\bar{q}_{t+1},r^+)$  we have
    \begin{align}
        & \quad \frac{\gamma_1-L_\rho}{2}\Big(\| x^*_+ - x^*(\bar{p}_{t+1}, \bar{q}_{t+1}, \hat{r}) \|^2  + \| y^*_+ - y^*( \bar{p}_{t+1}, \bar{q}_{t+1}, \hat{r}) \|^2\Big) \notag \\
        & \leq \max_r F_p(x^*_+, y^*_+;\bar{p}_{t+1},\bar{q}_{t+1},r) - F_p(x^*_+,y^*_+;\bar{p}_{t+1},\bar{q}_{t+1},\hat{r}) \notag \\
		& \leq \max_r F_p(x^*_+,y^*_+;\bar{p}_{t+1},\bar{q}_{t+1},r) - F_p(x^*_+,y^*_+;\bar{p}_{t+1},\bar{q}_{t+1},r^+) \notag \\
        & \leq \frac{1}{2\mu}\|\nabla_r F_p(x^*_+,y^*_+;\bar{p}_{t+1},\bar{q}_{t+1},r^+)\|^2 \notag \\
		& = \frac{\gamma_2^2}{2\mu}\|z^*(\bar{p}_{t+1},\bar{q}_{t+1},r^+) - r^+\|^2 \ , \notag 
    \end{align}
    where the second-to-last step holds due to the PL-property.

    Then, since $r^+ = \bar{r}_t + \alpha_r\eta(z^*(\bar{p}_t,\bar{q}_t,\bar{r}_t)-\bar{r}_t)$, by Young's inequality, we obtain
    \begin{align}
        & \quad \|z^*(\bar{p}_{t+1},\bar{q}_{t+1},r^+)-r^+\|^2 \notag \\
        & \leq 2\|z^*(\bar{p}_{t+1},\bar{q}_{t+1},r^+) - z^*(\bar{p}_t,\bar{q}_t,\bar{r}_t)\|^2 + 2(1-\alpha_r\eta)^2\|z^*(\bar{p}_t,\bar{q}_t,\bar{r}_t)-\bar{r}_t\|^2 \notag \\
		& \leq 4\|z^*(\bar{p}_{t+1},\bar{q}_{t+1},r^+) - z^*(\bar{p}_t,\bar{q}_t,r^+)\|^2 + 4\|z^*(\bar{p}_t,\bar{q}_t,r^+) - z^*(\bar{p}_t,\bar{q}_t,\bar{r}_t)\|^2 \notag \\
		& \quad + 2(1-\alpha_r\eta)^2\|z^*(\bar{p}_t,\bar{q}_t,\bar{r}_t)-\bar{r}_t\|^2 \notag \\
		& \leq 4\omega_5^2\big(\|\bar{p}_{t+1}-\bar{p}_t\|^2+\|\bar{q}_{t+1}-\bar{q}_t\|^2\big) + 4\omega_3^2\|r^+ - \bar{r}_t\|^2 + \frac{2(1-\alpha_r\eta)^2}{\alpha_r^2\eta^2}\|r^+ - \bar{r}_t\|^2 \notag \\
		& \leq 4\omega_5^2\big(\|\bar{p}_{t+1}-\bar{p}_t\|^2+\|\bar{q}_{t+1}-\bar{q}_t\|^2\big) + 4\big(\omega_3^2+\frac{(1-\alpha_r\eta)^2}{\alpha_r^2\eta^2}\big)\|r^+ - \bar{r}_t\|^2 \ . \notag
    \end{align}
    Finally, we derive
    \begin{align}
        & \quad \|r^+ - \bar{r}_t\|^2  \leq  2\|\bar{r}_{t+1} - \bar{r}_t\|^2 + 2\|r^+ - \bar{r}_{t+1} \|^2  \notag \\
        & = 2\|\bar{r}_{t+1} - \bar{r}_t\|^2 + 2\|\bar{r}_t + \alpha_r\eta(z^*(\bar{p}_t,\bar{q}_t,\bar{r}_t)-\bar{r}_t) - \bar{r}_{t} - \alpha_r\eta(\bar{z}_{t+1} - \bar{r}_{t}) \|^2  \notag \\
        & = 2\|\bar{r}_{t+1}-\bar{r}_t\|^2 + 2\alpha_r^2\eta^2\|\bar{z}_{t+1}-z^*(\bar{p}_t,\bar{q}_t,\bar{r}_t)\|^2 \notag \\
		& \leq 2\|\bar{r}_{t+1}-\bar{r}_t\|^2 + 4\alpha_r^2\eta^2\|\bar{z}_{t+1}-\bar{z}_t\|^2 + 4\alpha_r^2\eta^2\|\bar{z}_t - z^*(\bar{p}_t,\bar{q}_t,\bar{r}_t)\|^2 \notag \\
		& \leq 2\|\bar{r}_{t+1}-\bar{r}_t\|^2 + 4\alpha_r^2\eta^2\alpha_z^2\eta^2\|\bar{w}_t\|^2 + \frac{4\alpha_r^2}{\alpha_z^2(\gamma_2-L_\rho)^2}\|z^+-\bar{z}_t\|^2 \ , \notag
    \end{align}
    where the last step holds due to the the $(\gamma_2-L_\rho)$-strong concavity of $F_d(\cdot;\bar{p}_t,\bar{q}_t,\bar{r}_t)$, and the notation of $z^+ = \bar{z}_t + \alpha_z\eta\nabla_zF_d(\bar{z}_t;\bar{p}_t,\bar{q}_t,\bar{r}_t)$:
    \begin{align}
		(\gamma_2-L_\rho)\|\bar{z}_t - z^*(\bar{p}_t,\bar{q}_t,\bar{r}_t)\| \leq \|\nabla_z F_d(\bar{z}_t;\bar{p}_t,\bar{q}_t,\bar{r}_t)\| = \frac{1}{\alpha_z\eta}\|z^+-\bar{z}_t\| \ . \notag
	\end{align}
    Similarly, we derive
    \begin{align}\label{eq:z_plus_z_diff}
        & \quad \|z^+-\bar{z}_t\|^2 \leq 2\|z^+-\bar{z}_{t+1}\|^2 + 2\|\bar{z}_{t+1}-\bar{z}_t\|^2 \notag \\
		& = 2\alpha_z^2\eta^2\|\nabla_zF_d(\bar{z}_t;\bar{p}_t,\bar{q}_t,\bar{r}_t)-\bar{w}_t\|^2 + 2\alpha_z^2\eta^2\|\bar{w}_t\|^2 \notag \\
		& \leq 4\alpha_z^2\eta^2\|\nabla_zF_d(\bar{z}_t;\bar{p}_t,\bar{q}_t,\bar{r}_t)-\nabla_zF(\bar{x}_t,\bar{y}_t,\bar{z}_t;\bar{p}_t,\bar{q}_t,\bar{r}_t)\|^2  \notag \\
		& \quad + 4\alpha_z^2\eta^2\|\nabla_zF(\bar{x}_t,\bar{y}_t,\bar{z}_t;\bar{p}_t,\bar{q}_t,\bar{r}_t)-\bar{w}_t\|^2 + 2\alpha_z^2\eta^2\|\bar{w}_t\|^2 \notag \\
        & \leq 4\alpha_z^2\eta^2L_g^2\|x^*(\bar{z}_t;\bar{p}_t,\bar{q}_t,\bar{r}_t)-\bar{x}_t\|^2 + 4\alpha_z^2\eta^2\|\nabla_zF(\bar{x}_t,\bar{y}_t,\bar{z}_t;\bar{p}_t,\bar{q}_t,\bar{r}_t)-\bar{w}_t\|^2 \notag \\
		& \quad  + 2\alpha_z^2\eta^2\|\bar{w}_t\|^2 \notag \\
		& \leq \frac{4\alpha_z^2\eta^2L_g^2}{(\gamma_1-L_\rho)^2}\|\nabla_xF(\bar{x}_t,\bar{y}_t,\bar{z}_t;\bar{p}_t,\bar{q}_t,\bar{r}_t)\|^2 + 4\alpha_z^2\eta^2\|\nabla_zF(\bar{x}_t,\bar{y}_t,\bar{z}_t;\bar{p}_t,\bar{q}_t,\bar{r}_t)-\bar{w}_t\|^2 \notag \\
		& \quad + 2\alpha_z^2\eta^2\|\bar{w}_t\|^2 \ , 
    \end{align}
     where the last step holds from the $(\gamma_1-L_\rho)$-strongly convex of $F(\cdot,\cdot,\bar{z}_t;\bar{p}_t,\bar{q}_t,\bar{r}_t)$:
    \begin{align}
        \| x^*(\bar{z}_{t}; \bar{p}_{t}, \bar{q}_{t}, \bar{r}_t) - \bar{x}_{t}\|^2 \leq \frac{1}{(\gamma_1-L_{\rho})^2}\|  \nabla_{x} F(\bar{x}_{t}, \bar{y}_{t},  \bar{z}_{t}; \bar{p}_{t}, \bar{q}_{t}, \bar{r}_t) \|^2  \ . \notag 
    \end{align}
    Thus, we have
    \begin{align}\label{eq:r_plus_r_diff}
         &  \|r^+ - \bar{r}_t\|^2  \leq  2\|\bar{r}_{t+1}-\bar{r}_t\|^2 + 4\alpha_r^2\eta^2\Big(\alpha_z^2\eta^2 + \frac{2}{(\gamma_2-L_\rho)^2}\Big) \|\bar{w}_t\|^2 \notag \\
         & + \frac{16\alpha_r^2\eta^2L_g^2}{(\gamma_1-L_\rho)^2(\gamma_2-L_\rho)^2}\|\nabla_xF(\bar{x}_t,\bar{y}_t,\bar{z}_t;\bar{p}_t,\bar{q}_t,\bar{r}_t)\|^2 \notag \\
         & + \frac{16\alpha_r^2\eta^2}{(\gamma_2-L_\rho)^2}\|\nabla_zF(\bar{x}_t,\bar{y}_t,\bar{z}_t;\bar{p}_t,\bar{q}_t,\bar{r}_t)-\bar{w}_t\|^2 \ .   
    \end{align}
    Combining the above and dividing by $\frac{\gamma_1-L_\rho}{2}$ completes the proof.
\end{proof}

\begin{lemma}\label{lemma:xy_z_star_xy_star_r}
Given Assumptions~\ref{assumption:f-smooth}-\ref{assumption:variance}, by defining $\hat{r} := r^*(\bar{p}_{t+1}, \bar{q}_{t+1})$, and $z^+ = z + \alpha_{z}\eta\nabla_{z} F(x^*({z}; p, q, r), y^*({z}; p, q, r),  z; p, q, r)$, the following inequality holds:
    \begin{align}
	   &  \qquad \|x^*(\bar{z}_{t+1}; \bar{p}_{t+1}, \bar{q}_{t+1}, \bar{r}_{t}) - x^*(\bar{p}_{t+1}, \bar{q}_{t+1},  {r}^+) \|^2   \notag \\
       & \quad +  \| y^*(\bar{z}_{t+1}; \bar{p}_{t+1}, \bar{q}_{t+1}, \bar{r}_{t})- y^*(\bar{p}_{t+1}, \bar{q}_{t+1},  {r}^+) \|^2  \notag \\
       & \leq \Big(10\omega_2^2\alpha_z^2\eta^2\big(1+2\omega_3^2\alpha_r^2\eta^2\big) + \frac{10\omega_2^2\big(1+4\omega_3^2\alpha_r^2\eta^2\big)}{(\gamma_2-L_\rho)^2}\Big)\|\bar{w}_t\|^2 \notag \\
        & \quad + \Big(\frac{30\omega_2^2\alpha_z^2\eta^2L_g^2}{(\gamma_1-L_\rho)^2}+\frac{20\omega_2^2L_g^2\big(1+4\omega_3^2\alpha_r^2\eta^2\big)}{(\gamma_1-L_\rho)^2(\gamma_2-L_\rho)^2}\Big)\|\nabla_xF(\bar{x}_t,\bar{y}_t,\bar{z}_t;\bar{p}_t,\bar{q}_t,\bar{r}_t)\|^2 \notag \\
        & \quad + \Big(30\omega_2^2\alpha_z^2\eta^2+\frac{20\omega_2^2\big(1+4\omega_3^2\alpha_r^2\eta^2\big)}{(\gamma_2-L_\rho)^2}\Big)\|\nabla_zF(\bar{x}_t,\bar{y}_t,\bar{z}_t;\bar{p}_t,\bar{q}_t,\bar{r}_t)-\bar{w}_t\|^2 \notag \\
        & \quad + 5\omega_2^2\omega_5^2(\|\bar{p}_{t+1}-\bar{p}_t\|^2+\|\bar{q}_{t+1}-\bar{q}_t\|^2) + 10\omega_2^2\omega_3^2\|\bar{r}_{t+1}-\bar{r}_t\|^2 \ . 
    \end{align}
\end{lemma}

\begin{proof}
At first, from Lemmas~\ref{lemma:x_y_star_z_continuity} and ~\ref{lemma:z_star_continuity}, we decompose 
    \begin{align}
	   &  \qquad \|x^*(\bar{z}_{t+1}; \bar{p}_{t+1}, \bar{q}_{t+1}, \bar{r}_{t}) - x^*(\bar{p}_{t+1}, \bar{q}_{t+1},  {r}^+) \|^2   \notag \\
       & \quad +  \| y^*(\bar{z}_{t+1}; \bar{p}_{t+1}, \bar{q}_{t+1}, \bar{r}_{t})- y^*(\bar{p}_{t+1}, \bar{q}_{t+1},  {r}^+) \|^2  \notag \\
       & \leq \omega_2^2\|\bar{z}_{t+1} - z^*(\bar{p}_{t+1},\bar{q}_{t+1},r^+)\|^2  \notag \\
       & \leq 5\omega_2^2\|\bar{z}_{t+1}-z^+\|^2 + 5\omega_2^2\|z^+-\bar{z}_t\|^2 + 5\omega_2^2\|\bar{z}_t - z^*(\bar{p}_t,\bar{q}_t,\bar{r}_t)\|^2\notag \\
       & \quad + 5\omega_2^2\|z^*(\bar{p}_t,\bar{q}_t,\bar{r}_t)-z^*(\bar{p}_{t+1},\bar{q}_{t+1},\bar{r}_t)\|^2   + 5\omega_2^2\|z^*(\bar{p}_{t+1},\bar{q}_{t+1},\bar{r}_t)-z^*(\bar{p}_{t+1},\bar{q}_{t+1},r^+)\|^2 \notag \\
       & \leq 5\omega_2^2\|\bar{z}_{t+1}-z^+\|^2 + 5\omega_2^2(1 + \frac{1}{\alpha_z^2\eta^2 (\gamma_2-L_\rho)^2})\|z^+-\bar{z}_t\|^2  \notag \\
       & \quad + 5\omega_2^2\omega_5^2(\|\bar{p}_{t+1}-\bar{p}_t\|^2+\|\bar{q}_{t+1}-\bar{q}_t\|^2)  + 5\omega_2^2\omega_3^2\|r^+-\bar{r}_t\|^2 \ , \notag 
    \end{align}
    where the last step holds due to the $(\gamma_2-L_\rho)$-strong concavity of $F_d(\cdot;\bar{p}_t,\bar{q}_t,\bar{r}_t)$ and $z^+= \bar{z}_t+ \alpha_z\eta\nabla_zF_d(\bar{z}_t;\bar{p}_t,\bar{q}_t,\bar{r}_t)$, 
    \begin{align}
        (\gamma_2-L_\rho)\|\bar{z}_t - z^*(\bar{p}_t,\bar{q}_t,\bar{r}_t)\| \leq \|\nabla_z F_d(\bar{z}_t;\bar{p}_t,\bar{q}_t,\bar{r}_t)\| = \frac{1}{\alpha_z\eta}\|z^+-\bar{z}_t\| \ . \notag
    \end{align}
    In addition, following Eq.~(\ref{eq:z_plus_z_diff}), we derive
    \begin{align}
        \|\bar{z}_{t+1}-z^+\|^2 & \leq \frac{2\alpha_z^2\eta^2L_g^2}{(\gamma_1-L_\rho)^2}\|\nabla_xF(\bar{x}_t,\bar{y}_t,\bar{z}_t;\bar{p}_t,\bar{q}_t,\bar{r}_t)\|^2 \notag \\
        & \quad + 2\alpha_z^2\eta^2\|\nabla_zF(\bar{x}_t,\bar{y}_t,\bar{z}_t;\bar{p}_t,\bar{q}_t,\bar{r}_t)-\bar{w}_t\|^2 \ . \notag
    \end{align}
    Combining above inequalities with Eqs.~(\ref{eq:z_plus_z_diff}) and ~(\ref{eq:r_plus_r_diff}) completes the proof.
\end{proof}

\begin{lemma}\label{lemma:F_smooth}
    Given Assumptions~\ref{assumption:f-smooth}-\ref{assumption:variance}, by setting $\eta \leq \min\{\frac{1}{\alpha_p}, \frac{1}{\alpha_q} , \frac{1}{2\alpha_{x}L_{F,1}} \}$, and and defining $L_{F,1}:=\gamma_1+L_{\rho}$, $L_{F,2}:=\gamma_2+L_{\rho}$, the following inequality holds:
	\begin{align}
	    &   F(\bar{x}_{t+1}, \bar{y}_{t+1}, \bar{z}_{t+1}; \bar{p}_{t+1}, \bar{q}_{t+1}, \bar{r}_{t+1}) - F(\bar{x}_{t}, \bar{y}_{t}, \bar{z}_{t}; \bar{p}_{t}, \bar{q}_{t}, \bar{r}_{t})  \leq   - \frac{\gamma_1}{2\alpha_{p}\eta} \|\bar{p}_{t+1}-\bar{p}_{t}\|^2\notag \\
        & - \frac{\gamma_1}{2\alpha_{q}\eta}\|\bar{q}_{t+1}-\bar{q}_{t}\|^2   + \frac{\gamma_2(1-(1-\alpha_r\eta)^2)}{2\alpha_r^2\eta^2}\|\bar{r}_{t+1}-\bar{r}_{t}\|^2 - \frac{\alpha_x \eta}{4}\| \nabla_{x} F(\bar{x}_{t}, \bar{y}_{t}, \bar{z}_{t}; \bar{p}_{t}, \bar{q}_{t}, \bar{r}_{t}) \|^2 \notag \\
        &   -   \frac{\alpha_{y}\eta}{2} \|\nabla_{y}  F(\bar{x}_{t}, \bar{y}_{t}, \bar{z}_{t}; \bar{p}_{t}, \bar{q}_{t}, \bar{r}_{t}) \|^2 + \frac{\alpha_z\eta}{2} \|\nabla_{z} F(\bar{x}_{t}, \bar{y}_{t}, \bar{z}_{t}; \bar{p}_{t}, \bar{q}_{t}, \bar{r}_{t})\|^2  \\
        &  + \alpha_x \eta\| \nabla_{x} F(\bar{x}_{t}, \bar{y}_{t}, \bar{z}_{t}; \bar{p}_{t}, \bar{q}_{t}, \bar{r}_{t}) - \bar{u}_{t}\|^2 + \frac{\alpha_{y}\eta}{2} \|\nabla_{y}  F(\bar{x}_{t}, \bar{y}_{t}, \bar{z}_{t}; \bar{p}_{t}, \bar{q}_{t}, \bar{r}_{t}) -  \bar{v}_{t}\|^2 \notag \\
        & + \Big(4 \alpha_{x}^2\alpha_z\eta^3 L_{g}^2 -  \frac{\alpha_x \eta}{4} \Big) \| \bar{u}_{t}\|^2 + \Big( \frac{3\alpha_x\alpha_{y}^2 \eta^3L_{\rho}^2}{2} -  \frac{\alpha_{y}\eta}{4}\Big)  \|  \bar{v}_{t} \|^2 + \Big(\frac{ \alpha^2_z\eta^2 L_{F,2}}{2}+ \frac{3\alpha_z\eta}{4} \Big) \|\bar{w}_{t}\|^2 \ . \notag
	\end{align}
\end{lemma}

\begin{proof}
	At first, we decompose $F(\bar{x}_{t+1}, \bar{y}_{t+1}, \bar{z}_{t+1}; \bar{p}_{t+1}, \bar{q}_{t+1}, \bar{r}_{t+1}) - F(\bar{x}_{t}, \bar{y}_{t}, \bar{z}_{t}; \bar{p}_{t}, \bar{q}_{t}, \bar{r}_{t})$ as follows:
	\begin{align}
		& \quad F(\bar{x}_{t+1}, \bar{y}_{t+1}, \bar{z}_{t+1}; \bar{p}_{t+1}, \bar{q}_{t+1}, \bar{r}_{t+1}) - F(\bar{x}_{t}, \bar{y}_{t}, \bar{z}_{t}; \bar{p}_{t}, \bar{q}_{t}, \bar{r}_{t}) \notag \\
        & = \underbrace{F(\bar{x}_{t+1}, \bar{y}_{t+1}, \bar{z}_{t+1}; \bar{p}_{t+1}, \bar{q}_{t+1}, \bar{r}_{t+1})  - F(\bar{x}_{t+1}, \bar{y}_{t+1}, \bar{z}_{t+1}; \bar{p}_{t}, \bar{q}_{t}, \bar{r}_{t})}_{T_0} \notag \\
		& \quad + \underbrace{F(\bar{x}_{t+1}, \bar{y}_{t+1}, \bar{z}_{t+1}; \bar{p}_{t+1}, \bar{q}_{t+1}, \bar{r}_{t})  - F(\bar{x}_{t+1}, \bar{y}_{t+1}, \bar{z}_{t+1}; \bar{p}_{t}, \bar{q}_{t}, \bar{r}_{t})}_{T_1} \notag \\
        & \quad  + \underbrace{F(\bar{x}_{t+1}, \bar{y}_{t+1}, \bar{z}_{t+1}; \bar{p}_{t}, \bar{q}_{t}, \bar{r}_{t}) - F(\bar{x}_{t+1}, \bar{y}_{t+1}, \bar{z}_{t}; \bar{p}_{t}, \bar{q}_{t}, \bar{r}_{t})}_{T_2} \notag \\
		& \quad + \underbrace{ F(\bar{x}_{t+1}, \bar{y}_{t+1}, \bar{z}_{t}; \bar{p}_{t}, \bar{q}_{t}, \bar{r}_{t}) - F(\bar{x}_{t}, \bar{y}_{t+1}, \bar{z}_{t}; \bar{p}_{t}, \bar{q}_{t}, \bar{r}_{t})}_{T_3} \notag \\
        & \quad + \underbrace{F(\bar{x}_{t}, \bar{y}_{t+1}, \bar{z}_{t}; \bar{p}_{t}, \bar{q}_{t}, \bar{r}_{t})  - F(\bar{x}_{t}, \bar{y}_{t}, \bar{z}_{t}; \bar{p}_{t}, \bar{q}_{t}, \bar{r}_{t})}_{T_4} \ . \notag
	\end{align}
	Based on this decomposition, we bound $T_0$, $T_1$, $T_2$, $T_3$, and $T_4$, respectively. 

    Regarding $T_0$, it can be bounded as follows:
    \begin{align}
        T_0 & = -\frac{\gamma_2}{2}\|\bar{z}_{t+1}-\bar{r}_{t+1}\|^2 + \frac{\gamma_2}{2}\|\bar{z}_{t+1}-\bar{r}_{t}\|^2 = \frac{\gamma_2(1-(1-\alpha_r\eta)^2)}{2\alpha_r^2\eta^2}\|\bar{r}_{t+1}-\bar{r}_{t}\|^2  \ , \notag
    \end{align}
	
	Regarding $T_1$, it can be bounded as follows:
	\begin{align}
		T_1 & =\frac{\gamma_1}{2}\|\bar{x}_{t+1}-\bar{p}_{t+1}\|^2  + \frac{\gamma_1}{2}\|\bar{y}_{t+1}-\bar{q}_{t+1}\|^2  - \frac{\gamma_1}{2}\|\bar{x}_{t+1}-\bar{p}_{t}\|^2  - \frac{\gamma_1}{2}\|\bar{y}_{t+1}-\bar{q}_{t}\|^2  \notag \\
		& = \frac{\gamma_1}{2}\left(\|\bar{x}_{t+1}-\bar{p}_{t+1}\|^2-\|\bar{x}_{t+1}-\bar{p}_{t}\|^2 \right) + \frac{\gamma_1}{2}\left(\|\bar{y}_{t+1}-\bar{q}_{t+1}\|^2-\|\bar{y}_{t+1}-\bar{q}_{t}\|^2   \right) \notag \\
		& = \frac{\gamma_1(1-\alpha_{p}\eta)^2-1)}{2}\|\bar{x}_{t+1}-\bar{p}_{t}\|^2+ \frac{\gamma_1(1-\alpha_{q}\eta)^2-1)}{2}\|\bar{y}_{t+1}-\bar{q}_{t}\|^2    \notag \\
		& = \frac{\gamma_1(1-\alpha_{p}\eta)^2-1)}{2\alpha^2_{p}\eta^2} \|\bar{p}_{t+1}-\bar{p}_{t}\|^2 + \frac{\gamma_1(1-\alpha_{q}\eta)^2-1)}{2\alpha^2_{q}\eta^2}\|\bar{q}_{t+1}-\bar{q}_{t}\|^2    \notag \\
		& \leq  - \frac{\gamma_1}{2\alpha_{p}\eta} \|\bar{p}_{t+1}-\bar{p}_{t}\|^2 - \frac{\gamma_1}{2\alpha_{q}\eta}\|\bar{q}_{t+1}-\bar{q}_{t}\|^2   \ , \notag 
	\end{align}
	where the last step follows from $\eta<\frac{1}{\alpha_{p}}$ and $\eta<\frac{1}{\alpha_{q}}$. 

    	Regarding $T_2$, it can be bounded as follows:
	\begin{align}
		T_2 & \leq \langle \nabla_{z} F(\bar{x}_{t+1}, \bar{y}_{t+1}, \bar{z}_{t}; \bar{p}_{t}, \bar{q}_{t}, \bar{r}_{t}), \bar{z}_{t+1} - \bar{z}_{t} \rangle + \frac{L_{F,2}}{2} \|\bar{z}_{t+1} - \bar{z}_{t}\|^2 \notag \\
		& =  \alpha_z\eta \langle  \nabla_{z} F(\bar{x}_{t+1}, \bar{y}_{t+1}, \bar{z}_{t}; \bar{p}_{t}, \bar{q}_{t}, \bar{r}_{t}),\bar{w}_{t} \rangle + \frac{ \alpha^2_z\eta^2 L_{F,2}}{2} \|\bar{w}_{t}\|^2 \notag \\
		& = \alpha_z\eta \langle  \nabla_{z} F(\bar{x}_{t+1}, \bar{y}_{t+1}, \bar{z}_{t}; \bar{p}_{t}, \bar{q}_{t}, \bar{r}_{t}) - \nabla_{z} F(\bar{x}_{t}, \bar{y}_{t}, \bar{z}_{t}; \bar{p}_{t}, \bar{q}_{t}, \bar{r}_{t}),\bar{w}_{t} \rangle  \notag \\
        & \quad + \alpha_z\eta \langle \nabla_{z} F(\bar{x}_{t}, \bar{y}_{t}, \bar{z}_{t}; \bar{p}_{t}, \bar{q}_{t}, \bar{r}_{t}),\bar{w}_{t} \rangle  + \frac{ \alpha^2_z\eta^2 L_{F,2}}{2} \|\bar{w}_{t}\|^2 \notag \\
		& \leq 4 \alpha_z\eta \|\nabla_{z} F(\bar{x}_{t+1}, \bar{y}_{t+1}, \bar{z}_{t}; \bar{p}_{t}, \bar{q}_{t}, \bar{r}_{t}) - \nabla_{z} F(\bar{x}_{t}, \bar{y}_{t}, \bar{z}_{t}; \bar{p}_{t}, \bar{q}_{t}, \bar{r}_{t})\|^2 + \frac{\alpha_z\eta }{4}\|\bar{w}_{t}\|^2 \notag \\
		& \quad + \frac{\alpha_z\eta}{2} \|\nabla_{z} F(\bar{x}_{t}, \bar{y}_{t}, \bar{z}_{t}; \bar{p}_{t}, \bar{q}_{t}, \bar{r}_{t})\|^2 + \frac{\alpha_z\eta}{2}  \|\bar{w}_{t}\|^2 + \frac{ \alpha^2_z\eta^2 L_{F,2}}{2} \|\bar{w}_{t}\|^2 \notag \\
		& \leq 4 \alpha_{x}^2\alpha_z\eta^3 L_{g}^2  \| \bar{u}_{t}\|^2  + \frac{\alpha_z\eta}{2} \|\nabla_{z} F(\bar{x}_{t}, \bar{y}_{t}, \bar{z}_{t}; \bar{p}_{t}, \bar{q}_{t}, \bar{r}_{t})\|^2 + \left(\frac{ \alpha^2_z\eta^2 L_{F,2}}{2}+ \frac{3\alpha_z\eta}{4} \right) \|\bar{w}_{t}\|^2  \ , \notag
	\end{align}
    where the last step holds due to 
    \begin{align}
         & \quad \|\nabla_{z} F(\bar{x}_{t+1}, \bar{y}_{t+1}, \bar{z}_{t}; \bar{p}_{t}, \bar{q}_{t}, \bar{r}_{t}) - \nabla_{z} F(\bar{x}_{t}, \bar{y}_{t}, \bar{z}_{t}; \bar{p}_{t}, \bar{q}_{t}, \bar{r}_{t})\|^2  \notag \\
         & = \| \nabla_2 g(\bar{x}_{t+1}, \bar{z}_{t}) - \nabla_2 g(\bar{x}_{t}, \bar{z}_{t})\|^2  \leq L_{g}^2 \| \bar{x}_{t+1} - \bar{x}_{t}\|^2 \leq \alpha_{x}^2\eta^2L_{g}^2\|  \bar{u}_{t}\|^2 \ . \notag
    \end{align}

	Regarding $T_3$, it can be bounded as follows:
	\begin{align}
		T_3 & \leq \langle \nabla_{x} F(\bar{x}_{t}, \bar{y}_{t+1}, \bar{z}_{t}; \bar{p}_{t}, \bar{q}_{t}, \bar{r}_{t}), \bar{x}_{t+1} - \bar{x}_{t}\rangle + \frac{L_{F,1}}{2} \|\bar{x}_{t+1} - \bar{x}_{t}\|^2 \notag \\
		& = -\alpha_x \eta \langle \nabla_{x} F(\bar{x}_{t}, \bar{y}_{t+1}, \bar{z}_{t}; \bar{p}_{t}, \bar{q}_{t}, \bar{r}_{t}), \bar{u}_{t}\rangle + \frac{\alpha^2_x\eta^2 L_{F,1}}{2} \|\bar{u}_{t}\|^2 \notag \\
		& =  \frac{\alpha_x \eta}{2}\| \nabla_{x} F(\bar{x}_{t}, \bar{y}_{t+1}, \bar{z}_{t}; \bar{p}_{t}, \bar{q}_{t}, \bar{r}_{t}) - \bar{u}_{t}\|^2 -  \frac{\alpha_x \eta}{2}\| \nabla_{x} F(\bar{x}_{t}, \bar{y}_{t+1}, \bar{z}_{t}; \bar{p}_{t}, \bar{q}_{t}, \bar{r}_{t}) \|^2 \notag \\
        & \quad -  \frac{\alpha_x \eta}{2}\|  \bar{u}_{t}\|^2  +   \frac{\alpha^2_x\eta^2 L_{F,1}}{2}\|  \bar{u}_{t}\|^2 \notag \\
		& \leq - \frac{\alpha_x \eta}{4}\| \nabla_{x} F(\bar{x}_{t}, \bar{y}_{t}, \bar{z}_{t}; \bar{p}_{t}, \bar{q}_{t}, \bar{r}_{t}) \|^2  +  \frac{\alpha_x \eta}{2}\| \nabla_{x} F(\bar{x}_{t}, \bar{y}_{t+1}, \bar{z}_{t}; \bar{p}_{t}, \bar{q}_{t}, \bar{r}_{t}) - \bar{u}_{t}\|^2 -  \frac{\alpha_x \eta}{4}\|  \bar{u}_{t}\|^2  \notag \\
		& \quad +  \frac{\alpha_x \eta}{2}\|\nabla_{x} F(\bar{x}_{t}, \bar{y}_{t}, \bar{z}_{t}; \bar{p}_{t}, \bar{q}_{t}, \bar{r}_{t})  -  \nabla_{x} F(\bar{x}_{t}, \bar{y}_{t+1}, \bar{z}_{t}; \bar{p}_{t}, \bar{q}_{t}, \bar{r}_{t}) \|^2 \notag \\
		& \leq - \frac{\alpha_x \eta}{4}\| \nabla_{x} F(\bar{x}_{t}, \bar{y}_{t}, \bar{z}_{t}; \bar{p}_{t}, \bar{q}_{t}, \bar{r}_{t}) \|^2 +   \alpha_x \eta\| \nabla_{x} F(\bar{x}_{t}, \bar{y}_{t}, \bar{z}_{t}; \bar{p}_{t}, \bar{q}_{t}, \bar{r}_{t}) - \bar{u}_{t}\|^2 -  \frac{\alpha_x \eta}{4}\|  \bar{u}_{t}\|^2  \notag \\
		& \quad +\frac{3\alpha_x \eta}{2}\|\nabla_{x} F(\bar{x}_{t}, \bar{y}_{t}, \bar{z}_{t}; \bar{p}_{t}, \bar{q}_{t}, \bar{r}_{t})  -  \nabla_{x} F(\bar{x}_{t}, \bar{y}_{t+1}, \bar{z}_{t}; \bar{p}_{t}, \bar{q}_{t}, \bar{r}_{t}) \|^2 \notag \\
        & \leq - \frac{\alpha_x \eta}{4}\| \nabla_{x} F(\bar{x}_{t}, \bar{y}_{t}, \bar{z}_{t}; \bar{p}_{t}, \bar{q}_{t}, \bar{r}_{t}) \|^2  + \alpha_x \eta\| \nabla_{x} F(\bar{x}_{t}, \bar{y}_{t}, \bar{z}_{t}; \bar{p}_{t}, \bar{q}_{t}, \bar{r}_{t}) - \bar{u}_{t}\|^2 -  \frac{\alpha_x \eta}{4}\|  \bar{u}_{t}\|^2 \notag \\
        & \quad +  \frac{3\alpha_x\alpha_{y}^2 \eta^3L_{\rho}^2}{2}\| \bar{v}_{t} \|^2 \ , \notag
	\end{align}
	where the fourth step is due to $ -\|\nabla_{x} F(\bar{x}_{t}, \bar{y}_{t+1}, \bar{z}_{t}; \bar{p}_{t}, \bar{q}_{t}, \bar{r}_{t})\|^2\leq -\frac{1}{2}\| \nabla_{x} F(\bar{x}_{t}, \bar{y}_{t}, \bar{z}_{t}; \bar{p}_{t}, \bar{q}_{t}, \bar{r}_{t}) \|^2 + \|\nabla_{x} F(\bar{x}_{t}, \bar{y}_{t}, \bar{z}_{t}; \bar{p}_{t}, \bar{q}_{t}, \bar{r}_{t})  -  \nabla_{x} F(\bar{x}_{t}, \bar{y}_{t+1}, \bar{z}_{t}; \bar{p}_{t}, \bar{q}_{t}, \bar{r}_{t}) \|^2$ and $\eta \leq \frac{1}{2\alpha_{x}L_{F,1}}$.

	Regarding $T_4$, it can be bounded as follows:
	\begin{align}
		T_4 & \leq \langle  \nabla_{y}  F(\bar{x}_{t}, \bar{y}_{t}, \bar{z}_{t}; \bar{p}_{t}, \bar{q}_{t}, \bar{r}_{t}),  \bar{y}_{t+1} - \bar{y}_{t}\rangle + \frac{L_{F,1}}{2}\|\bar{y}_{t+1} - \bar{y}_{t}\|^2 \notag \\
		&  = -\alpha_{y}\eta  \langle  \nabla_{y}  F(\bar{x}_{t}, \bar{y}_{t}, \bar{z}_{t}; \bar{p}_{t}, \bar{q}_{t}, \bar{r}_{t}),  \bar{v}_{t}\rangle + \frac{\alpha^2_y\eta^2L_{F,1}}{2}\|\bar{v}_{t}\|^2 \notag \\
		& = \frac{\alpha_{y}\eta}{2} \|\nabla_{y}  F(\bar{x}_{t}, \bar{y}_{t}, \bar{z}_{t}; \bar{p}_{t}, \bar{q}_{t}, \bar{r}_{t}) -  \bar{v}_{t}\|^2 -   \frac{\alpha_{y}\eta}{2} \|\nabla_{y}  F(\bar{x}_{t}, \bar{y}_{t}, \bar{z}_{t}; \bar{p}_{t}, \bar{q}_{t}, \bar{r}_{t}) \|^2 \notag \\
        & \quad -  \frac{\alpha_{y}\eta}{2} \|  \bar{v}_{t}\|^2  + \frac{\alpha^2_y\eta^2L_{F,1}}{2}\|\bar{v}_{t}\|^2 \notag \\
		& \leq \frac{\alpha_{y}\eta}{2} \|\nabla_{y}  F(\bar{x}_{t}, \bar{y}_{t}, \bar{z}_{t}; \bar{p}_{t}, \bar{q}_{t}, \bar{r}_{t}) -  \bar{v}_{t}\|^2 -   \frac{\alpha_{y}\eta}{2} \|\nabla_{y}  F(\bar{x}_{t}, \bar{y}_{t}, \bar{z}_{t}; \bar{p}_{t}, \bar{q}_{t}, \bar{r}_{t}) \|^2 -  \frac{\alpha_{y}\eta}{4} \|  \bar{v}_{t}\|^2 \ , \notag 
	\end{align}
	where the last step is due to $\eta \leq \frac{1}{2\alpha_{y}L_{F,1}}$. 
	
	By combining $T_0$, $T_1$, $T_2$, $T_3$, and $T_4$, we complete the proof. 
	
\end{proof}

\begin{lemma}\label{lemma:F_d_smooth}
    Given Assumptions~\ref{assumption:f-smooth}-\ref{assumption:variance}, by defining
	\begin{align}
		& F_d(\bar{z}; \bar{p}, \bar{q}, \bar{r})= F(x^*(\bar{z}; \bar{p}, \bar{q}, \bar{r}), y^*(\bar{z}; \bar{p}, \bar{q}, \bar{r}), \bar{z}; \bar{p}, \bar{q}, \bar{r}) \ , \notag
	\end{align}
     $F_d(\bar{z}; \bar{p}, \bar{q}, \bar{r})$ is $L_d$-smooth, where $L_d = L_{\rho}(1 + \omega_2) + \gamma_2$, and the following inequality holds:
    \begin{align}
         & F_d(\bar{z}_{t+1}; \bar{p}_{t+1}, \bar{q}_{t+1}, \bar{r}_{t+1}) - F_d(\bar{z}_{t}; \bar{p}_{t}, \bar{q}_{t}, \bar{r}_{t})  \geq  \alpha_z\eta \langle \nabla_{z} F_d(\bar{z}_{t}; \bar{p}_{t}, \bar{q}_{t}, \bar{r}_{t} ),  \bar{w}_{t}\rangle  -\frac{\alpha^2_z\eta^2L_d}{2}\|w_{t}\|^2 \notag \\
        &  + \frac{\gamma_1}{2} \langle \bar{p}_{t+1}- \bar{p}_{t},  \bar{p}_{t+1}+ \bar{p}_{t}  -  2x^*(\bar{z}_{t+1}; \bar{p}_{t+1}, \bar{q}_{t+1}, \bar{r}_{t}) \rangle  \\
        &  + \frac{\gamma_1}{2}\langle  \bar{q}_{t+1}-  \bar{q}_{t},  \bar{q}_{t+1} +  \bar{q}_{t} -2y^*(\bar{z}_{t+1}; \bar{p}_{t+1}, \bar{q}_{t+1}, \bar{r}_{t}) \rangle + \frac{\gamma_2(1-(1-\alpha_r\eta)^2)}{2\alpha_r^2\eta^2}\|\bar{r}_{t+1}-\bar{r}_t\|^2 \ . \notag
    \end{align}
\end{lemma}

\begin{proof}
	At first, we decompose $F_d(\bar{z}_{t+1}; \bar{p}_{t+1}, \bar{q}_{t+1}, \bar{r}_{t+1}) - F_d(\bar{z}_{t}; \bar{p}_{t}, \bar{q}_{t}, \bar{r}_{t})$ as follows:
	\begin{align}
		&  F_d(\bar{z}_{t+1}; \bar{p}_{t+1}, \bar{q}_{t+1}, \bar{r}_{t+1}) - F_d(\bar{z}_{t}; \bar{p}_{t}, \bar{q}_{t}, \bar{r}_{t}) = \underbrace{F_d(\bar{z}_{t+1}; \bar{p}_{t+1}, \bar{q}_{t+1}, \bar{r}_{t+1}) - F_d(\bar{z}_{t+1}; \bar{p}_{t+1}, \bar{q}_{t+1}, \bar{r}_{t})}_{T_0} \notag \\
		&  + \underbrace{F_d(\bar{z}_{t+1}; \bar{p}_{t+1}, \bar{q}_{t+1}, \bar{r}_{t}) - F_d(\bar{z}_{t+1}; \bar{p}_{t}, \bar{q}_{t}, \bar{r}_{t})}_{T_{1}} + \underbrace{F_d(\bar{z}_{t+1}; \bar{p}_{t}, \bar{q}_{t}, \bar{r}_{t}) - F_d(\bar{z}_{t}; \bar{p}_{t}, \bar{q}_{t}, \bar{r}_{t})}_{T_2} \ . \notag 
	\end{align}
	In the following, we bound $T_{0}$, $T_{1}$ and $T_2$, respectively. 

    Regarding $T_0$, we derive
    \begin{align}
        T_0 = -\frac{\gamma_2}{2}\|\bar{z}_{t+1}-\bar{r}_{t+1}\|^2 + \frac{\gamma_2}{2}\|\bar{z}_{t+1}-\bar{r}_{t}\|^2= \frac{\gamma_2(1-(1-\alpha_r\eta)^2)}{2\alpha_r^2\eta^2}\|\bar{r}_{t+1}-\bar{r}_t\|^2  \ . \notag 
    \end{align}

    Regarding $T_1$, we derive:
    \begin{align}
        & \quad  T_1 = F(x^*(\bar{z}_{t+1}; \bar{p}_{t+1}, \bar{q}_{t+1}, \bar{r}_{t}), y^*(\bar{z}_{t+1}; \bar{p}_{t+1}, \bar{q}_{t+1}, \bar{r}_{t}), \bar{z}_{t+1}; \bar{p}_{t+1}, \bar{q}_{t+1}, \bar{r}_{t}) \notag \\
        & \qquad \quad  - F(x^*(\bar{z}_{t+1}; \bar{p}_{t}, \bar{q}_{t}, \bar{r}_{t}), y^*(\bar{z}_{t+1}; \bar{p}_{t}, \bar{q}_{t}, \bar{r}_{t}), \bar{z}_{t+1}; \bar{p}_{t}, \bar{q}_{t}, \bar{r}_{t}) \notag \\
        & \geq F(x^*(\bar{z}_{t+1}; \bar{p}_{t+1}, \bar{q}_{t+1}, \bar{r}_{t}), y^*(\bar{z}_{t+1}; \bar{p}_{t+1}, \bar{q}_{t+1}, \bar{r}_{t}), \bar{z}_{t+1}; \bar{p}_{t+1}, \bar{q}_{t+1}, \bar{r}_{t}) \notag \\
        & \quad - F(x^*(\bar{z}_{t+1}; \bar{p}_{t+1}, \bar{q}_{t+1}, \bar{r}_{t}), y^*(\bar{z}_{t+1}; \bar{p}_{t+1}, \bar{q}_{t+1}, \bar{r}_{t}), \bar{z}_{t+1}; \bar{p}_{t}, \bar{q}_{t}, \bar{r}_{t}) \notag \\
        & \geq \frac{\gamma_1}{2} \left( \|x^*(\bar{z}_{t+1}; \bar{p}_{t+1}, \bar{q}_{t+1}, \bar{r}_{t})-\bar{p}_{t+1} \|^2 -\|x^*(\bar{z}_{t+1}; \bar{p}_{t+1}, \bar{q}_{t+1}, \bar{r}_{t})-\bar{p}_{t} \|^2  \right) \notag \\
		& \quad  + \frac{\gamma_1}{2}\left(\| y^*(\bar{z}_{t+1}; \bar{p}_{t+1}, \bar{q}_{t+1}, \bar{r}_{t})- \bar{q}_{t+1}\|^2 - \| y^*(\bar{z}_{t+1}; \bar{p}_{t+1}, \bar{q}_{t+1}, \bar{r}_{t})- \bar{q}_{t}\|^2\right)  \notag \\
        & = \frac{\gamma_1}{2} \langle \bar{p}_{t+1}- \bar{p}_{t},  \bar{p}_{t+1}+ \bar{p}_{t}  -  2x^*(\bar{z}_{t+1}; \bar{p}_{t+1}, \bar{q}_{t+1}, \bar{r}_{t}) \rangle \notag \\
        & \quad + \frac{\gamma_1}{2}\langle  \bar{q}_{t+1}-  \bar{q}_{t},  \bar{q}_{t+1} +  \bar{q}_{t} -2y^*(\bar{z}_{t+1}; \bar{p}_{t+1}, \bar{q}_{t+1}, \bar{r}_{t}) \rangle \ . \notag 
    \end{align}
	Regarding $T_2$, it can be bounded as follows:
	\begin{align}
		T_2  & \geq   \langle \nabla_{z} F_d(\bar{z}_{t}; \bar{p}_{t}, \bar{q}_{t}, \bar{r}_{t}),  \bar{z}_{t+1}- \bar{z}_{t} \rangle  -\frac{L_d}{2}\|\bar{z}_{t+1}- \bar{z}_{t}\|^2\notag \\
		& = \alpha_z\eta \langle \nabla_{z} F_d(\bar{z}_{t}; \bar{p}_{t}, \bar{q}_{t}, \bar{r}_{t}),  \bar{w}_{t}\rangle  -\frac{\alpha^2_z\eta^2L_d}{2}\|w_{t}\|^2 \ . \notag 
	\end{align}
    By combining $T_0$, $T_1$, $T_2$, we complete the proof. 
\end{proof}

\begin{lemma}\label{lemma:h_smooth}
Given Assumptions~\ref{assumption:f-smooth}-\ref{assumption:variance}, by defining
	\begin{align}
		& H(p, q, r) =  \min_{x, y}\max_{z} F(x, y, z; p, q, r)  = \min_{x,y} F_p(x, y; p, q, r) \ , \notag \\
        & H_z(p, q) = \max_{r} H(p, q, r) \ , \quad  r^*(p, q) = \arg \max_{r} H(p, q, r) \ , \notag
	\end{align}
    and denoting $\hat{r} := r^*(\bar{p}_{t+1}, \bar{q}_{t+1})$, 
    the following inequality holds:
    \begin{align}
        H_{z}(\bar{p}_{t+1}, \bar{q}_{t+1}) - H_{z}(\bar{p}_{t}, \bar{q}_{t}) & \leq \frac{\gamma_1}{2} \langle \bar{p}_{t+1}- \bar{p}_{t},   \bar{p}_{t+1}+\bar{p}_{t} -2 x^*(\bar{p}_{t}, \bar{q}_{t}, \hat{r}) \rangle \notag \\
        & + \frac{\gamma_1}{2} \langle  \bar{q}_{t+1}- \bar{q}_{t},  \bar{q}_{t+1}+\bar{q}_{t} -2 y^*(\bar{p}_{t}, \bar{q}_{t}, \hat{r}) \rangle  \ . 
    \end{align}
\end{lemma}

\begin{proof}	
Since $H_z(\bar{p}_{t},\bar{q}_{t}) = H(\bar{p}_{t},\bar{q}_{t}, r^*(\bar{p}_{t},\bar{q}_{t}))$, denoting $\hat{r} := r^*(\bar{p}_{t+1}, \bar{q}_{t+1})$, we derive
    \begin{align}
		& \quad H_z(\bar{p}_{t+1},\bar{q}_{t+1}) - H_z(\bar{p}_t,\bar{q}_t)  = H(\bar{p}_{t+1},\bar{q}_{t+1},\hat{r}) - H(\bar{p}_t,\bar{q}_t, r^*(\bar{p}_{t},\bar{q}_{t})) \notag \\
		& \leq H(\bar{p}_{t+1},\bar{q}_{t+1},\hat{r}) - H(\bar{p}_t,\bar{q}_t, \hat{r}) \notag \\
        & = F_p(x^*(\bar{p}_{t+1},\bar{q}_{t+1},\hat r), y^*(\bar{p}_{t+1},\bar{q}_{t+1},\hat r); \bar{p}_{t+1},\bar{q}_{t+1},\hat r) - F_p(x^*(\bar{p}_t,\bar{q}_t,\hat r), y^*(\bar{p}_t,\bar{q}_t,\hat r);\bar{p}_t,\bar{q}_t,\hat r) \notag \\
		& \leq F_p(x^*(\bar{p}_{t},\bar{q}_{t},\hat r), y^*(\bar{p}_{t},\bar{q}_{t},\hat r); \bar{p}_{t+1},\bar{q}_{t+1},\hat r) - F_p(x^*(\bar{p}_t,\bar{q}_t,\hat r), y^*(\bar{p}_t,\bar{q}_t,\hat r);\bar{p}_t,\bar{q}_t,\hat r) \ . \notag
	\end{align}
    Hence, only the $(p,q)$-smoothing terms differ:
	\begin{align}
		& \quad H_{z}(\bar{p}_{t+1}, \bar{q}_{t+1}) - H_{z}(\bar{p}_{t}, \bar{q}_{t})  \leq \frac{\gamma_1}{2}\left(\|x^*(\bar{p}_{t}, \bar{q}_{t}, \hat{r}) - \bar{p}_{t+1}\|^2 - \|x^*(\bar{p}_{t}, \bar{q}_{t}, \hat{r}) - \bar{p}_{t}\|^2\right) \notag \\
		& \qquad \qquad  \quad + \frac{\gamma_1}{2}\left(\|y^*( \bar{p}_{t}, \bar{q}_{t}, \hat{r}) - \bar{q}_{t+1}\|^2 - \|y^*(\bar{p}_{t}, \bar{q}_{t}, \hat{r}) - \bar{q}_{t}\|^2\right) \notag \\
		& =\frac{\gamma_1}{2} \langle \bar{p}_{t+1}- \bar{p}_{t},   \bar{p}_{t+1}+\bar{p}_{t} -2 x^*(\bar{p}_{t}, \bar{q}_{t}, \hat{r}) \rangle + \frac{\gamma_1}{2} \langle  \bar{q}_{t+1}- \bar{q}_{t},  \bar{q}_{t+1}+\bar{q}_{t} -2 y^*(\bar{p}_{t}, \bar{q}_{t}, \hat{r}) \rangle \ . \notag
	\end{align}
\end{proof}

\begin{lemma}\label{lemma:potential_psi}
	Given Assumptions~\ref{assumption:f-smooth}-\ref{assumption:variance}, by defining 
    \begin{align}
		 \Psi_{t+1} & = F(\bar{x}_{t+1}, \bar{y}_{t+1}, \bar{z}_{t+1}; \bar{p}_{t+1}, \bar{q}_{t+1}, \bar{r}_{t+1}) - 2  F_d(\bar{z}_{t+1}; \bar{p}_{t+1}, \bar{q}_{t+1} , \bar{r}_{t+1}) \notag \\
         & \quad + 2H_{z}(\bar{p}_{t+1}, \bar{q}_{t+1}) \ , 
	\end{align}
    and setting $\alpha_{pq} =\alpha_{p} = \alpha_{q} $, 
    \begin{align}\label{eq:def_y_z_rho}
        & \alpha_z = \alpha_x\underbrace{\frac{(\gamma_1 - L_{\rho})^2}{72L^2_{g}}}_{c_{\alpha_z}=O(1)} \ , \quad \alpha_{pq} = \alpha_{x}\underbrace{\frac{(\gamma_1-L_{\rho})^4(\gamma_2-L_\rho)^2\mu}{72\times768\gamma_1L_{g}^2(5L_g^2\mu+16\gamma_2^2(\gamma_1-L_\rho))}}_{c_{\alpha_{pq}} = O(1/\kappa)} \ , \notag \\
        & \alpha_r = \alpha_{x}\underbrace{\frac{(\gamma_1-L_{\rho})^4(\gamma_2-L_\rho)^4}{72\times24\times640\gamma_1\gamma_2^2 L_g^4}}_{{c_{\alpha_r}=O(1)}} \ , \quad \alpha_{x} \leq \min\Big\{\frac{9}{80\gamma_1} \ , \frac{\sqrt{\mu(\gamma_1-L_{\rho})^3}(\gamma_2-L_{\rho})^2}{24\times16\sqrt{6\gamma_1c_{\alpha_{pq}}}L_gc_{\alpha_r}\gamma_2^2}\Big\} \ , \notag
    \end{align}
    and denoting $C_r = \frac{4\gamma_2^2}{\mu(\gamma_1-L_\rho)}\big(\omega_3^2+\frac{(1-\alpha_r\eta)^2}{\alpha_r^2\eta^2}\big)$, $C_\rho = 16\gamma_1\alpha_{pq}\eta\big(5\omega_2^2+20\omega_2^2\omega_3^2\alpha_r^2\eta^2+4C_r\alpha_r^2\eta^2\big)$, the following inequality holds:
        \begin{align}
        &  \Psi_{t+1} -  \Psi_{t} \leq  - \frac{\alpha_x \eta}{8} \| \nabla_{x} F(\bar{x}_{t}, \bar{y}_{t}, \bar{z}_{t}; \bar{p}_{t}, \bar{q}_{t}, \bar{r}_t) \|^2 - \frac{\alpha_{y}\eta}{2} \|\nabla_{y}  F(\bar{x}_{t}, \bar{y}_{t}, \bar{z}_{t}; \bar{p}_{t}, \bar{q}_{t}, \bar{r}_{t}) \|^2  \notag \\
        &  - \frac{\alpha_z\eta}{2} \|\nabla_{z} F(\bar{x}_{t}, \bar{y}_{t}, \bar{z}_{t}; \bar{p}_{t}, \bar{q}_{t}, \bar{r}_{t})\|^2 +   \Big( 16\gamma_1\alpha_{pq}\eta\big(5\omega_2^2\omega_3^2 + C_r\big) - \frac{\gamma_2}{2\alpha_r\eta}\Big) \|\bar{r}_{t+1}-\bar{r}_{t}\|^2 \notag \\
        &  + \Big(1 + \omega^2_1 + 8\alpha_{pq}\eta\omega_5^2\big( 5\omega_2^2 +\frac{2\gamma_2^2}{\mu(\gamma_1-L_\rho)}\big)- \frac{1}{4\alpha_{pq}\eta}\Big)\gamma_1  \big(\|\bar{p}_{t+1}- \bar{p}_{t}\|^2 + \|\bar{q}_{t+1}-\bar{q}_t\|^2\big) \notag \\
        & + \alpha_x \eta\| \nabla_{x} F(\bar{x}_{t}, \bar{y}_{t}, \bar{z}_{t}; \bar{p}_{t}, \bar{q}_{t}, \bar{r}_{t}) - \bar{u}_{t}\|^2  + \frac{\alpha_{y}\eta}{2} \|\nabla_{y}  F(\bar{x}_{t}, \bar{y}_{t}, \bar{z}_{t}; \bar{p}_{t}, \bar{q}_{t}, \bar{r}_{t}) -  \bar{v}_{t}\|^2 \notag \\
        & + \frac{3\alpha_z\eta }{2}\| \nabla_{z} F(\bar{x}_{t}, \bar{y}_{t},  \bar{z}_{t}; \bar{p}_{t}, \bar{q}_{t}, \bar{r}_t) -   \bar{w}_{t}\|^2  + \Big(4 \alpha_{x}^2\alpha_z\eta^3 L_{g}^2 -  \frac{\alpha_x \eta}{4} \Big) \| \bar{u}_{t}\|^2 \notag \\
        &  + \Big(  \frac{3\alpha_x\alpha_{y}^2 \eta^3L_{\rho}^2}{2} -  \frac{\alpha_{y}\eta}{4}\Big)  \|  \bar{v}_{t} \|^2 + \Big(\frac{ \alpha^2_z\eta^2 }{2}(L_{F,2}+2L_d+ 2C_\rho)  - \frac{\alpha_z\eta }{16} \Big) \|\bar{w}_{t}\|^2 \ . 
    \end{align}
\end{lemma}

\begin{proof}
Firstly, based on Lemmas~\ref{lemma:F_smooth}, ~\ref{lemma:F_d_smooth}, and \ref{lemma:h_smooth}, we have
    \begin{align}
        &  \Psi_{t+1} -  \Psi_{t} \leq - \frac{\gamma_1}{2\alpha_{p}\eta} \|\bar{p}_{t+1}-\bar{p}_{t}\|^2 - \frac{\gamma_1}{2\alpha_{q}\eta}\|\bar{q}_{t+1}-\bar{q}_{t}\|^2  - \frac{\gamma_2}{2\alpha_r\eta}\|\bar{r}_{t+1}-\bar{r}_{t}\|^2 \notag \\
        & - \frac{\alpha_x \eta}{4}\| \nabla_{x} F(\bar{x}_{t}, \bar{y}_{t}, \bar{z}_{t}; \bar{p}_{t}, \bar{q}_{t}, \bar{r}_{t}) \|^2  - \frac{\alpha_{y}\eta}{2} \|\nabla_{y}  F(\bar{x}_{t}, \bar{y}_{t}, \bar{z}_{t}; \bar{p}_{t}, \bar{q}_{t}, \bar{r}_{t}) \|^2  \notag \\
        & + \frac{\alpha_z\eta}{2} \|\nabla_{z} F(\bar{x}_{t}, \bar{y}_{t}, \bar{z}_{t}; \bar{p}_{t}, \bar{q}_{t}, \bar{r}_{t})\|^2 + \alpha_x \eta\| \nabla_{x} F(\bar{x}_{t}, \bar{y}_{t}, \bar{z}_{t}; \bar{p}_{t}, \bar{q}_{t}, \bar{r}_{t}) - \bar{u}_{t}\|^2 \notag \\
        &   + \frac{\alpha_{y}\eta}{2} \|\nabla_{y}  F(\bar{x}_{t}, \bar{y}_{t}, \bar{z}_{t}; \bar{p}_{t}, \bar{q}_{t}, \bar{r}_{t}) -  \bar{v}_{t}\|^2  + \Big(4 \alpha_{x}^2\alpha_z\eta^3 L_{g}^2 -  \frac{\alpha_x \eta}{4} \Big) \| \bar{u}_{t}\|^2  \notag \\
        &  + \Big(  \frac{3\alpha_x\alpha_{y}^2 \eta^3L_{\rho}^2}{2} -  \frac{\alpha_{y}\eta}{4}\Big)  \|  \bar{v}_{t} \|^2 + \Big(\frac{ \alpha^2_z\eta^2 L_{F,2}}{2}+ \frac{3\alpha_z\eta}{4} + \alpha^2_z\eta^2L_d\Big) \|\bar{w}_{t}\|^2  \notag \\
        &   - 2\alpha_z\eta \langle \nabla_{z} F_d(\bar{z}_{t}; \bar{p}_{t}, \bar{q}_{t}, \bar{r}_{t}),  \bar{w}_{t}\rangle \notag \\
        & + 2\gamma_1\langle \bar{p}_{t+1}- \bar{p}_{t},  \  x^*(\bar{z}_{t+1}; \bar{p}_{t+1}, \bar{q}_{t+1}, \bar{r}_{t}) - x^*(\bar{p}_{t}, \bar{q}_{t}, \hat{r}) \rangle \notag \\
        & + 2\gamma_1 \langle  \bar{q}_{t+1}- \bar{q}_{t}, \ y^*(\bar{z}_{t+1}; \bar{p}_{t+1}, \bar{q}_{t+1}, \bar{r}_{t}) - y^*(\bar{p}_{t}, \bar{q}_{t}, \hat{r}) \rangle  \ . 
    \end{align}
    In particular, we obtain
    	\begin{align}
		&\quad  -2 \langle \nabla_{z} F_d(\bar{z}_{t}; \bar{p}_{t}, \bar{q}_{t}, \bar{r}_{t}),  \bar{w}_{t}\rangle \notag \\
        & = -2 \langle \nabla_{z} F(x^*(\bar{z}_{t}; \bar{p}_{t}, \bar{q}_{t}, \bar{r}_{t}), y^*(\bar{z}_{t}; \bar{p}_{t}, \bar{q}_{t}, \bar{r}_{t}),  \bar{z}_{t}; \bar{p}_{t}, \bar{q}_{t}, \bar{r}_{t}),  \bar{w}_{t}\rangle \notag \\
		& = -2 \langle \nabla_{z} F(x^*(\bar{z}_{t}; \bar{p}_{t}, \bar{q}_{t}, \bar{r}_{t}), y^*(\bar{z}_{t}; \bar{p}_{t}, \bar{q}_{t}, \bar{r}_{t}),  \bar{z}_{t}; \bar{p}_{t}, \bar{q}_{t}, \bar{r}_{t}) -  \nabla_{z} F(\bar{x}_{t}, \bar{y}_{t},  \bar{z}_{t}; \bar{p}_{t}, \bar{q}_{t}, \bar{r}_{t}),  \bar{w}_{t}\rangle \notag \\
        & \quad -2 \langle \nabla_{z} F(\bar{x}_{t}, \bar{y}_{t},  \bar{z}_{t}; \bar{p}_{t}, \bar{q}_{t}, \bar{r}_{t}),  \bar{w}_{t}\rangle \notag \\
		& \leq 8 \|\nabla_{z} F(x^*(\bar{z}_{t}; \bar{p}_{t}, \bar{q}_{t}, \bar{r}_{t}), y^*(\bar{z}_{t}; \bar{p}_{t}, \bar{q}_{t}, \bar{r}_{t}),  \bar{z}_{t}; \bar{p}_{t}, \bar{q}_{t}, \bar{r}_{t}) -  \nabla_{z} F(\bar{x}_{t}, \bar{y}_{t},  \bar{z}_{t}; \bar{p}_{t}, \bar{q}_{t}, \bar{r}_{t})\|^2  \notag \\
		& \quad + \frac{1}{8} \|\bar{w}_{t}\|^2 + \| \nabla_{z} F(\bar{x}_{t}, \bar{y}_{t},  \bar{z}_{t}; \bar{p}_{t}, \bar{q}_{t}, \bar{r}_{t}) -   \bar{w}_{t}\|^2 - \| \nabla_{z} F(\bar{x}_{t}, \bar{y}_{t},  \bar{z}_{t}; \bar{p}_{t}, \bar{q}_{t}, \bar{r}_{t}) \|^2 - \|   \bar{w}_{t}\|^2 \notag \\
        & \leq \frac{8 L^2_{g}}{(\gamma_1-L_{\rho})^2}\|\nabla_{x} F(\bar{x}_{t}, \bar{y}_{t},  \bar{z}_{t}; \bar{p}_{t}, \bar{q}_{t})\|^2 +  \| \nabla_{z} F(\bar{x}_{t}, \bar{y}_{t},  \bar{z}_{t}; \bar{p}_{t}, \bar{q}_{t}) -   \bar{w}_{t}\|^2 \notag \\
        & \quad  - \| \nabla_{z} F(\bar{x}_{t}, \bar{y}_{t},  \bar{z}_{t}; \bar{p}_{t}, \bar{q}_{t}) \|^2 - \frac{7}{8}\|   \bar{w}_{t}\|^2 \ . 
        \end{align}
    In addition, we derive
    \begin{align}\label{eq:nabla_z_F_d}
    	& \quad 2\gamma_1 \langle \bar{p}_{t+1}- \bar{p}_{t},     x^*(\bar{z}_{t+1}; \bar{p}_{t+1}, \bar{q}_{t+1}, \bar{r}_{t})-x^*(\bar{p}_{t}, \bar{q}_{t}, \hat{r}) \rangle \notag  \\
    	& = 2\gamma_1 \langle \bar{p}_{t+1}- \bar{p}_{t},     x^*(\bar{z}_{t+1}; \bar{p}_{t+1}, \bar{q}_{t+1}, \bar{r}_{t}) - x^*(\bar{p}_{t+1}, \bar{q}_{t+1}, \hat{r}) \rangle \notag \\
    	& \quad + 2\gamma_1 \langle \bar{p}_{t+1}- \bar{p}_{t},   x^*(\bar{p}_{t+1}, \bar{q}_{t+1}, \hat{r})-x^*(\bar{p}_{t}, \bar{q}_{t}, \hat{r}) \rangle \notag \\
    	& \leq \frac{\gamma_1}{4\alpha_{p}\eta} \|\bar{p}_{t+1}- \bar{p}_{t}\|^2 + 4\gamma_1\alpha_{p}\eta \|x^*(\bar{z}_{t+1}; \bar{p}_{t+1}, \bar{q}_{t+1}, \bar{r}_{t}) - x^*( \bar{p}_{t+1}, \bar{q}_{t+1}, \hat{r}) \|^2  \notag \\
    	& \quad + \gamma_1\|\bar{p}_{t+1}- \bar{p}_{t}\|^2 + \gamma_1 \|x^*(\bar{p}_{t+1}, \bar{q}_{t+1}, \hat{r})-x^*(\bar{p}_{t}, \bar{q}_{t}, \hat{r})\|^2  \ . 
    \end{align}
    Similarly, we derive
    \begin{align}
	& \quad 2\gamma_1 \langle  \bar{q}_{t+1}- \bar{q}_{t}, y^*(\bar{z}_{t+1}; \bar{p}_{t+1}, \bar{q}_{t+1}, \bar{r}_{t})-y^*(\bar{p}_{t}, \bar{q}_{t}, \hat{r}) \rangle \notag \\
	& \leq \frac{\gamma_1}{4\alpha_{q}\eta} \| \bar{q}_{t+1}- \bar{q}_{t}\|^2 + 4\gamma_1\alpha_{q}\eta \| y^*(\bar{z}_{t+1}; \bar{p}_{t+1}, \bar{q}_{t+1}, \bar{r}_{t})- y^*(\bar{p}_{t+1}, \bar{q}_{t+1}, \hat{r}) \|^2 \notag \\
	& \quad + \gamma_1 \|  \bar{q}_{t+1}- \bar{q}_{t}\|^2 +  \gamma_1  \|y^*( \bar{p}_{t+1}, \bar{q}_{t+1}, \hat{r}) -y^*( \bar{p}_{t}, \bar{q}_{t}, \hat{r}) \|^2 \ . \notag
\end{align}
Combining these two inequalities, and applying Lemma~\ref{lemma:x_y_star_z_continuity} on the last two terms, we obtain
    \begin{align}\label{eq:gamma_p_q_xy_star}
	& \quad 2\gamma_1 \langle \bar{p}_{t+1}- \bar{p}_{t},     x^*(\bar{z}_{t+1}; \bar{p}_{t+1}, \bar{q}_{t+1}, \bar{r}_{t})-x^*(\bar{p}_{t}, \bar{q}_{t}, \hat{r}) \rangle \notag  \\
	& + 2\gamma_1 \langle  \bar{q}_{t+1}- \bar{q}_{t}, y^*(\bar{z}_{t+1}; \bar{p}_{t+1}, \bar{q}_{t+1}, \bar{r}_{t})-y^*(\bar{p}_{t}, \bar{q}_{t}, \hat{r}) \rangle \notag \\
	& \leq \frac{\gamma_1}{4\alpha_{p}\eta} \|\bar{p}_{t+1}- \bar{p}_{t}\|^2 + 4\gamma_1\alpha_{p}\eta \|x^*(\bar{z}_{t+1}; \bar{p}_{t+1}, \bar{q}_{t+1}, \bar{r}_{t}) - x^*( \bar{p}_{t+1}, \bar{q}_{t+1}, \hat{r}) \|^2 \notag \\
	& \quad + \frac{\gamma_1}{4\alpha_{q}\eta} \| \bar{q}_{t+1}- \bar{q}_{t}\|^2 + 4\gamma_1\alpha_{q}\eta \| y^*(\bar{z}_{t+1}; \bar{p}_{t+1}, \bar{q}_{t+1}, \bar{r}_{t})- y^*(\bar{p}_{t+1}, \bar{q}_{t+1}, \hat{r}) \|^2 \notag \\
	& \quad + \gamma_1\|\bar{p}_{t+1}- \bar{p}_{t}\|^2 + \gamma_1 \omega_1^2 \| \bar{p}_{t+1} - \bar{p}_{t}\|^2 + \gamma_1 \|  \bar{q}_{t+1}- \bar{q}_{t}\|^2 +  \gamma_1 \omega_1^2 \| \bar{q}_{t+1}-\bar{q}_{t} \|^2 \notag \\
	& = \Big(\frac{1}{4\alpha_{p}\eta} +1 + \omega^2_1 \Big)\gamma_1 \|\bar{p}_{t+1}- \bar{p}_{t}\|^2 +  \Big(\frac{1}{4\alpha_{q}\eta} +1 + \omega^2_1 \Big)\gamma_1 \|\bar{q}_{t+1}- \bar{q}_{t}\|^2 \notag \\
	& \quad + 4\gamma_1\alpha_{p}\eta \|x^*(\bar{z}_{t+1}; \bar{p}_{t+1}, \bar{q}_{t+1}, \bar{r}_{t}) - x^*( \bar{p}_{t+1}, \bar{q}_{t+1}, \hat{r}) \|^2 \notag \\
	& \quad + 4\gamma_1\alpha_{q}\eta \| y^*(\bar{z}_{t+1}; \bar{p}_{t+1}, \bar{q}_{t+1}, \bar{r}_{t})- y^*(\bar{p}_{t+1}, \bar{q}_{t+1}, \hat{r}) \|^2 \ . 
\end{align}

Applying Lemma~\ref{lemma:xy_xy_star_r} and Lemma~\ref{lemma:xy_z_star_xy_star_r}, we obtain
\begin{align}\label{eq:xy_z_star_xy_star}
    & \quad  \|x^*(\bar{z}_{t+1}; \bar{p}_{t+1}, \bar{q}_{t+1}, \bar{r}_{t}) - x^*(\bar{p}_{t+1}, \bar{q}_{t+1}, \hat{r}) \|^2 + \|y^*(\bar{z}_{t+1}; \bar{p}_{t+1}, \bar{q}_{t+1}, \bar{r}_{t}) - y^*(\bar{p}_{t+1}, \bar{q}_{t+1}, \hat{r}) \|^2 \notag \\
    & \leq 2\|x^*(\bar{z}_{t+1}; \bar{p}_{t+1}, \bar{q}_{t+1}, \bar{r}_{t}) - x^*(\bar{p}_{t+1}, \bar{q}_{t+1}, {r}^+) \|^2 + 2\|y^*(\bar{z}_{t+1}; \bar{p}_{t+1}, \bar{q}_{t+1}, \bar{r}_{t}) - y^*(\bar{p}_{t+1}, \bar{q}_{t+1}, {r}^+) \|^2 \notag \\
    & \quad + 2\| x^*(\bar{p}_{t+1}, \bar{q}_{t+1}, {r}^+) - x^*( \bar{p}_{t+1}, \bar{q}_{t+1}, \hat{r}) \|^2 + 2\| y^*(\bar{p}_{t+1}, \bar{q}_{t+1}, {r}^+) - y^*( \bar{p}_{t+1}, \bar{q}_{t+1}, \hat{r}) \|^2 \notag \\
    & \leq \Big[20\omega_2^2\alpha_z^2\eta^2\big(1+2\omega_3^2\alpha_r^2\eta^2\big) + \frac{20\omega_2^2\big(1+4\omega_3^2\alpha_r^2\eta^2\big)}{(\gamma_2-L_\rho)^2}+ 8C_r\alpha_r^2\eta^2\Big(\alpha_z^2\eta^2+ \frac{2}{(\gamma_2-L_\rho)^2}\Big)\Big]\|\bar{w}_t\|^2 \notag \\
        & \quad + \Big[\frac{60\omega_2^2\alpha_z^2\eta^2L_g^2}{(\gamma_1-L_\rho)^2}+\frac{40\omega_2^2L_g^2\big(1+4\omega_3^2\alpha_r^2\eta^2\big) +32C_r\alpha_r^2\eta^2L_g^2}{(\gamma_1-L_\rho)^2(\gamma_2-L_\rho)^2}\Big]\|\nabla_xF(\bar{x}_t,\bar{y}_t,\bar{z}_t;\bar{p}_t,\bar{q}_t,\bar{r}_t)\|^2 \notag \\
        & \quad + \Big[60\omega_2^2\alpha_z^2\eta^2+\frac{40\omega_2^2\big(1+4\omega_3^2\alpha_r^2\eta^2\big)+32C_r\alpha_r^2\eta^2}{(\gamma_2-L_\rho)^2} \Big]\|\nabla_zF(\bar{x}_t,\bar{y}_t,\bar{z}_t;\bar{p}_t,\bar{q}_t,\bar{r}_t)-\bar{w}_t\|^2 \notag \\
        & \quad + \Big(10\omega_2^2\omega_5^2+ \frac{4\omega_5^2\gamma_2^2}{\mu(\gamma_1-L_\rho)}\Big)(\|\bar{p}_{t+1}-\bar{p}_t\|^2+\|\bar{q}_{t+1}-\bar{q}_t\|^2) + \Big(20\omega_2^2\omega_3^2+4C_r\Big)\|\bar{r}_{t+1}-\bar{r}_t\|^2  \ . \notag 
\end{align}
Combining Eqs.~(\ref{eq:nabla_z_F_d}), ~(\ref{eq:gamma_p_q_xy_star}), ~(\ref{eq:xy_z_star_xy_star}), and denoting $\alpha_{pq}  = \alpha_{p} =\alpha_{q}$, $C_\rho = 16\gamma_1\alpha_{pq}\eta\big(5\omega_2^2+20\omega_2^2\omega_3^2\alpha_r^2\eta^2+4C_r\alpha_r^2\eta^2\big)$,  we obtain
    \begin{align}
        &  \Psi_{t+1} -  \Psi_{t} \leq   \Big(1 + \omega^2_1 + 8\alpha_{pq}\eta\omega_5^2\big( 5\omega_2^2 +\frac{2\gamma_2^2}{\mu(\gamma_1-L_\rho)}\big)- \frac{1}{4\alpha_{pq}\eta}\Big)\gamma_1  \big(\|\bar{p}_{t+1}- \bar{p}_{t}\|^2 + \|\bar{q}_{t+1}-\bar{q}_t\|^2\big) \notag \\
        & +   \Big( 16\gamma_1\alpha_{pq}\eta\big(5\omega_2^2\omega_3^2 + C_r\big) - \frac{\gamma_2}{2\alpha_r\eta}\Big) \|\bar{r}_{t+1}-\bar{r}_{t}\|^2 \notag \\
        & + \Big[\frac{L_g^2}{(\gamma_1-L_{\rho})^2}\Big(8\alpha_z\eta+ 240\gamma_1\alpha_{pq}\eta\alpha_z^2\eta^2\omega_2^2 + \frac{2C_\rho}{(\gamma_2-L_\rho)^2}\Big) - \frac{\alpha_x \eta}{4} \Big] \| \nabla_{x} F(\bar{x}_{t}, \bar{y}_{t}, \bar{z}_{t}; \bar{p}_{t}, \bar{q}_{t}, \bar{r}_t) \|^2 \notag \\
        & - \frac{\alpha_{y}\eta}{2} \|\nabla_{y}  F(\bar{x}_{t}, \bar{y}_{t}, \bar{z}_{t}; \bar{p}_{t}, \bar{q}_{t}, \bar{r}_{t}) \|^2  - \frac{\alpha_z\eta}{2} \|\nabla_{z} F(\bar{x}_{t}, \bar{y}_{t}, \bar{z}_{t}; \bar{p}_{t}, \bar{q}_{t}, \bar{r}_{t})\|^2 \notag \\
        & + \alpha_x \eta\| \nabla_{x} F(\bar{x}_{t}, \bar{y}_{t}, \bar{z}_{t}; \bar{p}_{t}, \bar{q}_{t}, \bar{r}_{t}) - \bar{u}_{t}\|^2  + \frac{\alpha_{y}\eta}{2} \|\nabla_{y}  F(\bar{x}_{t}, \bar{y}_{t}, \bar{z}_{t}; \bar{p}_{t}, \bar{q}_{t}, \bar{r}_{t}) -  \bar{v}_{t}\|^2 \notag \\
        & + \Big(\alpha_z\eta + 240\gamma_1\alpha_{pq}\eta\alpha_z^2\eta^2\omega_2^2 + \frac{2C_\rho}{(\gamma_2-L_\rho)^2}\Big) \| \nabla_{z} F(\bar{x}_{t}, \bar{y}_{t},  \bar{z}_{t}; \bar{p}_{t}, \bar{q}_{t}, \bar{r}_t) -   \bar{w}_{t}\|^2 \notag \\
        &   + \Big(4 \alpha_{x}^2\alpha_z\eta^3 L_{g}^2 -  \frac{\alpha_x \eta}{4} \Big) \| \bar{u}_{t}\|^2 + \Big(  \frac{3\alpha_x\alpha_{y}^2 \eta^3L_{\rho}^2}{2} -  \frac{\alpha_{y}\eta}{4}\Big)  \|  \bar{v}_{t} \|^2 \notag \\
        &  + \Big(\frac{ \alpha^2_z\eta^2}{2}(L_{F,2}+2L_d + 2C_\rho ) + \frac{C_\rho }{(\gamma_2-L_\rho)^2} - \frac{\alpha_z\eta }{8} \Big) \|\bar{w}_{t}\|^2 \ . 
    \end{align}
    
For the coefficient of $\| \nabla_{x} F(\bar{x}_{t}, \bar{y}_{t}, \bar{z}_{t}; \bar{p}_{t}, \bar{q}_{t}, \bar{r}_{t}) \|^2$, we enforce
\begin{align}
    & \frac{L_g^2}{(\gamma_1-L_{\rho})^2}\Big(8\alpha_z\eta+ 240\gamma_1\alpha_{pq}\eta\alpha_z^2\eta^2\omega_2^2 + \frac{2C_\rho}{(\gamma_2-L_\rho)^2}\Big) - \frac{\alpha_x \eta}{4} \leq - \frac{\alpha_x \eta}{8}  \ . \notag 
\end{align}
In particular, from $C_r = \frac{4\gamma_2^2}{\mu(\gamma_1-L_\rho)}\big(\omega_3^2+\frac{(1-\alpha_r\eta)^2}{\alpha_r^2\eta^2}\big)$, and $C_\rho = 16\gamma_1\alpha_{pq}\eta\big(5\omega_2^2+20\omega_2^2\omega_3^2\alpha_r^2\eta^2+4C_r\alpha_r^2\eta^2\big)$, we enforce
\begin{align}\label{eq:nabla_xF_alpha_x}
    & \frac{8\alpha_z\eta L_g^2}{(\gamma_1-L_{\rho})^2} \leq  \frac{\alpha_x \eta}{9} \ ,   \quad \frac{240\gamma_1\alpha_{pq}\eta\alpha_z^2\eta^2\omega_2^2L_g^2}{(\gamma_1-L_{\rho})^2} \leq  \frac{\alpha_x \eta}{24\times8} \ , \notag \\
    & \frac{32\gamma_1\alpha_{pq}\eta L_g^2}{(\gamma_1-L_{\rho})^2}\Big(\frac{5\omega_2^2}{(\gamma_2-L_\rho)^2} + \frac{4(1-\alpha_r\eta)^2}{(\gamma_2-L_\rho)^2}\frac{4\gamma_2^2}{\mu(\gamma_1-L_\rho)}\Big) \leq  \frac{\alpha_x \eta}{72\times24} \ , \notag \\
    &  \frac{32\gamma_1\alpha_{pq}\eta L_g^2}{(\gamma_1-L_{\rho})^2}\frac{20\omega_2^2\omega_3^2\alpha_r^2\eta^2}{(\gamma_2-L_\rho)^2} \leq  \frac{\alpha_x \eta}{72\times24} \ , \notag \\
    & \frac{32\gamma_1\alpha_{pq}\eta L_g^2}{(\gamma_1-L_{\rho})^2}\frac{4\alpha_r^2\eta^2}{(\gamma_2-L_\rho)^2} \frac{4\gamma_2^2\omega_3^2}{\mu(\gamma_1-L_\rho)}\leq  \frac{\alpha_x \eta}{72\times24} \ . 
\end{align}
By solving these inequalities, we obtain
\begin{align}
    & \alpha_z = \alpha_x\underbrace{\frac{(\gamma_1 - L_{\rho})^2}{72L^2_{g}}}_{c_{\alpha_z}=O(1)} \ , \quad \alpha_{pq} = \alpha_{x}\underbrace{\frac{(\gamma_1-L_{\rho})^4(\gamma_2-L_\rho)^2\mu}{72\times768\gamma_1L_{g}^2(5L_g^2\mu+16\gamma_2^2(\gamma_1-L_\rho))}}_{c_{\alpha_{pq}} = O(1/\kappa)} \ , \notag \\
    & \alpha_r = \alpha_{x}\underbrace{\frac{(\gamma_1-L_{\rho})^4(\gamma_2-L_\rho)^4}{72\times24\times640\gamma_1\gamma_2^2 L_g^4}}_{{c_{\alpha_r}=O(1)}} \ , \quad \alpha_{x} \leq \min\Big\{\frac{9}{80\gamma_1} \ , \frac{\sqrt{\mu(\gamma_1-L_{\rho})^3}(\gamma_2-L_{\rho})^2}{24\times16\sqrt{6\gamma_1c_{\alpha_{pq}}}L_gc_{\alpha_r}\gamma_2^2}\Big\} \ . \notag
\end{align}
In addition, for the coefficient of $\| \nabla_{z} F(\bar{x}_{t}, \bar{y}_{t},  \bar{z}_{t}; \bar{p}_{t}, \bar{q}_{t}, \bar{r}_t) -   \bar{w}_{t}\|^2$ and $ \|\bar{w}_{t}\|^2$, we obtain
\begin{align}
    & \quad \frac{C_\rho}{(\gamma_2-L_\rho)^2} \leq \frac{\alpha_{x}\eta(\gamma_1 - L_{\rho})^2}{72\times16L_{g}^2} = \frac{\alpha_z\eta}{16}  \ , \notag \\
    & \quad \alpha_z\eta + 240\gamma_1\alpha_{pq}\eta\alpha_z^2\eta^2\omega_2^2 + \frac{2C_\rho}{(\gamma_2-L_\rho)^2} \notag \\
    & \leq  \alpha_z\eta  + \frac{3\alpha_{x}\eta(\gamma_1 - L_{\rho})^2}{72
    \times8L_{g}^2} + \frac{\alpha_{x}\eta(\gamma_1 - L_{\rho})^2}{72
    \times8L_{g}^2} \leq \frac{3\alpha_z\eta }{2} \ . \notag 
\end{align}
Hence, the proof is completed.
\end{proof}

\section{Federated Lemmas}\label{app:fed}
\begin{lemma}\label{lemma:incremental-x-y-z}
    Given Assumptions~\ref{assumption:f-smooth}-\ref{assumption:variance}, we have the following inequalities:
    \begin{align}
        &  \frac{1}{K}\sum_{k=1}^{K}\mathbb{E}[\| x^{(k)}_{t+1} - x^{(k)}_{t} \|^2] \leq 6\alpha_x^2\eta^2\rho^2 \frac{1}{K}\sum_{k=1}^{K}\mathbb{E}[\| u^{(k)}_{1, t} - \bar{u}_{1, t} \|^2] + 6\alpha_x^2\eta^2 \frac{1}{K}\sum_{k=1}^{K}\mathbb{E}[\| u^{(k)}_{2, t} - \bar{u}_{2, t} \|^2] \notag \\
        & \quad + 6\alpha_x^2\eta^2 \frac{1}{K}\sum_{k=1}^{K}\mathbb{E}[\| u^{(k)}_{3, t} - \bar{u}_{3, t} \|^2]  + 6\alpha_x^2\eta^2 \mathbb{E}[\| \bar{u}_{t} \|^2] \ ,  \notag
    \end{align}
    \begin{align}
        &  \frac{1}{K}\sum_{k=1}^{K}\mathbb{E}[\| y^{(k)}_{t+1} - y^{(k)}_{t} \|^2] \leq 4\alpha_y^2\eta^2\rho^2 \frac{1}{K}\sum_{k=1}^{K}\mathbb{E}[\| v^{(k)}_{1, t} - \bar{v}_{1, t} \|^2] + 4\alpha_y^2\eta^2 \frac{1}{K}\sum_{k=1}^{K}\mathbb{E}[\| v^{(k)}_{2, t} - \bar{v}_{2, t} \|^2] \notag \\
        & \quad + 4\alpha_y^2\eta^2  \mathbb{E}[\| \bar{v}_{t} \|^2] \ ,  \notag
    \end{align}
    \begin{align}
        &  \frac{1}{K}\sum_{k=1}^{K}\mathbb{E}[\| z^{(k)}_{t+1} - z^{(k)}_{t} \|^2]  \leq  2\alpha_z^2\eta^2 \frac{1}{K}\sum_{k=1}^{K}\mathbb{E}[\|    w^{(k)}_{t} -  \bar{w}_{t}   \|^2] + 2\alpha_z^2\eta^2 \mathbb{E}[\|   \bar{w}_{t}  \|^2] \ .  \notag
    \end{align}
\end{lemma}

\begin{proof}
    Due to ${x}^{(k)}_{t+1} = x^{(k)}_{t} -\alpha_x \eta u^{(k)}_{t}$, we have
    \begin{align} 
        & \quad \frac{1}{K}\sum_{k=1}^{K}\mathbb{E}[\| x^{(k)}_{t+1} - x^{(k)}_{t} \|^2]  = \frac{1}{K}\sum_{k=1}^{K}\mathbb{E}[\| \alpha_x \eta u^{(k)}_{t} \|^2] \notag \\
        & \leq 2\alpha_x^2\eta^2 \frac{1}{K}\sum_{k=1}^{K}\mathbb{E}[\| u^{(k)}_{t} - \bar{u}_{t} \|^2] + 2\alpha_x^2\eta^2 \mathbb{E}[\| \bar{u}_{t} \|^2] \notag \\
        & = 2\alpha_x^2\eta^2 \frac{1}{K}\sum_{k=1}^{K}\mathbb{E}[\| \rho u^{(k)}_{1, t}  + (u^{(k)}_{2, t} - u^{(k)}_{3, t}) - \rho\bar{u}_{1, t}  -  (\bar{u}_{2, t} - \bar{u}_{3, t}) \|^2]  + 2\alpha_x^2\eta^2 \mathbb{E}[\| \bar{u}_{t} \|^2] \notag \\
        & \leq 6\alpha_x^2\eta^2\rho^2 \frac{1}{K}\sum_{k=1}^{K}\mathbb{E}[\| u^{(k)}_{1, t} - \bar{u}_{1, t} \|^2] + 6\alpha_x^2\eta^2 \frac{1}{K}\sum_{k=1}^{K}\mathbb{E}[\| u^{(k)}_{2, t} - \bar{u}_{2, t} \|^2]  \notag \\
        & \quad + 6\alpha_x^2\eta^2 \frac{1}{K}\sum_{k=1}^{K}\mathbb{E}[\| u^{(k)}_{3, t} - \bar{u}_{3, t} \|^2] + 6\alpha_x^2\eta^2 \mathbb{E}[\| \bar{u}_{t} \|^2] \ . \notag
    \end{align}
    Similarly, due to $y^{(k)}_{t+1} = y^{(k)}_{t} - \alpha_y\eta v^{(k)}_{t}$,  we have
    \begin{align}
        & \quad \frac{1}{K}\sum_{k=1}^{K}\mathbb{E}[\| y^{(k)}_{t+1} - y^{(k)}_{t} \|^2] = \frac{1}{K}\sum_{k=1}^{K}\mathbb{E}[\|  \alpha_y\eta v^{(k)}_{t}  \|^2] \notag \\
        & \leq 2\alpha_y^2\eta^2 \frac{1}{K}\sum_{k=1}^{K}\mathbb{E}[\| v^{(k)}_{t} - \bar{v}_{t} \|^2] + 2\alpha_y^2\eta^2 \mathbb{E}[\| \bar{v}_{t} \|^2] \notag \\
        & = 2\alpha_y^2\eta^2 \frac{1}{K}\sum_{k=1}^{K}\mathbb{E}[\| \rho v^{(k)}_{1, t}  +  v^{(k)}_{2, t} - \rho\bar{v}_{1, t}  -  \bar{v}_{2, t} \|^2] + 2\alpha_y^2\eta^2  \mathbb{E}[\| \bar{v}_{t} \|^2] \notag \\
        & \leq 4\alpha_y^2\eta^2\rho^2 \frac{1}{K}\sum_{k=1}^{K}\mathbb{E}[\| v^{(k)}_{1, t} - \bar{v}_{1, t} \|^2] + 4\alpha_y^2\eta^2 \frac{1}{K}\sum_{k=1}^{K}\mathbb{E}[\| v^{(k)}_{2, t} - \bar{v}_{2, t} \|^2] + 4\alpha_y^2\eta^2  \mathbb{E}[\| \bar{v}_{t} \|^2] \ .  \notag
    \end{align}
    Moreover, due to $z^{(k)}_{t+1} = z^{(k)}_{t} - \alpha_z\eta w^{(k)}_{t}$,  we have
    \begin{align}
        & \quad \frac{1}{K}\sum_{k=1}^{K}\mathbb{E}[\| z^{(k)}_{t+1} - z^{(k)}_{t} \|^2]  = \frac{1}{K}\sum_{k=1}^{K}\mathbb{E}[\|  \alpha_z\eta w^{(k)}_{t}  \|^2] \notag \\
        & \leq 2\alpha_z^2\eta^2 \frac{1}{K}\sum_{k=1}^{K}\mathbb{E}[\| w^{(k)}_{t} - \bar{w}_{t} \|^2] + 2\alpha_z^2\eta^2 \mathbb{E}[\| \bar{w}_{t} \|^2]  \ . \notag
    \end{align}
\end{proof}

\begin{lemma} \label{lemma:consensus-error-variables}
   Given Assumptions~\ref{assumption:f-smooth}-\ref{assumption:variance}, by denoting $s_{t}=\lfloor t/\tau\rfloor$, we have the following inequalities:
    \begin{align}
		& \frac{1}{T}\sum_{t=0}^{T-1} \frac{1}{K}\sum_{k=1}^{K}\mathbb{E}[\| x^{(k)}_{t} - \bar{x}_{t} \|^2]  \leq 3 \tau^2\alpha_x^2\eta^2 \Bigg(\rho^2 \frac{1}{T}\sum_{t=0}^{T-1}\frac{1}{K}\sum_{k=1}^{K}\mathbb{E}[\| u^{(k)}_{1, t} -\bar{u}_{1, t} \|^2]  \notag \\
        & \quad \qquad + \frac{1}{T}\sum_{t=0}^{T-1} \frac{1}{K}\sum_{k=1}^{K}\mathbb{E}[\| u^{(k)}_{2, t} -\bar{u}_{2, t} \|^2] + \frac{1}{T}\sum_{t=0}^{T-1}\frac{1}{K}\sum_{k=1}^{K}\mathbb{E}[\| u^{(k)}_{3, t} -\bar{u}_{3, t} \|^2]\Bigg) \ ,  \notag 
	\end{align}
    \begin{align}
		  \frac{1}{T}\sum_{t=0}^{T-1}\frac{1}{K}\sum_{k=1}^{K}\mathbb{E}[\| y^{(k)}_{t} - \bar{y}_{t} \|^2] &  \leq 2\tau^2\alpha_y^2\eta^2 \Bigg(\rho^2\frac{1}{T}\sum_{t=0}^{T-1}\frac{1}{K}\sum_{k=1}^{K}\mathbb{E}[\| v^{(k)}_{1, t} -\bar{v}_{1, t} \|^2] \notag \\
          & \quad \qquad + \frac{1}{T}\sum_{t=0}^{T-1} \frac{1}{K}\sum_{k=1}^{K}\mathbb{E}[\| v^{(k)}_{2, t} -\bar{v}_{2, t} \|^2] \Bigg) \ , \notag 
	\end{align}
    \begin{align}
		 \frac{1}{T}\sum_{t=0}^{T-1} \frac{1}{K}\sum_{k=1}^{K}\mathbb{E}[\| z^{(k)}_{t} - \bar{z}_{t} \|^2]  \leq  \tau^2\alpha_z^2\eta^2 \frac{1}{T}\sum_{t=0}^{T-1}  \frac{1}{K}\sum_{k=1}^{K}\mathbb{E}[\|   w^{(k)}_{t}   -  \bar{w}_{t}\|^2] \ . \notag 
	\end{align}
\end{lemma}

\begin{proof}
By denoting $s_{t}=\lfloor t/\tau\rfloor$, we have 
    \begin{align}
		& \quad \frac{1}{K}\sum_{k=1}^{K}\mathbb{E}[\| x^{(k)}_{t} - \bar{x}_{t} \|^2] \notag \\
		& = \frac{1}{K}\sum_{k=1}^{K}\mathbb{E}[\| x^{(k)}_{s_{t} \tau} - \alpha_x\eta\sum_{t'=s_{t} \tau}^{t-1}u^{(k)}_{t'} - \bar{x}_{s_{t} \tau} + \alpha_x\eta\sum_{t'=s_{t} \tau}^{t-1}\bar{u}_{t'} \|^2] \notag \\
		& = \frac{1}{K}\sum_{k=1}^{K}\mathbb{E}[\|  - \alpha_x\eta\sum_{t'=s_{t} \tau}^{t-1}u^{(k)}_{t'}  + \alpha_x\eta\sum_{t'=s_{t} \tau}^{t-1}\bar{u}_{t'} \|^2] \notag \\
		& \leq \tau\alpha_x^2\eta^2 \sum_{t'=s_{t} \tau}^{t-1} \frac{1}{K}\sum_{k=1}^{K}\mathbb{E}[\|  - u^{(k)}_{t'}  + \bar{u}_{t'} \|^2] \notag \\
		& = \tau\alpha_x^2\eta^2 \sum_{t'=s_{t} \tau}^{t-1} \frac{1}{K}\sum_{k=1}^{K}\mathbb{E}[\|  \rho u^{(k)}_{1, t'}  +(u^{(k)}_{2, t'} - u^{(k)}_{3, t'})  - \rho \bar{u}_{1, t'}  -  (\bar{u}_{2, t'} -\bar{u}_{3, t'}) \|^2] \notag \\
		& \leq 3 \tau\alpha_x^2\eta^2\rho^2 \sum_{t'=s_{t} \tau}^{t-1} \frac{1}{K}\sum_{k=1}^{K}\mathbb{E}[\| u^{(k)}_{1, t'} -\bar{u}_{1, t'} \|^2] + 3 \tau\alpha_x^2\eta^2\sum_{t'=s_{t} \tau}^{t-1} \frac{1}{K}\sum_{k=1}^{K}\mathbb{E}[\| u^{(k)}_{2, t'} -\bar{u}_{2, t'} \|^2] \notag \\
        & \quad + 3 \tau\alpha_x^2\eta^2\sum_{t'=s_{t} \tau}^{t-1} \frac{1}{K}\sum_{k=1}^{K}\mathbb{E}[\| u^{(k)}_{3, t'} -\bar{u}_{3, t'} \|^2] \ .  \notag
	\end{align}
    The other two inequalities can be proved in the same way. Therefore, we omit their detailed proof. 

\end{proof}

\begin{lemma}\label{lemma:grad_orig}
Given Assumptions~\ref{assumption:f-smooth}-\ref{assumption:variance}, by setting $\eta\leq \min\{\frac{1}{\sqrt{\beta_x}}, \frac{1}{\sqrt{\beta_y}}, \frac{1}{\sqrt{\beta_z}}\}$,  we have the following inequalities:
	\begin{align}
	\mathbb{E}[\|\mathcal{G}_{u, 1, t+1} \|^2]  & \leq (1-\beta_x\eta^2) \mathbb{E}[\|  \mathcal{G}_{u, 1, t}\|^2] + 2L_{f}^{2} \frac{1}{K^2}\sum_{k=1}^{K}\mathbb{E}[\| x^{(k)}_{t+1} - x^{(k)}_{t} \|^2] \notag \\
	&  \quad + 2L_{f}^{2} \frac{1}{K^2}\sum_{k=1}^{K}\mathbb{E}[\| y^{(k)}_{t+1} - y^{(k)}_{t} \|^2] + 2\beta_x^2\eta^4\sigma^2 \frac{1}{K} \  , \notag 
\end{align}
\begin{align}
	\mathbb{E}[\|\mathcal{G}_{v, 1, t+1} \|^2]  & \leq (1-\beta_y\eta^2) \mathbb{E}[\|\mathcal{G}_{v, 1, t} \|^2] +  2L_{f}^{2} \frac{1}{K^2}\sum_{k=1}^{K}\mathbb{E}[\| x^{(k)}_{t+1} - x^{(k)}_{t} \|^2]  \notag \\
	& \quad   + 2L_{f}^{2} \frac{1}{K^2}\sum_{k=1}^{K}\mathbb{E}[\| y^{(k)}_{t+1} - y^{(k)}_{t} \|^2]  + 2\beta_y^2\eta^4\sigma^2 \frac{1}{K} \ , \notag 
\end{align}
\begin{align}
	\mathbb{E}[\|\mathcal{G}_{u, 3, t+1} \|^2] & \leq (1-\beta_x\eta^2) \mathbb{E}[\| \mathcal{G}_{u, 3, t} \|^2]   + 2L_{g}^{2} \frac{1}{K^2}\sum_{k=1}^{K}\mathbb{E}[\| x^{(k)}_{t+1} - x^{(k)}_{t} \|^2]  \notag \\
	&  \quad  + 2L_{g}^{2} \frac{1}{K^2}\sum_{k=1}^{K}\mathbb{E}[\| z^{(k)}_{t+1} - z^{(k)}_{t} \|^2] + 2\beta_x^2\eta^4\sigma^2 \frac{1}{K} \ , \notag 
\end{align}
\end{lemma}

\begin{proof}
    According to the update rule of ${u}^{(k)}_{1, t+1} $, we have
    \begin{align}\label{eq:global-grad-var-1}
		& \quad \mathbb{E}[\|\mathcal{G}_{u, 1, t+1} \|^2] = \mathbb{E}[\|\frac{1}{K}\sum_{k=1}^{K}\nabla_{1} f^{(k)}(x^{(k)}_{t+1}, y^{(k)}_{t+1})- \frac{1}{K}\sum_{k=1}^{K}{u}^{(k)}_{1, t+1} \|^2]   \notag \\ 
		& = \mathbb{E}[\| \frac{1}{K}\sum_{k=1}^{K} ((1-\beta_x\eta^2)(u^{(k)}_{1, t} - \nabla_1 f^{(k)}(x^{(k)}_{t}, y^{(k)}_{t}; \xi^{(k)}_{t+1})) + \nabla_1 f^{(k)}(x^{(k)}_{t+1}, y^{(k)}_{t+1}; \xi^{(k)}_{t+1}))  \notag \\
        & \qquad \quad - \frac{1}{K}\sum_{k=1}^{K}\nabla_{1} f^{(k)}(x^{(k)}_{t+1}, y^{(k)}_{t+1}) \|^2] \notag \\
		& \leq (1-\beta_x\eta^2)^{2} \mathbb{E}[\| \frac{1}{K}\sum_{k=1}^{K}u^{(k)}_{1, t} - \nabla_1 f^{(k)}(x^{(k)}_{t}, y^{(k)}_{t}) \|^2] \notag \\
		& \quad + 2 \frac{1}{K^2}\sum_{k=1}^{K} \mathbb{E}[\| \nabla_1 f^{(k)}(x^{(k)}_{t+1}, y^{(k)}_{t+1}; \xi^{(k)}_{t+1}) - \nabla_1 f^{(k)}(x^{(k)}_{t}, y^{(k)}_{t}; \xi^{(k)}_{t+1}) \|^2] \notag \\
		& \quad + 2\beta_x^2\eta^4 \frac{1}{K^2}\sum_{k=1}^{K}\mathbb{E}[\| \nabla_1 f^{(k)}(x^{(k)}_{t}, y^{(k)}_{t}; \xi^{(k)}_{t+1}) - \nabla_1 f^{(k)}(x^{(k)}_{t}, y^{(k)}_{t})  \|^2] \notag \\
		& \leq (1-\beta_x\eta^2)  \mathbb{E}[\|\mathcal{G}_{u, 1, t} \|^2] + 2L_{f}^{2} \frac{1}{K^2}\sum_{n=1}^{N}\mathbb{E}[\| x^{(k)}_{t+1} - x^{(k)}_{t} \|^2]  \notag \\
		& \quad + 2L_{f}^{2} \frac{1}{K^2}\sum_{n=1}^{N}\mathbb{E}[\| y^{(k)}_{t+1} - y^{(k)}_{t} \|^2] + 2\beta_x^2\eta^4\sigma^2 \frac{1}{K} \ .
	\end{align} 
    Following the proof of Eq.~(\ref{eq:global-grad-var-1}), we can prove the other inequalities. Therefore, we omit the detailed proof to save space. 
\end{proof}

\begin{lemma}\label{lemma:grad_new}
Given Assumptions~\ref{assumption:f-smooth}-\ref{assumption:variance}, by setting $\eta\leq \min\{\frac{1}{\sqrt{\beta_x}}, \frac{1}{\sqrt{\beta_y}}\}$,  we have the following inequalities:
    \begin{align}
	\mathbb{E}[\|\mathcal{G}_{u, 2, t+1} \|^2] & \leq (1-\beta_x\eta^2) \mathbb{E}[\| \mathcal{G}_{u, 2, t}  \|^2] + 2L_{g}^{2} \frac{1}{K^2}\sum_{k=1}^{K}\mathbb{E}[\| x^{(k)}_{t+1} - x^{(k)}_{t} \|^2] \notag \\
	&  \quad  + 2L_{g}^{2} \frac{1}{K^2}\sum_{k=1}^{K}\mathbb{E}[\| y^{(k)}_{t+1} - y^{(k)}_{t} \|^2] + 2\beta_x^2\eta^4\sigma^2 \frac{1}{K} \ , \notag 
\end{align}
\begin{align}
	\mathbb{E}[\|\mathcal{G}_{v, 2, t+1} \|^2] & \leq (1-\beta_y\eta^2) \mathbb{E}[\| \mathcal{G}_{v, 2, t} \|^2] + 2L_{g}^{2} \frac{1}{K^2}\sum_{k=1}^{K}\mathbb{E}[\| x^{(k)}_{t+1} - x^{(k)}_{t} \|^2] \notag \\
	& \quad  + 2L_{g}^{2} \frac{1}{K^2}\sum_{k=1}^{K}\mathbb{E}[\| y^{(k)}_{t+1} - y^{(k)}_{t} \|^2] + 2\beta_y^2\eta^4\sigma^2 \frac{1}{K}  \ . \notag 
\end{align}
\begin{align}
	\mathbb{E}[\|\mathcal{G}_{w, t+1} \|^2] & \leq (1-\beta_z\eta^2) \mathbb{E}[\| \mathcal{G}_{w, t} \|^2]   + 2L_{g}^{2} \frac{1}{K^2}\sum_{k=1}^{K}\mathbb{E}[\| x^{(k)}_{t+1} - x^{(k)}_{t} \|^2] \notag \\
	&  \quad  + 2L_{g}^{2} \frac{1}{K^2}\sum_{k=1}^{K}\mathbb{E}[\| z^{(k)}_{t+1} - z^{(k)}_{t} \|^2] + 2\beta_z^2\eta^4\sigma^2 \frac{1}{K} \  . \notag
\end{align}
\end{lemma}

\begin{proof}
According to the update rule of $u^{(k)}_{2, t+1}$, we have
    \begin{align}
		& \quad \mathbb{E}[\|\mathcal{G}_{u, 2, t+1} \|^2] = \mathbb{E}[\|\frac{1}{K}\sum_{k=1}^{K}\nabla_1 g_{\gamma_1}^{(k)}(x^{(k)}_{t+1}, y^{(k)}_{t+1}; p^{(k)}_{t+1}, q^{(k)}_{t+1}))- \frac{1}{K}\sum_{k=1}^{K} u^{(k)}_{2, t+1} \|^2]    \notag \\ 
		& = \mathbb{E}[\| \frac{1}{K}\sum_{k=1}^{K} ((1-\beta_x\eta^2)(u^{(k)}_{2, t} - \nabla_1 g_{\gamma_1}^{(k)}(x^{(k)}_{t}, y^{(k)}_{t}; p^{(k)}_{t},q^{(k)}_{t}; \zeta^{(k)}_{t+1})) \notag \\
        & \qquad + \nabla_1 g_{\gamma_1}^{(k)}(x^{(k)}_{t+1}, y^{(k)}_{t+1}; p^{(k)}_{t+1}, q^{(k)}_{t+1}; \zeta^{(k)}_{t+1}))  - \frac{1}{K}\sum_{k=1}^{K}\nabla_{1} g_{\gamma_1}^{(k)}(x^{(k)}_{t+1}, y^{(k)}_{t+1}; p^{(k)}_{t+1}, q^{(k)}_{t+1}) \|^2] \notag \\
		& \leq (1-\beta_x\eta^2)^2\mathbb{E}[\| \bar{u}_{2, t} - \frac{1}{K}\sum_{k=1}^{K}\nabla_1 g_{\gamma_1}^{(k)}(x^{(k)}_{t}, y^{(k)}_{t}; p^{(k)}_{t}, q^{(k)}_{t}) \|^2] \notag \\
        & \quad + 2\frac{1}{K^2}\sum_{k=1}^{K}\mathbb{E}[\|  \nabla_1 g_{\gamma_1}^{(k)}(x^{(k)}_{t}, y^{(k)}_{t}; p^{(k)}_{t}, q^{(k)}_{t})  - \nabla_1 g_{\gamma_1}^{(k)}(x^{(k)}_{t}, y^{(k)}_{t}; p^{(k)}_{t}, q^{(k)}_{t}; \zeta^{(k)}_{t+1}) \notag \\
		& \qquad \qquad + \nabla_1 g_{\gamma_1}^{(k)}(x^{(k)}_{t+1}, y^{(k)}_{t+1}; p^{(k)}_{t+1}, q^{(k)}_{t+1}; \zeta^{(k)}_{t+1}) - \nabla_{1} g_{\gamma_1}^{(k)}(x^{(k)}_{t+1}, y^{(k)}_{t+1}; p^{(k)}_{t+1}, q^{(k)}_{t+1})\|^2]  \notag \\
		& \quad + 2\beta_x^2\eta^4\frac{1}{K^2}\sum_{k=1}^{K} \mathbb{E}[\| \nabla_1 g_{\gamma_1}^{(k)}(x^{(k)}_{t}, y^{(k)}_{t}; p^{(k)}_{t}, q^{(k)}_{t}; \zeta^{(k)}_{t+1}) - \nabla_1 g_{\gamma_1}^{(k)}(x^{(k)}_{t}, y^{(k)}_{t}; p^{(k)}_{t}, q^{(k)}_{t}) \|^2] \ . \notag 
    \end{align}
    Follows Assumption~\ref{assumption:variance}, the last term can be bounded by $2\beta_x^2\eta^4\sigma^2 \frac{1}{K} $. 
    
    In addition the second term can be bounded as follows
    \begin{align}
        & \quad  2\frac{1}{K^2}\sum_{k=1}^{K}\mathbb{E}[\|  \nabla_1 g_{\gamma_1}^{(k)}(x^{(k)}_{t}, y^{(k)}_{t}; p^{(k)}_{t}, q^{(k)}_{t})  - \nabla_1 g_{\gamma_1}^{(k)}(x^{(k)}_{t}, y^{(k)}_{t}; p^{(k)}_{t}, q^{(k)}_{t}; \zeta^{(k)}_{t+1}) \notag  \\
		& \qquad \qquad  + \nabla_1 g_{\gamma_1}^{(k)}(x^{(k)}_{t+1}, y^{(k)}_{t+1}; p^{(k)}_{t+1}, q^{(k)}_{t+1}; \zeta^{(k)}_{t+1}) - \nabla_{1} g_{\gamma_1}^{(k)}(x^{(k)}_{t+1}, y^{(k)}_{t+1}; p^{(k)}_{t+1}, q^{(k)}_{t+1})\|^2]  \notag \\
        & =  2\frac{1}{K^2}\sum_{k=1}^{K}\mathbb{E}[\|  \nabla_1 g^{(k)}(x^{(k)}_{t}, y^{(k)}_{t}) + \gamma_1(x^{(k)}_{t}-p^{(k)}_{t}) - \nabla_1 g^{(k)}(x^{(k)}_{t}, y^{(k)}_{t}; \zeta^{(k)}_{t+1}) - \gamma_1(x^{(k)}_{t}-p^{(k)}_{t})\notag \\
		& \qquad \qquad  + \nabla_1 g^{(k)}(x^{(k)}_{t+1}, y^{(k)}_{t+1}; \zeta^{(k)}_{t+1}) + \gamma_1(x^{(k)}_{t+1}-p^{(k)}_{t+1})  - \nabla_{1} g^{(k)}(x^{(k)}_{t+1}, y^{(k)}_{t+1})  - \gamma_1(x^{(k)}_{t+1}-p^{(k)}_{t+1})  \|^2]  \notag \\
        & =  2\frac{1}{K^2}\sum_{k=1}^{K}\mathbb{E}[\|  \nabla_1 g^{(k)}(x^{(k)}_{t}, y^{(k)}_{t}) - \nabla_1 g^{(k)}(x^{(k)}_{t}, y^{(k)}_{t}; \zeta^{(k)}_{t+1}) \notag \\
        & \qquad \qquad + \nabla_1 g^{(k)}(x^{(k)}_{t+1}, y^{(k)}_{t+1}; \zeta^{(k)}_{t+1})  - \nabla_{1} g^{(k)}(x^{(k)}_{t+1}, y^{(k)}_{t+1}) \|^2] \notag \\
        & \leq  2\frac{1}{K^2}\sum_{k=1}^{K}\mathbb{E}[\|  \nabla_1 g^{(k)}(x^{(k)}_{t+1}, y^{(k)}_{t+1}; \zeta^{(k)}_{t+1}) - \nabla_1 g^{(k)}(x^{(k)}_{t}, y^{(k)}_{t}; \zeta^{(k)}_{t+1})\|^2] \notag \\
        & \leq 2L_g^2\frac{1}{K^2}\sum_{k=1}^{K}\mathbb{E}[\| x^{(k)}_{t+1} - x^{(k)}_{t} \|^2]  + 2L_{g}^{2} \frac{1}{K^2}\sum_{k=1}^{K}\mathbb{E}[\| y^{(k)}_{t+1} - y^{(k)}_{t} \|^2] \ . \notag 
    \end{align}
Combining these inequalities completes the proof.
    
Using a similar argument, we can establish the other inequalities. Hence, we omit the details.
\end{proof}

\begin{lemma}\label{lemma:consensus-error-momentum-each-step-org}
Given Assumptions~\ref{assumption:f-smooth}-\ref{assumption:heterogeneous}, by setting $\eta\leq \min\{\frac{1}{\sqrt{\beta_x}}, \frac{1}{\sqrt{\beta_y}}\}$,  we have the following inequalities: 
	\begin{align}
	    (i). \ & \frac{1}{K}\sum_{k=1}^{K}\mathbb{E}[\|u^{(k)}_{1, t+1} -\bar{u}_{1, t+1}\|^2]  \leq (1+\frac{1}{\tau}+24\tau L_{f}^2\alpha_x^2\eta^2\rho^2)\frac{1}{K}\sum_{k=1}^{K}\mathbb{E}[\|u^{(k)}_{1, t} - \bar{u}_{1, t}\|^2] \notag \\
    	&  +   24\tau L_{f}^2\alpha_x^2\eta^2\frac{1}{K}\sum_{k=1}^{K}\mathbb{E}[\| u^{(k)}_{2, t} - \bar{u}_{2, t} \|^2] + 24\tau L_{f}^2\alpha_x^2\eta^2 \frac{1}{K}\sum_{k=1}^{K}\mathbb{E}[\| u^{(k)}_{3, t} - \bar{u}_{3, t} \|^2] \notag \\
        & + 16\tau L_{f}^2\alpha_y^2\eta^2\rho^2 \frac{1}{K}\sum_{k=1}^{K}\mathbb{E}[\| v^{(k)}_{1, t} - \bar{v}_{1, t} \|^2] + 16\tau L_{f}^2\alpha_y^2\eta^2 \frac{1}{K}\sum_{k=1}^{K}\mathbb{E}[\| v^{(k)}_{2, t} - \bar{v}_{2, t} \|^2] \notag \\
        & + 24\tau \beta_x^2\eta^4 L_{f}^{2} \frac{1}{K}\sum_{k=1}^{K}\mathbb{E}[\| x^{(k)}_{t} - \bar{x}_{t} \|^2] + 24\tau \beta_x^2\eta^4 L_{f}^{2} \frac{1}{K}\sum_{k=1}^{K}\mathbb{E}[\| y^{(k)}_{t} - \bar{y}_{t} \|^2]  \notag \\
    	&  + 24\tau L_{f}^2\alpha_x^2\eta^2 \mathbb{E}[\| \bar{u}_{t} \|^2]   + 16\tau L_{f}^2\alpha_y^2\eta^2 \mathbb{E}[\| \bar{v}_{t} \|^2] + 12\tau \beta_x^2\eta^4 \delta^2 + 8\tau \beta_x^2\eta^4  \sigma^{2} \ , \notag
	\end{align}
    \begin{align}
	  (ii). \  & \frac{1}{K}\sum_{k=1}^{K}\mathbb{E}[\|v^{(k)}_{1, t+1} -\bar{v}_{1, t+1}\|^2]  \leq (1+\frac{1}{\tau}+16\tau L_{f}^2\alpha_y^2\eta^2\rho^2 )\frac{1}{K}\sum_{k=1}^{K}\mathbb{E}[\|v^{(k)}_{1, t} - \bar{v}_{1, t}\|^2] \notag \\
    &  + 24\tau L_{f}^2\alpha_x^2\eta^2\rho^2 \frac{1}{K}\sum_{k=1}^{K}\mathbb{E}[\| u^{(k)}_{1, t} - \bar{u}_{1, t} \|^2]   + 24\tau L_{f}^2\alpha_x^2\eta^2 \frac{1}{K}\sum_{k=1}^{K}\mathbb{E}[\| u^{(k)}_{2, t} - \bar{u}_{2, t} \|^2]\notag \\
	&  + 24\tau L_{f}^2\alpha_x^2\eta^2 \frac{1}{K}\sum_{k=1}^{K}\mathbb{E}[\| u^{(k)}_{3, t} - \bar{u}_{3, t} \|^2] + 16\tau L_{f}^2\alpha_y^2\eta^2 \frac{1}{K}\sum_{k=1}^{K}\mathbb{E}[\| v^{(k)}_{2, t} - \bar{v}_{2, t} \|^2]    \notag \\
	& + 24\tau \beta_y^2\eta^4 L_{f}^{2} \frac{1}{K}\sum_{k=1}^{K}\mathbb{E}[\| x^{(k)}_{t} - \bar{x}_{t} \|^2] + 24\tau \beta_y^2\eta^4 L_{f}^{2} \frac{1}{K}\sum_{k=1}^{K}\mathbb{E}[\| y^{(k)}_{t} - \bar{y}_{t} \|^2] \notag \\
    & + 24\tau L_{f}^2\alpha_x^2\eta^2 \mathbb{E}[\| \bar{u}_{t} \|^2] + 16\tau L_{f}^2\alpha_y^2\eta^2 \mathbb{E}[\| \bar{v}_{t} \|^2]  + 12\tau \beta_y^2\eta^4 \delta^2 + 8\tau \beta_y^2\eta^4  \sigma^{2} \ , \notag 
\end{align}
	\begin{align}
		(iii). \ & \frac{1}{K}\sum_{k=1}^{K}\mathbb{E}[\|u^{(k)}_{3, t+1} -\bar{u}_{3, t+1}\|^2]\leq (1+\frac{1}{\tau}+ 24\tau L_{g}^2\alpha_x^2\eta^2)\frac{1}{K}\sum_{k=1}^{K}\mathbb{E}[\|u^{(k)}_{3, t} - \bar{u}_{3, t}\|^2]\notag \\
	    &  + 24\tau L_{g}^2\alpha_x^2\eta^2\rho^2 \frac{1}{K}\sum_{k=1}^{K}\mathbb{E}[\| u^{(k)}_{1, t} - \bar{u}_{1, t} \|^2] + 24\tau L_{g}^2\alpha_x^2\eta^2 \frac{1}{K}\sum_{k=1}^{K}\mathbb{E}[\| u^{(k)}_{2, t} - \bar{u}_{2, t} \|^2] \notag \\
        & + 8\tau L_{g}^2\alpha_z^2\eta^2 \frac{1}{K}\sum_{k=1}^{K}\mathbb{E}[\|    w^{(k)}_{t} -  \bar{w}_{t}   \|^2]  + 24\tau L_{g}^2\alpha_x^2\eta^2 \mathbb{E}[\| \bar{u}_{t} \|^2] + 8\tau L_{g}^2\alpha_z^2\eta^2 \mathbb{E}[\|   \bar{w}_{t}  \|^2]  \notag \\
        &  + 24\tau \beta_x^2\eta^4L_{g}^{2} \frac{1}{K}\sum_{k=1}^{K}\mathbb{E}[\| x^{(k)}_{t} - \bar{x}_{t} \|^2] + 24\tau \beta_x^2\eta^4L_{g}^{2} \frac{1}{K}\sum_{k=1}^{K}\mathbb{E}[\| z^{(k)}_{t} - \bar{z}_{t} \|^2] \notag \\
        & + 12\tau \beta_x^2\eta^4\delta^2 +8\tau \beta_x^2\eta^4\sigma^{2} \ ,   \notag 
	\end{align}
\end{lemma}

\begin{proof}
    According to the update rule of $u^{(k)}_{1, t+1}$ and $\bar{u}_{1, t+1}$, we have
     \begin{align}\label{eq:cons_u_1}
	& \quad \frac{1}{K}\sum_{k=1}^{K}\mathbb{E}[\|u^{(k)}_{1, t+1} -\bar{u}_{1, t+1}\|^2] \notag \\
	& = \frac{1}{K}\sum_{k=1}^{K}\mathbb{E}[\|\Big((1-\beta_x\eta^2)(u^{(k)}_{1, t} - \nabla_1 f^{(k)}(x^{(k)}_{t}, y^{(k)}_{t}; \xi^{(k)}_{t+1})) + \nabla_1 f^{(k)}(x^{(k)}_{t+1}, y^{(k)}_{t+1}; \xi^{(k)}_{t+1})\Big) \notag \\
	& \quad- \frac{1}{K}\sum_{k'=1}^{K}\Big((1-\beta_x\eta^2)(u^{(k')}_{1, t} - \nabla_1 f^{(k')}(x^{(k')}_{t}, y^{(k')}_{t}; \xi^{(k')}_{t+1})) + \nabla_1 f^{(k')}(x^{(k')}_{t+1}, y^{(k')}_{t+1}; \xi^{(k')}_{t+1})\Big) \|^2] \notag \\
	& \leq (1-\beta_x\eta^2)^2(1+\frac{1}{\tau})\frac{1}{K}\sum_{k=1}^{K}\mathbb{E}[\|u^{(k)}_{1, t} - \frac{1}{K}\sum_{k'=1}^{K} u^{(k')}_{1, t}\|^2] \notag \\
	& \quad+ 2(1+\tau)\frac{1}{K}\sum_{k=1}^{K}\mathbb{E}[\|- \nabla_1 f^{(k)}(x^{(k)}_{t}, y^{(k)}_{t}; \xi^{(k)}_{t+1}) + \nabla_1 f^{(k)}(x^{(k)}_{t+1}, y^{(k)}_{t+1}; \xi^{(k)}_{t+1}) \|^2] \notag \\
	& \quad+ 2(1+\tau)\beta_x^2\eta^4 \frac{1}{K}\sum_{k=1}^{K}\mathbb{E}[\|  \nabla_1 f^{(k)}(x^{(k)}_{t}, y^{(k)}_{t}; \xi^{(k)}_{t+1}) - \frac{1}{K}\sum_{k'=1}^{K} \nabla_1 f^{(k')}(x^{(k')}_{t}, y^{(k')}_{t}; \xi^{(k')}_{t+1}) \|^2] \notag \\
	& \leq (1-\beta_x\eta^2)^2(1+\frac{1}{\tau})\frac{1}{K}\sum_{k=1}^{K}\mathbb{E}[\|u^{(k)}_{1, t} - \bar{u}_{1, t}\|^2] \notag \\
	& \quad+ 2(1+\tau)L_{f}^2\frac{1}{K}\sum_{k=1}^{K}\mathbb{E}[\| x^{(k)}_{t+1} - x^{(k)}_{t} \|^2] + 2(1+\tau)L_{f}^2\frac{1}{K}\sum_{k=1}^{K}\mathbb{E}[\| y^{(k)}_{t+1} - y^{(k)}_{t} \|^2]  \\
	& \quad+ 2(1+\tau)\beta_x^2\eta^4 \frac{1}{K}\sum_{k=1}^{K}\mathbb{E}[\|  \nabla_1 f^{(k)}(x^{(k)}_{t}, y^{(k)}_{t}; \xi^{(k)}_{t+1}) - \frac{1}{K}\sum_{k'=1}^{K} \nabla_1 f^{(k')}(x^{(k')}_{t}, y^{(k')}_{t}; \xi^{(k')}_{t+1}) \|^2] \ . \notag
    \end{align}
    For the last term in Eq.~(\ref{eq:cons_u_1}), we obtain
    \begin{align}\label{eq:cons_u_1_last}
        & \quad \frac{1}{K}\sum_{k=1}^{K}\mathbb{E}[\|  \nabla_1 f^{(k)}(x^{(k)}_{t}, y^{(k)}_{t}; \xi^{(k)}_{t+1}) - \frac{1}{K}\sum_{k'=1}^{K} \nabla_1 f^{(k')}(x^{(k')}_{t}, y^{(k')}_{t}; \xi^{(k')}_{t+1}) \|^2] \notag \\
        & \leq  \frac{1}{K}\sum_{k=1}^{K}\mathbb{E}[\|  \nabla_1 f^{(k)}(x^{(k)}_{t}, y^{(k)}_{t}; \xi^{(k)}_{t+1}) - \nabla_1 f^{(k)}(x^{(k)}_{t}, y^{(k)}_{t}) + \nabla_1 f^{(k)}(x^{(k)}_{t}, y^{(k)}_{t}) \notag \\
        & \qquad - \frac{1}{K}\sum_{k'=1}^{K}\nabla_1 f^{(k')}(x^{(k')}_{t}, y^{(k')}_{t}) + \frac{1}{K}\sum_{k'=1}^{K}\nabla_1 f^{(k')}(x^{(k')}_{t}, y^{(k')}_{t}) \notag \\
        & \qquad - \frac{1}{K}\sum_{k'=1}^{K}\nabla_1 f^{(k')}(x^{(k')}_{t}, y^{(k')}_{t}; \xi^{(k')}_{t+1}) \|^2] \notag \\
        & \leq 2\sigma^2 + \frac{1}{K}\sum_{k=1}^{K}\mathbb{E}[\| \nabla_1 f^{(k)}(x^{(k)}_{t}, y^{(k)}_{t})  -   \nabla_1 f^{(k)}(\bar{x}, \bar{y}) + \nabla_1 f^{(k)}(\bar{x}, \bar{y}) -  \frac{1}{K}\sum_{k'=1}^{K}\nabla_1 f^{(k')}(\bar{x}_{t}, \bar{y}_{t}) \notag \\
        & \qquad \qquad \qquad + \frac{1}{K}\sum_{k'=1}^{K}\nabla_1 f^{(k')}(\bar{x}_{t}, \bar{y}_{t})  - \frac{1}{K}\sum_{k'=1}^{K}\nabla_1 f^{(k')}(x^{(k')}_{t}, y^{(k')}_{t})  \|^2] \notag \\
        & \leq 2\sigma^2 + 3\frac{1}{K}\sum_{k=1}^{K}\mathbb{E}[\| \nabla_1 f^{(k)}(x^{(k)}_{t}, y^{(k)}_{t})  -   \nabla_1 f^{(k)}(\bar{x}, \bar{y}) \|^2] \notag \\
        & \quad + 3\frac{1}{K}\sum_{k=1}^{K}\mathbb{E}[\| \nabla_1 f^{(k)}(\bar{x}, \bar{y}) -  \frac{1}{K}\sum_{k'=1}^{K}\nabla_1 f^{(k')}(\bar{x}_{t}, \bar{y}_{t}) \|^2] \notag \\
        & \quad + 3\frac{1}{K}\sum_{k=1}^{K}\mathbb{E}[\| \frac{1}{K}\sum_{k'=1}^{K}\nabla_1 f^{(k')}(\bar{x}_{t}, \bar{y}_{t})  - \frac{1}{K}\sum_{k'=1}^{K}\nabla_1 f^{(k')}(x^{(k')}_{t}, y^{(k')}_{t})  \|^2] \notag \\
        & \leq 2\sigma^2 + 6L_f^2\frac{1}{K}\sum_{k=1}^{K}\mathbb{E}[\| x^{(k)}_{t} - \bar{x}_{t}\|^2] + 6L_f^2\frac{1}{K}\sum_{k=1}^{K}\mathbb{E}[\| y^{(k)}_{t} - \bar{y}_{t}\|^2] + 3\delta^2 \ .  
    \end{align}
    Thus, by substituting these inequalities into Eq.~(\ref{eq:cons_u_1}) with $1 + \tau \leq 2\tau$, and applying Lemma~\ref{lemma:incremental-x-y-z}, the proof is complete.
    
    The other inequalities can be proved in the same way. Therefore, we omit the detailed proof. 
\end{proof}

\begin{lemma}\label{lemma:consensus-error-momentum-each-step-new}
Given Assumptions~\ref{assumption:f-smooth}-\ref{assumption:heterogeneous}, by setting $\eta\leq \min\{\frac{1}{\sqrt{\beta_x}}, \frac{1}{\sqrt{\beta_y}}\}$,  we have the following inequalities: 
	\begin{align}
		(i). \  &  \frac{1}{K}\sum_{k=1}^{K}\mathbb{E}[\|u^{(k)}_{2, t+1} -\bar{u}_{2, t+1}\|^2]  \leq (1+\frac{1}{\tau} + 72\tau(L_g^2+\gamma_1^2)\alpha_x^2\eta^2 )\frac{1}{K}\sum_{k=1}^{K}\mathbb{E}[\|u^{(k)}_{2, t} - \bar{u}_{2, t}\|^2]  \notag \\
        &  + 72\tau(L_g^2+\gamma_1^2)\alpha_x^2\eta^2\rho^2 \frac{1}{K}\sum_{k=1}^{K}\mathbb{E}[\| u^{(k)}_{1, t} - \bar{u}_{1, t} \|^2] + 72\tau(L_g^2+\gamma_1^2) \alpha_x^2\eta^2 \frac{1}{K}\sum_{k=1}^{K}\mathbb{E}[\| u^{(k)}_{3, t} - \bar{u}_{3, t} \|^2] \notag \\
        & + 48\tau L_g^2\alpha_y^2\eta^2\rho^2 \frac{1}{K}\sum_{k=1}^{K}\mathbb{E}[\| v^{(k)}_{1, t} - \bar{v}_{1, t} \|^2] + 48\tau L_g^2\alpha_y^2\eta^2 \frac{1}{K}\sum_{k=1}^{K}\mathbb{E}[\| v^{(k)}_{2, t} - \bar{v}_{2, t} \|^2] \notag \\
        &  + 36\tau\gamma_1^2 \alpha_{p}^2\eta^2 \frac{1}{K}\sum_{k=1}^{K} \mathbb{E}[\| x^{(k)}_{t+1} - \bar{x}_{t+1} \|^2] + 36\tau\gamma_1^2  \mathbb{E}[\| \bar{p}_{t+1} - \bar{p}_{t} \|^2] \notag \\
        & +  12\tau\beta_x^2\eta^4(6L_g^2 + \gamma_1^2) \frac{1}{K}\sum_{k=1}^{K}\mathbb{E}[\| x^{(k)}_{t} - \bar{x}_{t}\|^2]+ 72\tau\beta_x^2\eta^4L_g^2 \frac{1}{K}\sum_{k=1}^{K}\mathbb{E}[\| y^{(k)}_{t} - \bar{y}_{t}\|^2]  \notag \\
        & + 12\tau\gamma_1^2\Big(3\alpha_{p}^2\eta^2 + \beta_{x}^2\eta^4\Big)\frac{1}{K}\sum_{k=1}^{K} \mathbb{E}[\| \bar{p}_{t} - p^{(k)}_{t} \|^2] + 72\tau(L_g^2+\gamma_1^2) \alpha_x^2\eta^2 \mathbb{E}[\| \bar{u}_{t} \|^2] \notag \\
        &  + 48\tau L_g^2\alpha_y^2\eta^2  \mathbb{E}[\| \bar{v}_{t} \|^2]  + 24\tau\beta_x^2\eta^4\sigma^2 + 36\tau\beta_x^2\eta^4\delta^2  \ . \notag 
	\end{align}
	\begin{align}
		(ii). \   &  \frac{1}{K}\sum_{k=1}^{K}\mathbb{E}[\|v^{(k)}_{2, t+1} -\bar{v}_{2, t+1}\|^2]  \leq (1+\frac{1}{\tau} + 48\tau(L_g^2+\gamma_1^2)\alpha_y^2\eta^2 )\frac{1}{K}\sum_{k=1}^{K}\mathbb{E}[\|v^{(k)}_{2, t} - \bar{v}_{2, t}\|^2] \notag \\
        &  + 72\tau L_g^2\alpha_x^2\eta^2\rho^2 \frac{1}{K}\sum_{k=1}^{K}\mathbb{E}[\| u^{(k)}_{1, t} - \bar{u}_{1, t} \|^2] + 72\tau L_g^2 \alpha_x^2\eta^2 \frac{1}{K}\sum_{k=1}^{K}\mathbb{E}[\| u^{(k)}_{2, t} - \bar{u}_{2, t} \|^2] \notag \\
        & + 72\tau L_g^2 \alpha_x^2\eta^2 \frac{1}{K}\sum_{k=1}^{K}\mathbb{E}[\| u^{(k)}_{3, t} - \bar{u}_{3, t} \|^2] + 48\tau(L_g^2+\gamma_1^2)\alpha_y^2\eta^2\rho^2 \frac{1}{K}\sum_{k=1}^{K}\mathbb{E}[\| v^{(k)}_{1, t} - \bar{v}_{1, t} \|^2] \notag \\
        & + 36\tau\gamma_1^2\alpha_{q}^2\eta^2 \frac{1}{K}\sum_{k=1}^{K} \mathbb{E}[\| y^{(k)}_{t+1} - \bar{y}_{t+1} \|^2] + 36\tau\gamma_1^2  \mathbb{E}[\| \bar{q}_{t+1} - \bar{q}_{t} \|^2] \notag \\
        & +  72\tau\beta_y^2\eta^4L_g^2\frac{1}{K}\sum_{k=1}^{K}\mathbb{E}[\| x^{(k)}_{t} - \bar{x}_{t}\|^2]+ 12\tau\beta_y^2\eta^4(6L_g^2 +   \gamma_1^2)\frac{1}{K}\sum_{k=1}^{K}\mathbb{E}[\| y^{(k)}_{t} - \bar{y}_{t}\|^2] \notag \\
        &  + 12\tau\gamma_1^2\Big( 3\alpha_{q}^2\eta^2 + \beta_{y}^2\eta^4\Big)\frac{1}{K}\sum_{k=1}^{K} \mathbb{E}[\| \bar{q}_{t} - q^{(k)}_{t} \|^2]+ 72\tau L_g^2\alpha_x^2\eta^2 \mathbb{E}[\| \bar{u}_{t} \|^2] \notag \\
        &   + 48\tau(L_g^2+\gamma_1^2)\alpha_y^2\eta^2  \mathbb{E}[\| \bar{v}_{t} \|^2] + 24\tau\beta_y^2\eta^4\sigma^2 + 36\tau\beta_y^2\eta^4\delta^2 \ . \notag 
	\end{align}
    \begin{align}
		(iii). \ & \frac{1}{K}\sum_{k=1}^{K}\mathbb{E}[\|w^{(k)}_{t+1} -\bar{w}_{t+1}\|^2]  \leq \Big(1+\frac{1}{\tau} + 24\tau(L_g^2+\gamma_2^2)\alpha_z^2\eta^2 \Big)\frac{1}{K}\sum_{k=1}^{K}\mathbb{E}[\|w^{(k)}_{t} - \bar{w}_{t}\|^2]  \notag \\
		&  + 72\tau L_g^2\alpha_x^2\eta^2\rho^2 \frac{1}{K}\sum_{k=1}^{K}\mathbb{E}[\| u^{(k)}_{1, t} - \bar{u}_{1, t} \|^2] + 72\tau L_g^2 \alpha_x^2\eta^2 \frac{1}{K}\sum_{k=1}^{K}\mathbb{E}[\| u^{(k)}_{2, t} - \bar{u}_{2, t} \|^2] \notag \\
		& + 72\tau L_g^2 \alpha_x^2\eta^2 \frac{1}{K}\sum_{k=1}^{K}\mathbb{E}[\| u^{(k)}_{3, t} - \bar{u}_{3, t} \|^2] + 36\tau\gamma_2^2 \alpha_{r}^2\eta^2 \frac{1}{K}\sum_{k=1}^{K} \mathbb{E}[\| z^{(k)}_{t+1} - \bar{z}_{t+1} \|^2]\notag \\
		&   + 36\tau\gamma_2^2  \mathbb{E}[\| \bar{r}_{t+1} - \bar{r}_{t} \|^2] + 12\tau\gamma_2^2\Big(3\alpha_{r}^2\eta^2 + \beta_{z}^2\eta^4\Big)\frac{1}{K}\sum_{k=1}^{K} \mathbb{E}[\| \bar{r}_{t} - r^{(k)}_{t} \|^2] \notag \\
		& + 72\tau\beta_z^2\eta^4 L_g^2 \frac{1}{K}\sum_{k=1}^{K}\mathbb{E}[\| x^{(k)}_{t} - \bar{x}_{t}\|^2]+ 12\tau\beta_z^2\eta^4(6L_g^2+\gamma_2^2) \frac{1}{K}\sum_{k=1}^{K}\mathbb{E}[\| z^{(k)}_{t} - \bar{z}_{t}\|^2]  \notag \\
		& + 72\tau L_g^2 \alpha_x^2\eta^2 \mathbb{E}[\| \bar{u}_{t} \|^2] + 24\tau(L_g^2+\gamma_2^2)\alpha_z^2\eta^2  \mathbb{E}[\| \bar{w}_{t} \|^2]  + 24\tau\beta_z^2\eta^4\sigma^2 + 36\tau\beta_z^2\eta^4\delta^2  \ . \notag
	\end{align}
\end{lemma}

\begin{proof}
According to the update rule of $u^{(k)}_{2, t+1}$ and $\bar{u}_{2, t+1}$, we have
 \begin{align}\label{eq:cons_u_2}
	& \quad \frac{1}{K}\sum_{k=1}^{K}\mathbb{E}[\|u^{(k)}_{2, t+1} -\bar{u}_{2, t+1}\|^2] \notag   \\
	& = \frac{1}{K}\sum_{k=1}^{K}\mathbb{E}[\|\Big((1-\beta_x\eta^2)(u^{(k)}_{2, t} - \nabla_1 g_{\gamma_1}^{(k)}(x^{(k)}_{t}, y^{(k)}_{t}; p^{(k)}_{t}, q^{(k)}_{t}; \zeta^{(k)}_{t+1})) \notag \\
    & \qquad  + \nabla_1 g_{\gamma_1}^{(k)}(x^{(k)}_{t+1}, y^{(k)}_{t+1}; p^{(k)}_{t+1}, q^{(k)}_{t+1}; \zeta^{(k)}_{t+1})\Big) - \frac{1}{K}\sum_{k=1}^{N}\Big((1-\beta_x\eta^2)(u^{(k')}_{2, t} \notag \\
	& \qquad \quad   - \nabla_1 g_{\gamma_1}^{(k')}(x^{(k')}_{t}, y^{(k')}_{t}; p^{(k')}_{t}, q^{(k')}_{t}; \zeta^{(k')}_{t+1})) + \nabla_1 g_{\gamma_1}^{(k')}(x^{(k')}_{t+1}, y^{(k')}_{t+1}; p^{(k')}_{t+1}, q^{(k')}_{t+1}; \zeta^{(k')}_{t+1})\Big) \|^2] \notag \\
	& \leq (1-\beta_x\eta^2)^2(1+\frac{1}{\tau})\frac{1}{K}\sum_{k=1}^{K}\mathbb{E}[\|u^{(k)}_{2, t} - \bar{u}_{2, t}\|^2] \notag \\
	& \quad+ 2(1+\tau)\frac{1}{K}\sum_{k=1}^{K}\mathbb{E}[\| \nabla_1 g_{\gamma_1}^{(k)}(x^{(k)}_{t+1}, y^{(k)}_{t+1}; p^{(k)}_{t+1}, q^{(k)}_{t+1}; \zeta^{(k)}_{t+1}) \notag \\
    & \qquad \quad - \nabla_1 g_{\gamma_1}^{(k)}(x^{(k)}_{t}, y^{(k)}_{t}; p^{(k)}_{t}, q^{(k)}_{t}; \zeta^{(k)}_{t+1}) \|^2] \notag  \\
	& \quad+ 2\beta_x^2\eta^4(1+\tau) \frac{1}{K}\sum_{k=1}^{K}\mathbb{E}[\|  \nabla_1 g_{\gamma_1}^{(k)}(x^{(k)}_{t}, y^{(k)}_{t}; p^{(k)}_{t}, q^{(k)}_{t}; \zeta^{(k)}_{t+1}) \notag \\
    & \qquad \quad - \frac{1}{K}\sum_{k'=1}^{K}\nabla_1 g_{\gamma_1}^{(k')}(x^{(k')}_{t}, y^{(k')}_{t}; p^{(k')}_{t}, q^{(k')}_{t}; \zeta^{(k')}_{t+1}) \|^2] \ .   
\end{align}
For the second term in Eq.~(\ref{eq:cons_u_2}), we obtain
\begin{align}
   & \quad \frac{1}{K}\sum_{k=1}^{K} \mathbb{E}[\| \nabla_1 g_{\gamma_1}^{(k)}(x^{(k)}_{t+1}, y^{(k)}_{t+1}; p^{(k)}_{t+1}, q^{(k)}_{t+1}; \zeta^{(k)}_{t+1})  - \nabla_1 g_{\gamma_1}^{(k)}(x^{(k)}_{t}, y^{(k)}_{t}; p^{(k)}_{t}, q^{(k)}_{t}; \zeta^{(k)}_{t+1}) \|^2] \notag \\
   & =  \frac{1}{K}\sum_{k=1}^{K} \mathbb{E}[\| \nabla_1 g^{(k)}(x^{(k)}_{t+1}, y^{(k)}_{t+1}; \zeta^{(k)}_{t+1}) + \gamma_1(x^{(k)}_{t+1}-p^{(k)}_{t+1}) - \nabla_1 g^{(k)}(x^{(k)}_{t}, y^{(k)}_{t}; \zeta^{(k)}_{t+1}) \notag \\
   & \qquad \qquad  - \gamma_1(x^{(k)}_{t}-p^{(k)}_{t})  \|^2] \notag \\
   & \leq 3\frac{1}{K}\sum_{k=1}^{K} \mathbb{E}[\| \nabla_1 g^{(k)}(x^{(k)}_{t+1}, y^{(k)}_{t+1}; \zeta^{(k)}_{t+1}) - \nabla_1 g^{(k)}(x^{(k)}_{t}, y^{(k)}_{t}; \zeta^{(k)}_{t+1})\|^2] \notag  \\
   & \quad + 3\gamma_1^2\frac{1}{K}\sum_{k=1}^{K} \mathbb{E}[\| x^{(k)}_{t+1} - x^{(k)}_{t} \|^2]   + 3\gamma_1^2\frac{1}{K}\sum_{k=1}^{K} \mathbb{E}[\| p^{(k)}_{t+1} - p^{(k)}_{t} \|^2] \notag \\
   &  \leq  3(L_g^2+\gamma_1^2)\frac{1}{K}\sum_{k=1}^{K} \mathbb{E}[\| x^{(k)}_{t+1} - x^{(k)}_{t} \|^2] + 3L_g^2 \frac{1}{K}\sum_{k=1}^{K} \mathbb{E}[\| x^{(k)}_{t+1} - x^{(k)}_{t} \|^2] \notag \\
   & \quad + 3\gamma_1^2\frac{1}{K}\sum_{k=1}^{K} \mathbb{E}[\| p^{(k)}_{t+1} - p^{(k)}_{t} \|^2]  \ .  \notag 
\end{align}
Moreover, we have
\begin{align}
    & \quad \frac{1}{K}\sum_{k=1}^{K} \mathbb{E}[\| p^{(k)}_{t+1} - p^{(k)}_{t} \|^2] = \alpha_{p}^2\eta^2 \frac{1}{K}\sum_{k=1}^{K} \mathbb{E}[\| x^{(k)}_{t+1} - p^{(k)}_{t} \|^2] \notag \\
    & \leq 3\alpha_{p}^2\eta^2 \frac{1}{K}\sum_{k=1}^{K} \mathbb{E}[\| x^{(k)}_{t+1} - \bar{x}_{t+1} \|^2] + 3\alpha_{p}^2\eta^2  \mathbb{E}[\| \bar{x}_{t+1} - \bar{p}_{t} \|^2] + 3\alpha_{p}^2\eta^2 \frac{1}{K}\sum_{k=1}^{K} \mathbb{E}[\| \bar{p}_{t} - p^{(k)}_{t} \|^2] \notag \\
    & = 3\alpha_{p}^2\eta^2 \frac{1}{K}\sum_{k=1}^{K} \mathbb{E}[\| x^{(k)}_{t+1} - \bar{x}_{t+1} \|^2] + 3  \mathbb{E}[\| \bar{p}_{t+1} - \bar{p}_{t} \|^2] + 3\alpha_{p}^2\eta^2 \frac{1}{K}\sum_{k=1}^{K} \mathbb{E}[\| \bar{p}_{t} - p^{(k)}_{t} \|^2] \ . \notag 
\end{align}

For the third term in Eq.~(\ref{eq:cons_u_2}), we obtain
\begin{align}
    & \quad \frac{1}{K}\sum_{k=1}^{K}\mathbb{E}[\|  \nabla_1 g_{\gamma_1}^{(k)}(x^{(k)}_{t}, y^{(k)}_{t}; p^{(k)}_{t}, q^{(k)}_{t}; \zeta^{(k)}_{t+1}) - \frac{1}{K}\sum_{k'=1}^{K}\nabla_1 g_{\gamma_1}^{(k')}(x^{(k')}_{t}, y^{(k')}_{t}; p^{(k')}_{t}, q^{(k')}_{t}; \zeta^{(k')}_{t+1}) \|^2] \notag  \\
    & \leq \frac{1}{K}\sum_{k=1}^{K}\mathbb{E}[\|  \nabla_1 g^{(k)}(x^{(k)}_{t}, y^{(k)}_{t}; \zeta^{(k)}_{t+1}) + \gamma_1(x^{(k)}_{t} - p^{(k)}_{t})  -  \frac{1}{K}\sum_{k'=1}^{K}\nabla_1 g^{(k')}(x^{(k')}_{t}, y^{(k')}_{t}; \zeta^{(k')}_{t+1})\notag \\
    & \quad  - \frac{1}{K}\sum_{k'=1}^{K}\gamma_1(x^{(k')}_{t} - p^{(k')}_{t})   \|^2] \notag \\
    & \leq 3\frac{1}{K}\sum_{k=1}^{K}\mathbb{E}[\|  \nabla_1 g^{(k)}(x^{(k)}_{t}, y^{(k)}_{t}; \zeta^{(k)}_{t+1}) - \frac{1}{K}\sum_{k'=1}^{K}\nabla_1 g^{(k')}(x^{(k')}_{t}, y^{(k')}_{t}; \zeta^{(k')}_{t+1}) \|^2] \notag \\
    & \quad + 3\gamma_1^2\frac{1}{K}\sum_{k=1}^{K}\mathbb{E}[\| x^{(k)}_{t} - \bar{x}_{t} \|^2] + 3\gamma_1^2\frac{1}{K}\sum_{k=1}^{K}\mathbb{E}[\| p^{(k)}_{t} - \bar{p}_{t} \|^2]  \notag \\
    & \leq 6\sigma^2 + 9\delta^2 +  (18L_g^2 + 3\gamma_1^2) \frac{1}{K}\sum_{k=1}^{K}\mathbb{E}[\| x^{(k)}_{t} - \bar{x}_{t}\|^2]+ 18L_g^2 \frac{1}{K}\sum_{k=1}^{K}\mathbb{E}[\| y^{(k)}_{t} - \bar{y}_{t}\|^2] \notag \\
    & \quad + 3\gamma_1^2\frac{1}{K}\sum_{k=1}^{K}\mathbb{E}[\| p^{(k)}_{t} - \bar{p}_{t} \|^2]  \ , \notag
\end{align}
where the last step follows Eq.~(\ref{eq:cons_u_1_last}). 
Thus, by substituting these inequalities into Eq.~(\ref{eq:cons_u_2}) with $1 + \tau \leq 2\tau$, and applying Lemma~\ref{lemma:incremental-x-y-z}, the proof is complete.

Using a similar argument, we can establish the other inequalities. Hence, we omit the details.


\end{proof}

\begin{lemma}\label{lemma:consensus_total}
	Given Assumptions~\ref{assumption:f-smooth}-\ref{assumption:heterogeneous},  by defining $L_{\rho,\gamma_1}^2 = L_{f}^2\rho^2+  4L_{g}^2 +  3\gamma_1^2$, and setting
    \begin{align}\label{eq:eta_alpha}
     & \eta \leq  \frac{1}{360 \tau L_{\rho,\gamma_1}}  \  ,\alpha_x\leq 10 \ ,  \alpha_y\leq 10 \ ,  \alpha_z\leq 10 \ , \alpha_{p} \leq 10\gamma_1 \ , \alpha_{q} \leq 10\gamma_1 \ ,  \notag \\
     &  \beta_x\leq \frac{L_{\rho,\gamma_1}^2}{K} \ , \beta_y\leq \frac{L_{\rho,\gamma_1}^2}{K} \ ,  \beta_z\leq \frac{L_{\rho,\gamma_1}^2}{K} \ , 
\end{align}
we have
{
\begin{align}
	& \quad \frac{1}{T}\sum_{t=0}^{T-1}\Bigg(\rho^2\frac{1}{K}\sum_{k=1}^{K}\mathbb{E}[\|u^{(k)}_{1, t} -\bar{u}_{1, t}\|^2]  + \frac{1}{K}\sum_{k=1}^{K}\mathbb{E}[\|u^{(k)}_{2, t} -\bar{u}_{2, t}\|^2]  +\frac{1}{K}\sum_{k=1}^{K}\mathbb{E}[\|u^{(k)}_{3, t} -\bar{u}_{3, t}\|^2]\notag \\
	& \qquad \qquad + \rho^2\frac{1}{K}\sum_{k=1}^{K}\mathbb{E}[\|v^{(k)}_{1, t} -\bar{v}_{1, t}\|^2]  + \frac{1}{K}\sum_{k=1}^{K}\mathbb{E}[\|v^{(k)}_{2, t} -\bar{v}_{2, t}\|^2]  + \frac{1}{K}\sum_{k=1}^{K}\mathbb{E}[\|w^{(k)}_{t} -\bar{w}_{t}\|^2] \Bigg)\notag \\
    & \leq 288\tau^2 \alpha_x^2\eta^2L_{\rho,\gamma}^2 \frac{1}{T}\sum_{t=0}^{T-1}\mathbb{E}[\| \bar{u}_{t} \|^2]  + 192\tau^2\alpha_y^2\eta^2L_{\rho,\gamma}^2 \frac{1}{T}\sum_{t=0}^{T-1}\mathbb{E}[\| \bar{v}_{t} \|^2] +  48\tau^2 \alpha_z^2\eta^2L_{\rho,\gamma}^2 \frac{1}{T}\sum_{t=0}^{T-1}\mathbb{E}[\|   \bar{w}_{t}  \|^2]  \notag \\
    & \quad  + 432\tau^2\gamma_1^2 \frac{1}{T}\sum_{t=0}^{T-1} \mathbb{E}[\| \bar{p}_{t+1} - \bar{p}_{t} \|^2] + 432\tau^2\gamma_1^2\frac{1}{T}\sum_{t=0}^{T-1}  \mathbb{E}[\| \bar{q}_{t+1} - \bar{q}_{t} \|^2] \notag \\
    & \quad + 432\tau^2\gamma_2^2\frac{1}{T}\sum_{t=0}^{T-1}  \mathbb{E}[\| \bar{r}_{t+1} - \bar{r}_{t} \|^2] + 3\tau S \ , \notag 
\end{align}}
where the definition of $S$ is in Eq.~(\ref{eq:s_def}).
\end{lemma}

\begin{proof}
From Lemma~\ref{lemma:consensus-error-momentum-each-step-org}, Lemma~\ref{lemma:consensus-error-momentum-each-step-new}, and define $L_{\rho,\gamma}^2 = L_{f}^2\rho^2+  5L_{g}^2 +  3(\gamma_1^2+\gamma_2^2)$,  we obtain
\begin{align}
	& \quad \rho^2\frac{1}{K}\sum_{k=1}^{K}\mathbb{E}[\|u^{(k)}_{1, t+1} -\bar{u}_{1, t+1}\|^2]  + \frac{1}{K}\sum_{k=1}^{K}\mathbb{E}[\|u^{(k)}_{2, t+1} -\bar{u}_{2, t+1}\|^2]  +\frac{1}{K}\sum_{k=1}^{K}\mathbb{E}[\|u^{(k)}_{3, t+1} -\bar{u}_{3, t+1}\|^2]\notag \\
	& \quad + \rho^2\frac{1}{K}\sum_{k=1}^{K}\mathbb{E}[\|v^{(k)}_{1, t+1} -\bar{v}_{1, t+1}\|^2]  + \frac{1}{K}\sum_{k=1}^{K}\mathbb{E}[\|v^{(k)}_{2, t+1} -\bar{v}_{2, t+1}\|^2]  + \frac{1}{K}\sum_{k=1}^{K}\mathbb{E}[\|w^{(k)}_{t+1} -\bar{w}_{t+1}\|^2] \notag  \\
& \leq \Big(1+\frac{1}{\tau} + 48\tau \alpha_x^2\eta^2L_{\rho,\gamma}^2\Big)\rho^2\frac{1}{K}\sum_{k=1}^{K}\mathbb{E}[\|u^{(k)}_{1, t} - \bar{u}_{1, t}\|^2] + 48\tau \alpha_x^2\eta^2L_{\rho,\gamma}^2 \mathbb{E}[\| \bar{u}_{t} \|^2]  \notag \\
& \quad + \Big(1+\frac{1}{\tau} + 48\tau \alpha_x^2\eta^2L_{\rho,\gamma}^2\Big)\frac{1}{K}\sum_{k=1}^{K}\mathbb{E}[\| u^{(k)}_{2, t} - \bar{u}_{2, t} \|^2]    + 32\tau\alpha_y^2\eta^2L_{\rho,\gamma}^2 \mathbb{E}[\| \bar{v}_{t} \|^2]  \notag \\
& \quad + \Big(1+\frac{1}{\tau} + 48\tau \alpha_x^2\eta^2L_{\rho,\gamma}^2\Big) \frac{1}{K}\sum_{k=1}^{K} \mathbb{E}[\| u^{(k)}_{3, t} - \bar{u}_{3, t} \|^2] + 8\tau L_{\rho,\gamma}^2\alpha_z^2\eta^2 \mathbb{E}[\|   \bar{w}_{t}  \|^2]  \notag \\
& \quad + \Big(1+\frac{1}{\tau}+32\tau\alpha_y^2\eta^2L_{\rho,\gamma}^2  \Big) \rho^2 \frac{1}{K}\sum_{k=1}^{K}\mathbb{E}[\| v^{(k)}_{1, t} - \bar{v}_{1, t} \|^2] + S_{x}\frac{1}{K}\sum_{k=1}^{K}\mathbb{E}[\| x^{(k)}_{t} - \bar{x}_{t} \|^2] \notag \\
& \quad + \Big( 1+\frac{1}{\tau}+32\tau\alpha_y^2\eta^2L_{\rho,\gamma}^2\Big) \frac{1}{K}\sum_{k=1}^{K}\mathbb{E}[\| v^{(k)}_{2, t} - \bar{v}_{2, t} \|^2]  + S_{y}\frac{1}{K}\sum_{k=1}^{K}\mathbb{E}[\| y^{(k)}_{t} - \bar{y}_{t} \|^2]  \notag \\
& \quad    + \Big(1+\frac{1}{\tau}+ 8\tau  L_{\rho,\gamma}^2\alpha_z^2\eta^2 \Big)\frac{1}{K}\sum_{k=1}^{K}\mathbb{E}[\|    w^{(k)}_{t} -  \bar{w}_{t}   \|^2]    + S_{z}\frac{1}{K}\sum_{k=1}^{K}\mathbb{E}[\| z^{(k)}_{t} - \bar{z}_{t} \|^2]  \notag \\
& \quad + S_{p}\frac{1}{K}\sum_{k=1}^{K} \mathbb{E}[\| \bar{p}_{t} - p^{(k)}_{t} \|^2] + S_{q}\frac{1}{K}\sum_{k=1}^{K} \mathbb{E}[\| \bar{q}_{t} - q^{(k)}_{t} \|^2] + S_{r}\frac{1}{K}\sum_{k=1}^{K} \mathbb{E}[\| \bar{r}_{t} - r^{(k)}_{t} \|^2] \notag \\
& \quad + 36\tau\gamma_1^2 \alpha_{p}^2\eta^2 \frac{1}{K}\sum_{k=1}^{K} \mathbb{E}[\| x^{(k)}_{t+1} - \bar{x}_{t+1} \|^2] + 36\tau\gamma_1^2\alpha_{q}^2\eta^2 \frac{1}{K}\sum_{k=1}^{K} \mathbb{E}[\| y^{(k)}_{t+1} - \bar{y}_{t+1} \|^2]  \notag \\
& \quad  +  36\tau\gamma_2^2\alpha_{r}^2\eta^2 \frac{1}{K}\sum_{k=1}^{K} \mathbb{E}[\| z^{(k)}_{t+1} - \bar{z}_{t+1} \|^2] + 36\tau\gamma_1^2  \mathbb{E}[\| \bar{p}_{t+1} - \bar{p}_{t} \|^2] \notag \\
& \quad  + 36\tau\gamma_1^2  \mathbb{E}[\| \bar{q}_{t+1} - \bar{q}_{t} \|^2]   + 36\tau\gamma_2^2  \mathbb{E}[\| \bar{r}_{t+1} - \bar{r}_{t} \|^2] + S  \ . 
\end{align}
where we denote
\begin{align}\label{eq:s_def}
    & S_{x} = 12\tau\eta^4 (2\beta_x^2L_{f}^{2}\rho^2 + 2\beta_y^2 L_{f}^{2}\rho^2 + 8\beta_x^2 L_{g}^{2} + 6\beta_y^2 L_{g}^{2} + 6\beta_z^2L_{g}^{2} + \beta_x^2\gamma_1^2)  \ , \notag \\
    & S_{y}  =  12\tau\eta^4 (2\beta_x^2 L_{f}^{2}\rho^2 + 2\beta_y^2 L_{f}^{2}\rho^2 + 6\beta_x^2L_g^2 + 6\beta_y^2L_g^2 + \beta_y^2\gamma_1^2)  \ , \notag \\
    & S_{z} =  12\tau\eta^4(2\beta_x^2L_{g}^{2} + 6\beta_z^2L_{g}^{2}+\beta_z^2\gamma_2^2)  \ , \notag \\
    & S_{p} = 36\tau\gamma_1^2\alpha_{p}^2\eta^2 + 12\tau\gamma_1^2\beta_{x}^2\eta^4  \ ,  \\
    & S_{q} = 36\tau\gamma_1^2\alpha_{q}^2\eta^2 + 12\tau\gamma_1^2\beta_{y}^2\eta^4 \ , \notag \\
    & S_{r} = 36\tau\gamma_2^2\alpha_{r}^2\eta^2 + 12\tau\gamma_2^2\beta_{z}^2\eta^4 \ , \notag \\
    & S = 12\tau (\beta_x^2+\beta_y^2)\eta^4 \delta^2\rho^2 + 8\tau (\beta_x^2+\beta_y^2)\eta^4  \sigma^{2}\rho^2 \notag + 12\tau(4\beta_x^2 + 3\beta_y^2+\beta_z^2)\eta^4\delta^2 \notag \\
    & \quad \quad +8\tau (4\beta_x^2 + 3\beta_y^2+\beta_z^2)\eta^4\sigma^{2}  \ . \notag 
\end{align}
By setting $\alpha_x\leq 10$, $\alpha_y<10$, $\alpha_z<10$,  $ \eta \leq \frac{1}{360 \tau L_{\rho,\gamma} }$, we obtain
\begin{align}
     & 48\tau \alpha_x^2\eta^2L_{\rho,\gamma}^2 \leq 4800\tau\eta^2L_{\rho,\gamma}^2 \leq \frac{4800\tau}{360\times360\tau^2} \leq \frac{1}{25\tau} \ ,  \notag \\
     & 32\tau\alpha_y^2\eta^2L_{\rho,\gamma}^2 \leq 3200\tau\eta^2L_{\rho,\gamma}^2 \leq \frac{3200\tau}{360\times360\tau^2}\leq \frac{1}{25\tau} \ ,  \notag \\
     & 8\tau  L_{\rho,\gamma}^2\alpha_z^2\eta^2 \leq 800\tau L_{\rho,\gamma}^2 \eta^2 \leq \frac{800\tau}{360\times360\tau^2}\leq \frac{1}{25\tau} \ . \notag 
\end{align}
Moreover, we derive
\begin{align}\label{eq:consensus_xavg}
    & \quad \frac{1}{K}\sum_{k=1}^{K} \mathbb{E}[\| x^{(k)}_{t+1} - \bar{x}_{t+1} \|^2] = \frac{1}{K}\sum_{k=1}^{K} \mathbb{E}[\| x^{(k)}_{t\tau} - \alpha_{x}\eta\sum_{t'=s_{t}\tau}^{t} u^{(k)}_{t'} - \bar{x}_{t\tau} + \alpha_{x}\eta\sum_{t'=s_{t}\tau}^{t} \bar{u}_{t'}  \|^2] \notag \\
    & \leq  \tau\alpha_{x}^2\eta^2 \sum_{\hat{t}=s_{t}\tau}^{t} \frac{1}{K}\sum_{k=1}^{K} \mathbb{E}[\| u^{(k)}_{\hat{t}} - \bar{u}_{\hat{t}}\|^2] \ , 
\end{align}
and
        \begin{align}\label{eq:consensus_pavg}
            & \quad \frac{1}{K}\sum_{k=1}^{K}\mathbb{E}[\| p^{(k)}_{t}- \bar{p}_{t}\|^2]  = \frac{1}{K}\sum_{k=1}^{K}\mathbb{E}[\| p^{(k)}_{t-1} + \alpha_{p}\eta(x^{(k)}_{t}- {p}^{(k)}_{t-1}) - \bar{p}_{t-1} - \alpha_{p}\eta(\bar{x}_{t} - \bar{p}_{t-1}) \|^2] \notag \\
            & = \frac{1}{K}\sum_{k=1}^{K}\mathbb{E}[\| (1-\alpha_{p}\eta)(p^{(k)}_{t-1} -\bar{p}_{t-1} ) + \alpha_{p}\eta (x^{(k)}_{t}  -  \bar{x}_{t})  \|^2] \notag \\
            & \leq (1+\frac{\alpha_{p}\eta}{1-\alpha_{p}\eta})(1-\alpha_{p}\eta)^2\frac{1}{K}\sum_{k=1}^{K}\mathbb{E}[\| p^{(k)}_{t-1} -\bar{p}_{t-1} \|^2] \notag \\
            & \quad + (1+\frac{1-\alpha_{p}\eta}{\alpha_{p}\eta})\alpha_{p}^2\eta^2 \frac{1}{K}\sum_{k=1}^{K}\mathbb{E}[\| x^{(k)}_{t}  -  \bar{x}_{t}  \|^2] \notag \\
            & = (1-\alpha_{p}\eta)\frac{1}{K}\sum_{k=1}^{K}\mathbb{E}[\| p^{(k)}_{t-1} -\bar{p}_{t-1} \|^2] +\alpha_{p}\eta \frac{1}{K}\sum_{k=1}^{K}\mathbb{E}[\| x^{(k)}_{t}  -  \bar{x}_{t}  \|^2] \notag \\
            & \leq \sum_{\hat{t}=s_t\tau}^{t}(1-\alpha_{p}\eta)^{t-\hat{t}}\alpha_{p}\eta \frac{1}{K}\sum_{k=1}^{K}\mathbb{E}[\| x^{(k)}_{\hat{t}}  -  \bar{x}_{\hat{t}}  \|^2] \leq \alpha_{p}\eta\sum_{\hat{t}=s_t\tau}^{t} \frac{1}{K}\sum_{k=1}^{K}\mathbb{E}[\| x^{(k)}_{\hat{t}}  -  \bar{x}_{\hat{t}}  \|^2] \ ,
       \end{align}
where $s_t$ denotes the index of the latest communication round before iteration $t$.

Similarly, we have
\begin{align}\label{eq:consensus_qavg}
        &  \frac{1}{K}\sum_{k=1}^{K} \mathbb{E}[\| y^{(k)}_{t+1} - \bar{y}_{t+1} \|^2] \leq  \tau\alpha_{y}^2\eta^2 \sum_{\hat{t}=s_{t}\tau}^{t} \frac{1}{K}\sum_{k=1}^{K} \mathbb{E}[\| v^{(k)}_{\hat{t}} - \bar{v}_{\hat{t}}\|^2] \ , \notag \\
         &   \frac{1}{K}\sum_{k=1}^{K}\mathbb{E}[\| q^{(k)}_{t}- \bar{q}_{t}\|^2]  \leq \alpha_{q}\eta\sum_{\hat{t}=s_t\tau}^{t} \frac{1}{K}\sum_{k=1}^{K}\mathbb{E}[\| y^{(k)}_{\hat{t}}  -  \bar{y}_{\hat{t}}  \|^2]  \ , \notag \\
         &  \frac{1}{K}\sum_{k=1}^{K} \mathbb{E}[\| z^{(k)}_{t+1} - \bar{z}_{t+1} \|^2] \leq  \tau\alpha_{z}^2\eta^2 \sum_{\hat{t}=s_{t}\tau}^{t} \frac{1}{K}\sum_{k=1}^{K} \mathbb{E}[\| w^{(k)}_{\hat{t}} - \bar{w}_{\hat{t}}\|^2] \ , \notag \\
         &   \frac{1}{K}\sum_{k=1}^{K}\mathbb{E}[\| r^{(k)}_{t}- \bar{r}_{t}\|^2]  \leq \alpha_{r}\eta\sum_{\hat{t}=s_t\tau}^{t} \frac{1}{K}\sum_{k=1}^{K}\mathbb{E}[\| z^{(k)}_{\hat{t}}  -  \bar{z}_{\hat{t}}  \|^2]  \ , 
\end{align}
where $s_t$ denotes the index of the latest communication round before iteration $t$.

Thus, by setting $s_t = \left\lfloor\frac{t}{\tau} \right\rfloor$, applying Eq.(\ref{eq:consensus_xavg}), (\ref{eq:consensus_pavg}), and (\ref{eq:consensus_qavg}), we obtain:
\begin{align}
	& \quad \rho^2\frac{1}{K}\sum_{k=1}^{K}\mathbb{E}[\|u^{(k)}_{1, t+1} -\bar{u}_{1, t+1}\|^2]  + \frac{1}{K}\sum_{k=1}^{K}\mathbb{E}[\|u^{(k)}_{2, t+1} -\bar{u}_{2, t+1}\|^2]  +\frac{1}{K}\sum_{k=1}^{K}\mathbb{E}[\|u^{(k)}_{3, t+1} -\bar{u}_{3, t+1}\|^2]\notag \\
	& \quad + \rho^2\frac{1}{K}\sum_{k=1}^{K}\mathbb{E}[\|v^{(k)}_{1, t+1} -\bar{v}_{1, t+1}\|^2]  + \frac{1}{K}\sum_{k=1}^{K}\mathbb{E}[\|v^{(k)}_{2, t+1} -\bar{v}_{2, t+1}\|^2]  + \frac{1}{K}\sum_{k=1}^{K}\mathbb{E}[\|w^{(k)}_{t+1} -\bar{w}_{t+1}\|^2]\notag \\
    & \leq 48\tau \alpha_x^2\eta^2L_{\rho,\gamma}^2 \sum_{t'=s_t\tau}^{t-1}(1+\frac{26}{25\tau})^{t-1-t'}\mathbb{E}[\| \bar{u}_{t'} \|^2] + 32\tau\alpha_y^2\eta^2L_{\rho,\gamma}^2 \sum_{t'=s_t\tau}^{t-1}(1+\frac{26}{25\tau})^{t-1-t'}\mathbb{E}[\| \bar{v}_{t'} \|^2] \notag \\
    & \quad  +  8\tau L_{\rho,\gamma}^2 \alpha_z^2\eta^2 \sum_{t'=s_t\tau}^{t-1}(1+\frac{26}{25\tau})^{t-1-t'}\mathbb{E}[\|   \bar{w}_{t'}  \|^2] \notag \\
    & \quad + ( \tau\alpha_{p}\eta S_{p}  + S_{x}) \sum_{t'=s_t\tau}^{t-1}(1+\frac{26}{25\tau})^{t-1-t'}\frac{1}{K}\sum_{k=1}^{K}\mathbb{E}[\| x^{(k)}_{t'} - \bar{x}_{t'} \|^2]\notag \\
    & \quad + ( \tau\alpha_{q}\eta S_{q} + S_{y} ) \sum_{t'=s_t\tau}^{t-1}(1+\frac{26}{25\tau})^{t-1-t'}\frac{1}{K}\sum_{k=1}^{K}\mathbb{E}[\| y^{(k)}_{t'} - \bar{y}_{t'} \|^2]  \notag \\
    & \quad +  (\tau\alpha_{r}\eta S_{r}+ S_{z})\sum_{t'=s_t\tau}^{t-1}(1+\frac{26}{25\tau})^{t-1-t'}\frac{1}{K}\sum_{k=1}^{K}\mathbb{E}[\| z^{(k)}_{t'} - \bar{z}_{t'} \|^2] \notag \\
    & \quad + 72\tau^3\gamma_1^2 \alpha_{x}^2\alpha_{p}^2\eta^4\sum_{t'=s_t\tau}^{t-1}(1+\frac{26}{25\tau})^{t-1-t'} \frac{1}{K}\sum_{k=1}^{K} \mathbb{E}[\| u^{(k)}_{t'} - \bar{u}_{t'} \|^2] \notag \\
    & \quad + 72\tau^3\gamma_1^2\alpha_{y}^2\alpha_{q}^2\eta^4 \sum_{t'=s_t\tau}^{t-1}(1+\frac{26}{25\tau})^{t-1-t'} \frac{1}{K}\sum_{k=1}^{K} \mathbb{E}[\| v^{(k)}_{t'} - \bar{v}_{t'} \|^2]   \notag \\
     & \quad + 72\tau^3\gamma_2^2\alpha_{z}^2\alpha_{r}^2\eta^4 \sum_{t'=s_t\tau}^{t-1}(1+\frac{26}{25\tau})^{t-1-t'} \frac{1}{K}\sum_{k=1}^{K} \mathbb{E}[\| w^{(k)}_{t'} - \bar{w}_{t'} \|^2]   \notag \\
    & \quad + 72\tau\gamma_1^2 \sum_{t'=s_t\tau}^{t-1}(1+\frac{26}{25\tau})^{t-1-t'} \mathbb{E}[\| \bar{p}_{t'+1} - \bar{p}_{t'} \|^2] + 72\tau\gamma_1^2 \sum_{t'=s_t\tau}^{t-1}(1+\frac{26}{25\tau})^{t-1-t'} \mathbb{E}[\| \bar{q}_{t'+1} - \bar{q}_{t'} \|^2]\notag \\
    & \quad  + 72\tau\gamma_2^2 \sum_{t'=s_t\tau}^{t-1}(1+\frac{26}{25\tau})^{t-1-t'} \mathbb{E}[\| \bar{r}_{t'+1} - \bar{r}_{t'} \|^2] + S\sum_{t'=s_t\tau}^{t-1}(1+\frac{26}{25\tau})^{t-1-t'} \ . \notag 
\end{align}
Moreover, since $ (1+\frac{26}{25\tau})^{t-1-t'} \leq  (1+\frac{26}{25\tau})^{\tau} \leq 3$, denoting
\begin{align}
     & S_{x}' = 3\tau^3\alpha_{x}^2\eta^2\Big(\tau\alpha_{p}\eta S_{p}  + S_{x} + 72\tau\gamma_1^2 \alpha_{p}^2\eta^2 \Big) \ , \notag \\
     & S_{y}' = 3\tau^3\alpha_{y}^2\eta^2\Big(\tau\alpha_{q}\eta S_{q}  + S_{y} + 72\tau\gamma_1^2 \alpha_{q}^2\eta^2 \Big) \ , \notag \\
     & S_{z}' = 3\tau^3\alpha_{z}^2\eta^2\Big(\tau\alpha_{r}\eta S_{r}  + S_{y} + 72\tau\gamma_2^2 \alpha_{r}^2\eta^2 \Big)\ , \notag 
\end{align}
applying Lemma~\ref{lemma:consensus-error-variables} and summing up $t'$ from $s_t\tau$ to $(s_t+1)\tau -1$, we obtain
\begin{align}
	& \quad \sum_{t'=s_t\tau}^{(s_t+1)\tau-1}\Bigg(\rho^2\frac{1}{K}\sum_{k=1}^{K}\mathbb{E}[\|u^{(k)}_{1, t'} -\bar{u}_{1, t'}\|^2]  + \frac{1}{K}\sum_{k=1}^{K}\mathbb{E}[\|u^{(k)}_{2, t'} -\bar{u}_{2, t'}\|^2]  +\frac{1}{K}\sum_{k=1}^{K}\mathbb{E}[\|u^{(k)}_{3, t'} -\bar{u}_{3, t'}\|^2] \notag  \\
	& \qquad \qquad + \rho^2\frac{1}{K}\sum_{k=1}^{K}\mathbb{E}[\|v^{(k)}_{1, t'} -\bar{v}_{1, t'}\|^2]  + \frac{1}{K}\sum_{k=1}^{K}\mathbb{E}[\|v^{(k)}_{2, t'} -\bar{v}_{2, t'}\|^2]  + \frac{1}{K}\sum_{k=1}^{K}\mathbb{E}[\|w^{(k)}_{t'} -\bar{w}_{t'}\|^2] \Bigg)\notag \\
    & \leq 144\tau^2 \alpha_x^2\eta^2L_{\rho,\gamma}^2 \sum_{t'=s_t\tau}^{(s_t+1)\tau-1}\mathbb{E}[\| \bar{u}_{t'} \|^2]  + 96\tau^2\alpha_y^2\eta^2L_{\rho,\gamma}^2 \sum_{t'=s_t\tau}^{(s_t+1)\tau-1}\mathbb{E}[\| \bar{v}_{t'} \|^2]  \notag \\
    & \quad +  24\tau^2 \alpha_z^2\eta^2L_{\rho,\gamma}^2 \sum_{t'=s_t\tau}^{t-1}\mathbb{E}[\|   \bar{w}_{t'}  \|^2] +  3\tau S\sum_{t'=s_t\tau}^{(s_t+1)\tau-1} + 216\tau^2\gamma_1^2 \sum_{t'=s_t\tau}^{(s_t+1)\tau-1} \mathbb{E}[\| \bar{p}_{t'+1} - \bar{p}_{t'} \|^2]  \notag \\
    & \quad  + 216\tau^2\gamma_1^2\sum_{t'=s_t\tau}^{(s_t+1)\tau-1}  \mathbb{E}[\| \bar{q}_{t'+1} - \bar{q}_{t'} \|^2] + 216\tau^2\gamma_2^2\sum_{t'=s_t\tau}^{(s_t+1)\tau-1}  \mathbb{E}[\| \bar{r}_{t'+1} - \bar{r}_{t'} \|^2] \notag \\
    & \quad +  3S_{x}'\rho^2\sum_{t'=s_t\tau}^{(s_t+1)\tau-1} \frac{1}{K}\sum_{k=1}^{K}\mathbb{E}[\| u^{(k)}_{1,t'} - \bar{u}_{1,t'} \|^2] +  3S_{x}'\sum_{t'=s_t\tau}^{(s_t+1)\tau-1} \frac{1}{K}\sum_{k=1}^{K}\mathbb{E}[\| u^{(k)}_{2,t'} - \bar{u}_{2,t'} \|^2] \notag \\
    & \quad +  3S_{x}'\sum_{t'=s_t\tau}^{(s_t+1)\tau-1} \frac{1}{K}\sum_{k=1}^{K}\mathbb{E}[\| u^{(k)}_{3,t'} - \bar{u}_{3,t'} \|^2] +  2S_{y}'\rho^2\sum_{t'=s_t\tau}^{(s_t+1)\tau-1} \frac{1}{K}\sum_{k=1}^{K}\mathbb{E}[\| v^{(k)}_{1,t'} - \bar{v}_{1,t'} \|^2] \notag \\
    & \quad +  2S_{y}'\sum_{t'=s_t\tau}^{(s_t+1)\tau-1} \frac{1}{K}\sum_{k=1}^{K}\mathbb{E}[\| v^{(k)}_{2,t'} - \bar{v}_{2,t'} \|^2]  + S_{z}' \sum_{t'=s_t\tau}^{(s_t+1)\tau-1} \frac{1}{K}\sum_{k=1}^{K}\mathbb{E}[\| w^{(k)}_{t'} - \bar{w}_{t'} \|^2]  \ . \notag 
\end{align}
Next, we would like to enforce $S_{x}' \leq \frac{1}{6}$, and make sure $1 - 3S_{x}' \geq \frac{1}{2}$. In particular, we decompose $S_{x}'$ as follows:
\begin{align}
    S_{x}' & = 3\tau^3\alpha_{x}^2\eta^2\Big(\tau\alpha_{p}\eta S_{p}  + S_{x} + 72\tau\gamma_1^2 \alpha_{p}^2\eta^2 \Big) \notag \\
    & = 3\tau^4\alpha_{x}^2\eta^2\Big(\underbrace{\tau\alpha_{p}\eta 36\gamma_1^2\alpha_{p}^2\eta^2}_{S_{x,1}} + \underbrace{\tau\alpha_{p}\eta12\gamma_1^2\beta_{x}^2\eta^4}_{S_{x,2}}  \\
    & + \underbrace{24\eta^4 (\beta_x^2L_{f}^{2}\rho^2 + \beta_y^2 L_{f}^{2}\rho^2 + 4\beta_x^2 L_{g}^{2} +3\beta_y^2 L_{g}^{2} + 3\beta_z^2L_{g}^{2})}_{S_{x,3}} +  \underbrace{12\gamma_1^2\eta^4\beta_x^2 }_{S_{x,4}} +  \underbrace{72\gamma_1^2 \alpha_{p}^2\eta^2}_{S_{x,5}} \Big) \ . \notag 
\end{align}
By imposing $\alpha_{pq} \leq 10\gamma$, and recalling the definition of $\alpha_{pq}$ in Eq.~(\ref{eq:def_y_z_rho}), we obtain $\alpha_{x} \leq \frac{96\times1280\gamma_1^2L_{\rho}^2}{(\gamma_1-L_{\rho})^3\mu} \leq \frac{96\times1280(\gamma_1 + L_{\rho})^4}{(\gamma-L_{\rho})^3\mu}$. Since $\alpha_{x} \leq 10$ holds is enforced by our parameter choice, the above inequality is automatically satisfied. Therefore, with $\eta \leq \frac{1}{360 \tau L_{\rho,\gamma} }$, we have
\begin{align}
     &  3\tau^4\alpha_{x}^2\eta^2 S_{x,1} = 3\tau^4\alpha_{x}^2\eta^2\tau\gamma_1^2 \times 36\alpha_{p}^3\eta^3 \leq \frac{108\times100\tau^5\gamma_1^2\times1000\gamma_1^3}{360^5\tau^5L_{\rho,\gamma}^{5}} \leq \frac{1}{30} \ . \notag 
\end{align}
Similarly, by setting $\beta_{x} \leq \frac{L_{\rho,\gamma}^2}{K}$, we derive
\begin{align}
    &  3\tau^4\alpha_{x}^2\eta^2 S_{x,2} = 36\tau^4\alpha_{x}^2\eta^2\tau\gamma^2\beta_{x}^2\eta^4\alpha_{p}\eta \leq \frac{36\times100\tau^5\gamma_1^2\times10\gamma L_{\rho,\gamma}^4}{360^7\tau^7L_{\rho,\gamma}^{7}\times K^2} \leq \frac{1}{30} \ . \notag 
\end{align}
In addition, by setting $\beta_{y} \leq \frac{L_{\rho,\gamma}}{K}$, $\beta_{z} \leq \frac{L_{\rho,\gamma}}{K}$, we obtain
\begin{align}
    & \quad 3\tau^4\alpha_{x}^2\eta^2 S_{x,3} =  72\tau^4\alpha_{x}^2\eta^6 (\beta_x^2L_{f}^{2}\rho^2 + \beta_y^2 L_{f}^{2}\rho^2 + 4\beta_x^2 L_{g}^{2} + 3\beta_y^2 L_{g}^{2} + 3\beta_z^2L_{g}^{2}) \notag \\
    & \leq \frac{72\times200\tau^4(L_f^2\rho^2+5L_g^2) L_{\rho,\gamma}^4}{360^6\tau^6L_{\rho,\gamma}^{6}\times K^2} \leq \frac{1}{30} \ . \notag 
\end{align}
For the remaining two terms, we also derive
\begin{align}
    & 36\tau^4\alpha_{x}^2\eta^2\gamma_1^2\eta^4\beta_x^2 \leq \frac{3600\tau^4\gamma_1^2L_{\rho,\gamma}^4}{360^6\tau^6L_{\rho,\gamma}^{6}\times K^2} \leq \frac{1}{30} \ , \notag \\
    & 216\tau^4\alpha_{x}^2\eta^26\gamma_1^2\alpha_{p}^2\eta^2 \leq \frac{216\times100\tau^4\gamma_1^2\times100\gamma_1^2}{360^4\tau^4L_{\rho,\gamma}^{4}} \leq \frac{1}{30} \ . \notag 
\end{align}
In a similar manner, we can also make sure $1 - 3S_{y}' \geq \frac{1}{2}$,  $1 - 3S_{z}' \geq \frac{1}{2}$. 

Therefore, we obtain
\begin{align}
	& \quad \frac{1}{T}\sum_{t=0}^{T-1}\Bigg(\rho^2\frac{1}{K}\sum_{k=1}^{K}\mathbb{E}[\|u^{(k)}_{1, t} -\bar{u}_{1, t}\|^2]  + \frac{1}{K}\sum_{k=1}^{K}\mathbb{E}[\|u^{(k)}_{2, t} -\bar{u}_{2, t}\|^2]  +\frac{1}{K}\sum_{k=1}^{K}\mathbb{E}[\|u^{(k)}_{3, t} -\bar{u}_{3, t}\|^2]\notag \\
	& \qquad \qquad + \rho^2\frac{1}{K}\sum_{k=1}^{K}\mathbb{E}[\|v^{(k)}_{1, t} -\bar{v}_{1, t}\|^2]  + \frac{1}{K}\sum_{k=1}^{K}\mathbb{E}[\|v^{(k)}_{2, t} -\bar{v}_{2, t}\|^2]  + \frac{1}{K}\sum_{k=1}^{K}\mathbb{E}[\|w^{(k)}_{t} -\bar{w}_{t}\|^2] \Bigg)\notag \\
    & \leq 288\tau^2 \alpha_x^2\eta^2L_{\rho,\gamma}^2 \frac{1}{T}\sum_{t=0}^{T-1}\mathbb{E}[\| \bar{u}_{t} \|^2]  + 192\tau^2\alpha_y^2\eta^2L_{\rho,\gamma}^2 \frac{1}{T}\sum_{t=0}^{T-1}\mathbb{E}[\| \bar{v}_{t} \|^2] +  48\tau^2 \alpha_z^2\eta^2L_{\rho,\gamma}^2 \frac{1}{T}\sum_{t=0}^{T-1}\mathbb{E}[\|   \bar{w}_{t}  \|^2]  \notag \\
    & \quad  + 432\tau^2\gamma_1^2 \frac{1}{T}\sum_{t=0}^{T-1} \mathbb{E}[\| \bar{p}_{t+1} - \bar{p}_{t} \|^2] + 432\tau^2\gamma_1^2\frac{1}{T}\sum_{t=0}^{T-1}  \mathbb{E}[\| \bar{q}_{t+1} - \bar{q}_{t} \|^2] \notag \\
    & \quad + 432\tau^2\gamma_2^2\frac{1}{T}\sum_{t=0}^{T-1}  \mathbb{E}[\| \bar{r}_{t+1} - \bar{r}_{t} \|^2] + 3\tau S \ . \notag 
\end{align}
The proof is completed.
\end{proof}

\section{Proof of Theorem~\ref{theorem:main}}\label{app:proof}
We propose a novel potential function as follows:
\begin{align}\label{eq:potential_total}
    & \mathcal{P}_{t} = \Psi_{t} + c_1\rho^2\mathcal{G}_{u, 1, t}  + c_2\mathcal{G}_{u, 2, t}   + c_3\mathcal{G}_{u, 3, t}   + c_4\rho^2\mathcal{G}_{v, 1, t} + c_5 \mathcal{G}_{v, 2, t}  + c_6\mathcal{G}_{w, t}  \ . 
\end{align}
In particular, we have
\begin{align}
    & \quad \mathbb{E}[\| \nabla_{x} F(\bar{x}_{t}, \bar{y}_{t}, \bar{z}_{t}; \bar{p}_{t}, \bar{q}_{t}, \bar{r}_{t}) - \bar{u}_{t}\|^2] \notag \\
    & =  \mathbb{E}[\| \rho \nabla_{1} f(\bar{x}_{t}, \bar{y}_{t}) + \nabla_{1} g_{\gamma_1}(\bar{x}_{t}, \bar{y}_{t};\bar{p}_{t}, \bar{q}_{t}) - \nabla_{1} g_{\gamma_2}(\bar{x}_{t}, \bar{z}_{t}; \bar{r}_{t})  - \bar{u}_{t}\|^2] \notag \\
    & \leq 2\mathbb{E}[\| \rho \nabla_{1} f(\bar{x}_{t}, \bar{y}_{t}) + \nabla_{1} g_{\gamma_1}(\bar{x}_{t}, \bar{y}_{t};\bar{p}_{t}, \bar{q}_{t}) - \nabla_{1} g_{\gamma_2}(\bar{x}_{t}, \bar{z}_{t}; \bar{r}_{t}) \notag \\
    & \quad \underbrace{ - \frac{1}{K}\sum_{k=1}^{K}\Big( \rho \nabla_{1} f^{(k)}({x}_{t}^{(k)}, {y}_{t}^{(k)}) + \nabla_{1} g_{\gamma_1}^{(k)}({x}_{t}^{(k)}, {y}_{t}^{(k)};{p}_{t}^{(k)}, {q}_{t}^{(k)}) - \nabla_{1} g_{\gamma_1}^{(k)}({x}_{t}^{(k)}, {z}_{t}^{(k)}; {r}_{t}^{(k)})  \Big) \|^2] }_{T_1} \notag \\
    &   + 2\underbrace{\mathbb{E}[\|\frac{1}{K}\sum_{k=1}^{K} \Big( \rho \nabla_{1} f^{(k)}({x}_{t}^{(k)}, {y}_{t}^{(k)}) + \nabla_{1} g_{\gamma_1}^{(k)}({x}_{t}^{(k)}, {y}_{t}^{(k)};{p}_{t}^{(k)}, {q}_{t}^{(k)}) - \nabla_{1} g_{\gamma_1}^{(k)}({x}_{t}^{(k)}, {z}_{t}^{(k)}; {r}_{t}^{(k)})  \Big)  - \bar{u}_{t} \|^2]}_{T_2} \ . \notag 
\end{align}
The second term can be simply bounded as
\begin{align}
    & T_2 \leq 3\rho^2\mathcal{G}_{u, 1, t}  +  3\mathcal{G}_{u, 2, t}   +  3\mathcal{G}_{u, 3, t} \ . \notag 
\end{align}
For the first term, from the following inequality
\begin{align}
    & \quad  \mathbb{E}[\| \nabla_{1} g_{\gamma_1}(\bar{x}_{t}, \bar{y}_{t};\bar{p}_{t}, \bar{q}_{t})  - \frac{1}{K}\sum_{k=1}^{K}\nabla_{1} g_{\gamma_1}({x}_{t}^{(k)}, {y}_{t}^{(k)};{p}_{t}^{(k)}, {q}_{t}^{(k)}) \|^2] \notag   \\
    & =  \mathbb{E}[\| \nabla_{1} g(\bar{x}_{t}, \bar{y}_{t}) + \gamma(\bar{x}_{t} - \bar{p}_{t}) - \frac{1}{K}\sum_{k=1}^{K} \nabla_{1} g^{(k)}({x}_{t}^{(k)}, {y}_{t}^{(k)}) - \frac{1}{K}\sum_{k=1}^{K} \gamma({x}_{t}^{(k)} - {p}_{t}^{(k)}) \|^2] \notag \\
    & \leq \frac{1}{K}\sum_{k=1}^{K} \mathbb{E}[\| \nabla_{1} g^{(k)}(\bar{x}_{t}, \bar{y}_{t})- \ \nabla_{1} g({x}_{t}^{(k)}, {y}_{t}^{(k)})\|^2] \notag \\
    & \leq L_g^2\frac{1}{K}\sum_{k=1}^{K} \mathbb{E}[\| \bar{x}_{t} - {x}_{t}^{(k)} \|^2] + L_g^2\frac{1}{K}\sum_{k=1}^{K} \mathbb{E}[\| \bar{y}_{t} - {y}_{t}^{(k)} \|^2] \ , \notag 
\end{align}
we obtain
\begin{align}
    & T_1 \leq 3\rho^2 \mathbb{E}[\| \nabla_{1} f(\bar{x}_{t}, \bar{y}_{t}) - \frac{1}{K}\sum_{k=1}^{K}\nabla_{1} f^{(k)}({x}_{t}^{(k)}, {y}_{t}^{(k)}) \|^2] \notag  \\
    & \quad + 3\mathbb{E}[\| \nabla_{1} g_{\gamma_1}(\bar{x}_{t}, \bar{y}_{t};\bar{p}_{t}, \bar{q}_{t})  - \frac{1}{K}\sum_{k=1}^{K}\nabla_{1} g_{\gamma_1}^{(k)}({x}_{t}^{(k)}, {y}_{t}^{(k)};{p}_{t}^{(k)}, {q}_{t}^{(k)}) \|^2]  \notag \\
     & \quad + 3\mathbb{E}[\|  \nabla_{1}  g_{\gamma_2}(\bar{x}_{t}, \bar{z}_{t};  \bar{r}_{t})  - \frac{1}{K}\sum_{k=1}^{K} \nabla_{1} g_{\gamma_2}^{(k)}({x}_{t}^{(k)}, {z}_{t}^{(k)}; {r}_{t}^{(k)}) \|^2]  \notag \\
    & \leq 3(\rho^2L_f^2+2L_g^2) \frac{1}{K}\sum_{k=1}^{K} \mathbb{E}[\| \bar{x}_{t} - {x}_{t}^{(k)} \|^2] + 3(\rho^2L_f^2+L_g^2)\frac{1}{K}\sum_{k=1}^{K} \mathbb{E}[\| \bar{y}_{t} - {y}_{t}^{(k)} \|^2] \notag \\
    & \quad + 3L_g^2 \frac{1}{K}\sum_{k=1}^{K}\mathbb{E}[\| \bar{z}_{t} - {z}_{t}^{(k)} \|^2]   \ .  \notag 
\end{align}
Thus, combing $T_1$ and $T_2$, we derive
\begin{align}
    &  \mathbb{E}[\| \nabla_{x} F(\bar{x}_{t}, \bar{y}_{t}, \bar{z}_{t}; \bar{p}_{t}, \bar{q}_{t}, \bar{r}_{t}) - \bar{u}_{t}\|^2] \leq 6(\rho^2L_f^2+2L_g^2) \frac{1}{K}\sum_{k=1}^{K} \mathbb{E}[\| \bar{x}_{t} - {x}_{t}^{(k)} \|^2] \notag \\
    & + 6(\rho^2L_f^2+L_g^2)\frac{1}{K}\sum_{k=1}^{K} \mathbb{E}[\| \bar{y}_{t} - {y}_{t}^{(k)} \|^2] + 6L_g^2 \frac{1}{K}\sum_{k=1}^{K}\mathbb{E}[\| \bar{z}_{t} - {z}_{t}^{(k)} \|^2]  \notag \\
    & + 6\rho^2\mathcal{G}_{u, 1, t}  +  6\mathcal{G}_{u, 2, t}   +  6\mathcal{G}_{u, 3, t}  \ . \notag 
\end{align}
Similarly, we obtain
\begin{align}
    & \quad \mathbb{E}[\|\nabla_{y}  F(\bar{x}_{t}, \bar{y}_{t}, \bar{z}_{t}; \bar{p}_{t}, \bar{q}_{t}, \bar{r}_{t}) -  \bar{v}_{t}\|^2] =  \mathbb{E}[\|\nabla_{2}  f(\bar{x}_{t}, \bar{y}_{t}) + \nabla_{2}  g_{\gamma_1}(\bar{x}_{t}, \bar{y}_{t}; \bar{p}_{t}, \bar{q}_{t}) -  \bar{v}_{t}\|^2] \notag  \\
    & \leq 4(\rho^2L_f^2+L_g^2) \frac{1}{K}\sum_{k=1}^{K} \mathbb{E}[\| \bar{x}_{t} - {x}_{t}^{(k)} \|^2] +  4(\rho^2L_f^2+L_g^2)\frac{1}{K}\sum_{k=1}^{K} \mathbb{E}[\| \bar{y}_{t} - {y}_{t}^{(k)} \|^2]    \notag \\
    & \quad + 4\rho^2\mathcal{G}_{v, 1, t} +  4\mathcal{G}_{v, 2, t} \ , \notag 
\end{align}
and 
\begin{align}
     & \quad \mathbb{E}[\|\nabla_{z}  F(\bar{x}_{t}, \bar{y}_{t}, \bar{z}_{t}; \bar{p}_{t}, \bar{q}_{t}, \bar{r}_{t}) -  \bar{w}_{t}\|^2] = \mathbb{E}[\| \nabla_{2}  g_{\gamma_2}(\bar{x}_{t}, \bar{z}_{t}; \bar{r}_{t}) -  \bar{w}_{t}\|^2] \notag \\
     & \leq 2L_g^2 \frac{1}{K}\sum_{k=1}^{K} \mathbb{E}[\| \bar{x}_{t} - {x}_{t}^{(k)} \|^2] + 2L_g^2 \frac{1}{K}\sum_{k=1}^{K}\mathbb{E}[\| \bar{z}_{t} - {z}_{t}^{(k)} \|^2] +  2\mathcal{G}_{w, t} \ . \notag 
\end{align}

Then, from Lemma~\ref{lemma:potential_psi}, \ref{lemma:grad_orig}, \ref{lemma:grad_new}, by setting 
\begin{align}\label{eq:c_value}
    c_1 = c_2 = c_3 = \frac{6\alpha_{x}}{\beta_{x}\eta} \ , \quad c_4  = c_5 = \frac{2\alpha_{y}}{\beta_{y}\eta} \ , \quad c_6 = \frac{3\alpha_{z}}{\beta_{z}\eta}  \ , 
\end{align}
we can eliminate the gradient error terms, with denotation 
\begin{align}\label{eq:def_mc}
    & \mathcal{M}_1 = \frac{1}{K} (2L_{f}^{2}\rho^2c_1 +2L_{g}^{2}c_2+ 2L_{g}^{2}c_3  +  2L_{f}^{2}\rho^2c_4 + 2L_{g}^{2}c_5 + 2L_{g}^{2}c_6)\ , \notag \\
    & \mathcal{M}_2 = \frac{1}{K}  (2L_{f}^{2}\rho^2c_1 +  2L_{g}^{2}c_2 + 2L_{f}^{2}\rho^2c_4  + 2L_{g}^{2}c_5)\ , \notag \\
    & \mathcal{M}_3 = \frac{1}{K}( 2L_{g}^{2}c_3 + 2L_{g}^{2}c_6) \ ,  \notag \\
    & \mathcal{M}_4 =  6\alpha_x\eta(\rho^2L_f^2+2L_g^2)+  2\alpha_{y}\eta(\rho^2L_f^2+L_g^2) + 3\alpha_z\eta L_g^2  \ ,  \\
    & \mathcal{M}_5 =  6\alpha_x\eta(\rho^2L_f^2+L_g^2) + 2\alpha_{y}\eta(\rho^2L_f^2+L_g^2)  \ , \notag \\
    & \mathcal{M}_6 = 6\alpha_x\eta L_g^2 + 3\alpha_z\eta L_g^2 \ , \notag \\
    & \mathcal{M}_0 = \frac{2\sigma^2}{K}\Big(\beta_x^2\eta^4\rho^2c_1 + \beta_x^2\eta^4c_2 + \beta_x^2\eta^4c_3  + \beta_y^2\eta^4\rho^2c_4 + \beta_y^2\eta^4c_5  + \beta_z^2\eta^4c_6 \Big) \ , \notag 
\end{align}  
we obtain
\begin{align}
    &  \mathcal{P}_{t+1} -  \mathcal{P}_{t} \leq - \frac{\alpha_x \eta}{8}  \mathbb{E}[\| \nabla_{x} F(\bar{x}_{t}, \bar{y}_{t}, \bar{z}_{t}; \bar{p}_{t}, \bar{q}_{t}, \bar{r}_{t}) \|^2] - \frac{\alpha_{y}\eta}{2} \mathbb{E}[\|\nabla_{y}  F(\bar{x}_{t}, \bar{y}_{t}, \bar{z}_{t}; \bar{p}_{t}, \bar{q}_{t}, \bar{r}_{t}) \|^2] \notag \\
    &  - \frac{\alpha_z\eta}{2} \mathbb{E}[\|\nabla_{z} F(\bar{x}_{t}, \bar{y}_{t}, \bar{z}_{t}; \bar{p}_{t}, \bar{q}_{t}, \bar{r}_{t})\|^2 ]  +   \Big( 16\gamma_1\alpha_{pq}\eta\big(5\omega_2^2\omega_3^2 + C_r\big) - \frac{\gamma_2}{2\alpha_r\eta}\Big) \mathbb{E}[\|\bar{r}_{t+1}-\bar{r}_{t}\|^2] \notag \\
    &  + \Big(1 + \omega^2_1 + 8\alpha_{pq}\eta\omega_5^2\big( 5\omega_2^2 +\frac{2\gamma_2^2}{\mu(\gamma_1-L_\rho)}\big)- \frac{1}{4\alpha_{pq}\eta}\Big)\gamma_1  \big(\mathbb{E}[\|\bar{p}_{t+1}- \bar{p}_{t}\|^2] + \mathbb{E}[\|\bar{q}_{t+1}-\bar{q}_t\|^2]\big) \notag \\
    &  + \Big(4 \alpha_{x}^2\alpha_z\eta^3 L_{g}^2 -  \frac{\alpha_x \eta}{4} \Big) \mathbb{E}[\| \bar{u}_{t}\|^2]   + \Big(\frac{3\alpha_x\alpha_{y}^2 \eta^3L_{\rho}^2}{2}  -  \frac{\alpha_{y}\eta}{4}\Big)  \mathbb{E}[\|  \bar{v}_{t} \|^2]  \notag \\
    & + \Big(\frac{ \alpha^2_z\eta^2 }{2}(L_{F,2}+2L_d+ 2C_\rho)  - \frac{\alpha_z\eta }{16} \Big) \mathbb{E}[\|\bar{w}_{t}\|^2] \notag \\
    &  + \mathcal{M}_1\frac{1}{K}\sum_{k=1}^{K}\mathbb{E}[\| x^{(k)}_{t+1} - x^{(k)}_{t} \|^2] + \mathcal{M}_2\frac{1}{K}\sum_{k=1}^{K}\mathbb{E}[\| y^{(k)}_{t+1} - y^{(k)}_{t} \|^2]  + \mathcal{M}_3\frac{1}{K}\sum_{k=1}^{K}\mathbb{E}[\| z^{(k)}_{t+1} - z^{(k)}_{t} \|^2] \notag \\
    &  +  \mathcal{M}_4 \frac{1}{K}\sum_{k=1}^{K} \mathbb{E}[\| \bar{x}_{t} - {x}_{t}^{(k)} \|^2]  + \mathcal{M}_5 \frac{1}{K}\sum_{k=1}^{K} \mathbb{E}[\| \bar{y}_{t} - {y}_{t}^{(k)} \|^2] + \mathcal{M}_6 \frac{1}{K}\sum_{k=1}^{K}\mathbb{E}[\| \bar{z}_{t} - {z}_{t}^{(k)} \|^2] + \mathcal{M}_0 \ . \notag 
\end{align}

Then, by applying Lemma~\ref{lemma:incremental-x-y-z}, we obtain
\begin{align}
    &  \mathcal{P}_{t+1} -  \mathcal{P}_{t} \leq - \frac{\alpha_x \eta}{8}  \mathbb{E}[\| \nabla_{x} F(\bar{x}_{t}, \bar{y}_{t}, \bar{z}_{t}; \bar{p}_{t}, \bar{q}_{t}, \bar{r}_{t}) \|^2] - \frac{\alpha_{y}\eta}{2} \mathbb{E}[\|\nabla_{y}  F(\bar{x}_{t}, \bar{y}_{t}, \bar{z}_{t}; \bar{p}_{t}, \bar{q}_{t}, \bar{r}_{t}) \|^2] \notag   \\
    & - \frac{\alpha_z\eta}{2} \mathbb{E}[\|\nabla_{z} F(\bar{x}_{t}, \bar{y}_{t}, \bar{z}_{t}; \bar{p}_{t}, \bar{q}_{t}, \bar{r}_{t})\|^2 ]  +   \Big( 16\gamma_1\alpha_{pq}\eta\big(5\omega_2^2\omega_3^2 + C_r\big) - \frac{\gamma_2}{2\alpha_r\eta}\Big) \mathbb{E}[\|\bar{r}_{t+1}-\bar{r}_{t}\|^2] \notag \\
    &  + \Big(1 + \omega^2_1 + 8\alpha_{pq}\eta\omega_5^2\big( 5\omega_2^2 +\frac{2\gamma_2^2}{\mu(\gamma_1-L_\rho)}\big)- \frac{1}{4\alpha_{pq}\eta}\Big)\gamma_1  \big(\mathbb{E}[\|\bar{p}_{t+1}- \bar{p}_{t}\|^2] + \mathbb{E}[\|\bar{q}_{t+1}-\bar{q}_t\|^2]\big) \notag \\
    & + \Big(4 \alpha_{x}^2\alpha_z\eta^3 L_{g}^2 +  6\alpha_x^2\eta^2\mathcal{M}_1 -  \frac{\alpha_x \eta}{4} \Big) \mathbb{E}[\| \bar{u}_{t}\|^2]  + \Big(\frac{3\alpha_x\alpha_{y}^2 \eta^3L_{\rho}^2}{2}   + 4\alpha_y^2\eta^2\mathcal{M}_2 -  \frac{\alpha_{y}\eta}{4}\Big)  \mathbb{E}[\|  \bar{v}_{t} \|^2]  \notag \\
    &  + \Big(\frac{ \alpha^2_z\eta^2 }{2}(L_{F,2}+2L_d+ 2C_\rho) +  2\alpha_z^2\eta^2\mathcal{M}_3 - \frac{\alpha_z\eta }{16}\Big) \mathbb{E}[\|\bar{w}_{t}\|^2]  \notag \\
    &   + 6\alpha_x^2\eta^2\mathcal{M}_1 \Big(\rho^2 \frac{1}{K}\sum_{k=1}^{K}\mathbb{E}[\| u^{(k)}_{1, t} - \bar{u}_{1, t} \|^2] + \frac{1}{K}\sum_{k=1}^{K}\mathbb{E}[\| u^{(k)}_{2, t} - \bar{u}_{2, t} \|^2]  + \frac{1}{K}\sum_{k=1}^{K}\mathbb{E}[\| u^{(k)}_{3, t} - \bar{u}_{3, t} \|^2] \Big)\notag \\
    & + 4\alpha_y^2\eta^2\mathcal{M}_2 \Big(\rho^2\frac{1}{K}\sum_{k=1}^{K}\mathbb{E}[\| v^{(k)}_{1, t} - \bar{v}_{1, t} \|^2]  +  \frac{1}{K}\sum_{k=1}^{K}\mathbb{E}[\| v^{(k)}_{2, t} - \bar{v}_{2, t} \|^2] \Big)\notag \\
    &    +  2\alpha_z^2\eta^2\mathcal{M}_3 \frac{1}{K}\sum_{k=1}^{K}\mathbb{E}[\|    w^{(k)}_{t} -  \bar{w}_{t}   \|^2]  + \mathcal{M}_4 \frac{1}{K}\sum_{k=1}^{K} \mathbb{E}[\| \bar{x}_{t} - {x}_{t}^{(k)} \|^2]  \notag \\
    & +\mathcal{M}_5  \frac{1}{K}\sum_{k=1}^{K} \mathbb{E}[\| \bar{y}_{t} - {y}_{t}^{(k)} \|^2] + \mathcal{M}_6 \frac{1}{K}\sum_{k=1}^{K}\mathbb{E}[\| \bar{z}_{t} - {z}_{t}^{(k)} \|^2] + \mathcal{M}_0 \ . \notag 
\end{align}

By summing over $t$ from $0$ to $T -1$, applying Lemma~\ref{lemma:consensus-error-variables} 
we obtain
\begin{align}
    & \frac{\mathcal{P}_{T} -  \mathcal{P}_{0}}{ T} \leq - \frac{\alpha_x \eta}{8} \frac{1}{T}\sum_{t=0}^{T-1}  \mathbb{E}[\| \nabla_{x} F(\bar{x}_{t}, \bar{y}_{t}, \bar{z}_{t}; \bar{p}_{t}, \bar{q}_{t}, \bar{r}_{t}) \|^2] \notag \\
    & - \frac{\alpha_{y}\eta}{2} \frac{1}{T}\sum_{t=0}^{T-1} \mathbb{E}[\|\nabla_{y}  F(\bar{x}_{t}, \bar{y}_{t}, \bar{z}_{t}; \bar{p}_{t}, \bar{q}_{t}, \bar{r}_{t}) \|^2] - \frac{\alpha_z\eta}{2} \frac{1}{T}\sum_{t=0}^{T-1} \mathbb{E}[\|\nabla_{z} F(\bar{x}_{t}, \bar{y}_{t}, \bar{z}_{t}; \bar{p}_{t}, \bar{q}_{t}, \bar{r}_{t})\|^2 ]  \notag  \\
    &  +   \Big( 16\gamma_1\alpha_{pq}\eta\big(5\omega_2^2\omega_3^2 + C_r\big) - \frac{\gamma_2}{2\alpha_r\eta}\Big) \frac{1}{T}\sum_{t=0}^{T-1} \mathbb{E}[\|\bar{r}_{t+1}-\bar{r}_{t}\|^2] \notag \\
    &  + \Big(1 + \omega^2_1 + 8\alpha_{pq}\eta\omega_5^2\big( 5\omega_2^2 +\frac{2\gamma_2^2}{\mu(\gamma_1-L_\rho)}\big)- \frac{1}{4\alpha_{pq}\eta}\Big)\gamma_1  \frac{1}{T}\sum_{t=0}^{T-1} \big(\mathbb{E}[\|\bar{p}_{t+1}- \bar{p}_{t}\|^2] + \mathbb{E}[\|\bar{q}_{t+1}-\bar{q}_t\|^2]\big) \notag \\
    & + \Big(4 \alpha_{x}^2\alpha_z\eta^3 L_{g}^2 +  6\alpha_x^2\eta^2\mathcal{M}_1 -  \frac{\alpha_x \eta}{4} \Big) \frac{1}{T}\sum_{t=0}^{T-1}\mathbb{E}[\| \bar{u}_{t}\|^2] \notag \\
    &  + \Big(\frac{3\alpha_x\alpha_{y}^2 \eta^3L_{\rho}^2}{2} + 4\alpha_y^2\eta^2\mathcal{M}_2 -  \frac{\alpha_{y}\eta}{4}\Big)  \frac{1}{T}\sum_{t=0}^{T-1}\mathbb{E}[\|  \bar{v}_{t} \|^2]  \notag \\
    &  + \Big(\frac{ \alpha^2_z\eta^2 }{2}(L_{F,2}+2L_d+ 2C_\rho) +  2\alpha_z^2\eta^2\mathcal{M}_3 - \frac{\alpha_z\eta }{16}\Big) \frac{1}{T}\sum_{t=0}^{T-1}\mathbb{E}[\|\bar{w}_{t}\|^2]  \notag \\
	&   + \alpha_x^2\eta^2(6\mathcal{M}_1+3 \tau^2\mathcal{M}_4) \frac{1}{T}\sum_{t=0}^{T-1}\Big(\rho^2\frac{1}{K}\sum_{k=1}^{K}\mathbb{E}[\| u^{(k)}_{1, t} - \bar{u}_{1, t} \|^2] + \frac{1}{K}\sum_{k=1}^{K}\mathbb{E}[\| u^{(k)}_{2, t} - \bar{u}_{2, t} \|^2] \notag \\
    & \qquad \qquad + \frac{1}{K}\sum_{k=1}^{K}\mathbb{E}[\| u^{(k)}_{3, t} - \bar{u}_{3, t} \|^2] \Big) \notag \\
    & + \alpha_y^2\eta^2(4\mathcal{M}_2+2\tau^2\mathcal{M}_5)\frac{1}{T}\sum_{t=0}^{T-1}\Big( \rho^2 \frac{1}{K}\sum_{k=1}^{K}\mathbb{E}[\| v^{(k)}_{1, t} - \bar{v}_{1, t} \|^2] + \frac{1}{K}\sum_{k=1}^{K}\mathbb{E}[\| v^{(k)}_{2, t} - \bar{v}_{2, t} \|^2] \Big)  \notag \\
    & + \alpha_z^2\eta^2(2\mathcal{M}_3+\tau^2\mathcal{M}_6)\frac{1}{T}\sum_{t=0}^{T-1} \frac{1}{K}\sum_{k=1}^{K}\mathbb{E}[\|    w^{(k)}_{t} -  \bar{w}_{t}   \|^2] + \mathcal{M}_{0} \ . \notag 
\end{align}
Moreover, by defining
\begin{align}\label{eq:def_m}
    \mathcal{M} & = \frac{6(\alpha_x^2+\alpha_y^2+\alpha_z^2)\eta^2}{K} (2L_{f}^{2}\rho^2c_1 +2L_{g}^{2}c_2+ 2L_{g}^{2}c_3  +  2L_{f}^{2}\rho^2c_4 + 2L_{g}^{2}c_5 + 2L_{g}^{2}c_6) \notag \\
    & \quad + 3 \tau^2(\alpha_x^2+\alpha_y^2+\alpha_z^2)\eta^2 \Big(6\alpha_x\eta(\rho^2L_f^2+2L_g^2)+  2\alpha_{y}\eta(\rho^2L_f^2+L_g^2) + 6\alpha_z\eta L_g^2\Big) \ , 
\end{align}
and applying Lemma~\ref{lemma:consensus_total}, we obtain
\begin{align}
    & \frac{\mathcal{P}_{T} -  \mathcal{P}_{0}}{ T} \leq - \frac{\alpha_x \eta}{8} \frac{1}{T}\sum_{t=0}^{T-1}  \mathbb{E}[\| \nabla_{x} F(\bar{x}_{t}, \bar{y}_{t}, \bar{z}_{t}; \bar{p}_{t}, \bar{q}_{t}, \bar{r}_{t}) \|^2] \notag \\
    & - \frac{\alpha_{y}\eta}{2} \frac{1}{T}\sum_{t=0}^{T-1} \mathbb{E}[\|\nabla_{y}  F(\bar{x}_{t}, \bar{y}_{t}, \bar{z}_{t}; \bar{p}_{t}, \bar{q}_{t}, \bar{r}_{t}) \|^2]  - \frac{\alpha_z\eta}{2} \frac{1}{T}\sum_{t=0}^{T-1} \mathbb{E}[\|\nabla_{z} F(\bar{x}_{t}, \bar{y}_{t}, \bar{z}_{t}; \bar{p}_{t}, \bar{q}_{t}, \bar{r}_{t})\|^2 ] \notag \\
    &  +   \Big( 16\gamma_1\alpha_{pq}\eta\big(5\omega_2^2\omega_3^2 + C_r\big)+ 432\tau^2\gamma_2^2\mathcal{M} - \frac{\gamma_2}{2\alpha_r\eta}\Big) \frac{1}{T}\sum_{t=0}^{T-1} \mathbb{E}[\|\bar{r}_{t+1}-\bar{r}_{t}\|^2] \notag \\
    &  + \Big(1 + \omega^2_1 + 8\alpha_{pq}\eta\omega_5^2\big( 5\omega_2^2 +\frac{2\gamma_2^2}{\mu(\gamma_1-L_\rho)}\big) + 432\tau^2\gamma_1\mathcal{M}- \frac{1}{4\alpha_{pq}\eta}\Big)\gamma_1  \frac{1}{T}\sum_{t=0}^{T-1} \big(\mathbb{E}[\|\bar{p}_{t+1}- \bar{p}_{t}\|^2] \notag \\
    & \qquad \qquad \qquad + \mathbb{E}[\|\bar{q}_{t+1}-\bar{q}_t\|^2]\big) \notag \\
    & + \Big(4 \alpha_{x}^2\alpha_z\eta^3 L_{g}^2 +  6\alpha_x^2\eta^2\mathcal{M}_1 + 288\tau^2 \alpha_x^2\eta^2L_{\rho,\gamma}^2\mathcal{M} -  \frac{\alpha_x \eta}{4} \Big) \frac{1}{T}\sum_{t=0}^{T-1}\mathbb{E}[\| \bar{u}_{t}\|^2] \notag \\
    &  + \Big( \frac{3\alpha_x\alpha_{y}^2 \eta^3L_{\rho}^2}{2} + 4\alpha_y^2\eta^2\mathcal{M}_2 + 192\tau^2\alpha_y^2\eta^2L_{\rho,\gamma}^2 \mathcal{M} -  \frac{\alpha_{y}\eta}{4}\Big)  \frac{1}{T}\sum_{t=0}^{T-1}\mathbb{E}[\|  \bar{v}_{t} \|^2] + \mathcal{M}_{0}  \notag  \\
    &  + \Big(\frac{ \alpha^2_z\eta^2 }{2}(L_{F,2}+2L_d+ 2C_\rho) +  2\alpha_z^2\eta^2\mathcal{M}_3  +  48\tau^2 \alpha_z^2\eta^2L_{g}^2 \mathcal{M} - \frac{\alpha_z\eta }{16}\Big) \frac{1}{T}\sum_{t=0}^{T-1}\mathbb{E}[\|\bar{w}_{t}\|^2] + 3\tau S\mathcal{M} \ . \notag 
\end{align}
To eliminate $\frac{1}{T}\sum_{t=0}^{T-1} \mathbb{E}[\|\bar{p}_{t+1}- \bar{p}_{t}\|^2] $ and $\frac{1}{T}\sum_{t=0}^{T-1} \mathbb{E}[\|\bar{q}_{t+1}- \bar{q}_{t}\|^2] $, we enforce
\begin{align}
   1 + \omega^2_1 + 8\alpha_{pq}\eta\omega_5^2\big( 5\omega_2^2 +\frac{2\gamma_2^2}{\mu(\gamma_1-L_\rho)}\big) + 432\tau^2\gamma_1^2\mathcal{M} - \frac{1}{4\alpha_{pq}\eta} \leq -\frac{1}{8\alpha_{pq}\eta} \ . \notag 
\end{align}
Specifically, we enforce
\begin{align}
   &  1 + \omega^2_1 \leq \frac{1}{24\alpha_{pq}\eta}  \ , \quad 8\alpha_{pq}\eta\omega_5^2\big( 5\omega_2^2 +\frac{2\gamma_2^2}{\mu(\gamma_1-L_\rho)}\big) \leq \frac{1}{24\alpha_{pq}\eta}  \ , \quad  432\tau^2\gamma\mathcal{M} \leq \frac{1}{24\alpha_{pq}\eta}  \ .  \notag 
\end{align}
By solving the first two inequalities, we derive
\begin{align}
    & \alpha_{x} \leq  \min\{ \frac{1}{24c_{\alpha_{pq}}(1 + \omega^2_1)} \ , \frac{1}{8\sqrt{3}c_{\alpha_{pq}}\omega_5\sqrt{5\omega_2^2+\frac{2\gamma_2^2}{\mu(\gamma_1-L_\rho)}}} \} \ . \notag 
\end{align}
To solve the last inequality, from Eq.~(\ref{eq:eta_alpha}), we derive:
\begin{align}
    (\alpha_x^2+\alpha_y^2+\alpha_z^2)\eta^2 \leq \frac{300}{360^2\tau^2L_{\rho,\gamma}^2} \ , \notag 
\end{align}
and from the definition of $\mathcal{M}$ in Eq.~(\ref{eq:def_m}), we enforce: 
\begin{align}\label{eq:M_alpha_rho}
    & 432\tau^2\gamma_1\frac{6\times300(2L_{f}^{2}\rho^2c_1 +2L_{g}^{2}c_2+ 2L_{g}^{2}c_3) }{360^2\tau^2L_{\rho,\gamma}^2K} \leq \frac{1}{24\times6\alpha_{pq}\eta} \ , \notag \\
    & 432\tau^2\gamma_1\frac{6\times300(2L_{f}^{2}\rho^2c_4 + 2L_{g}^{2}c_5 )}{360^2\tau^2L_{\rho,\gamma}^2K}  \leq \frac{1}{24\times6\alpha_{pq}\eta} \ , \notag \\
    & 432\tau^2\gamma_1 \frac{6\times300\times2L_{g}^{2}c_6}{360^2\tau^2L_{\rho,\gamma}^2K}\leq \frac{1}{24\times6\alpha_{pq}\eta} \ , \notag \\
    & 432\tau^2\gamma_1\frac{300\times 3\tau^2 \times 6\alpha_x\eta(\rho^2L_f^2+2L_g^2)}{360^2\tau^2L_{\rho,\gamma}^2}   \leq \frac{1}{24\times6\alpha_{pq}\eta} \ , \notag \\
    & 432\tau^2\gamma_1 \frac{300\times 3 \tau^2\times 2\alpha_{y}\eta(\rho^2L_f^2+L_g^2) }{360^2\tau^2L_{\rho,\gamma}^2}    \leq \frac{1}{24\times6\alpha_{pq}\eta} \ , \notag \\
    & 432\tau^2\gamma_1 \frac{300\times3 \tau^2\times3\alpha_z\eta L_g^2 }{360^2\tau^2L_{\rho,\gamma}^2}   \leq \frac{1}{24\times6\alpha_{pq}\eta} \ .  
\end{align}

To solve these inequalities, we apply the value of $\{c_i\}_{i=1}^6$ in Eq.~(\ref{eq:c_value}), with $\alpha_{y} = \alpha_{x}$, $\alpha_{z} = \alpha_{x}c_{\alpha_{z}}$, $\alpha_{pq} = \alpha_{x}c_{\alpha_{pq}} $, and obtain
\begin{align}
     \alpha_{x} \leq \min\{ \frac{\sqrt{\beta_{x}K}}{72\sqrt{2\gamma_1 c_{\alpha_{pq}}}} \ , \frac{\sqrt{\beta_{y}K}}{24\sqrt{6\gamma_1 c_{\alpha_{pq}}}} \ ,  \frac{\sqrt{\beta_{z}K}}{72\sqrt{\gamma_1 c_{\alpha_{z}}c_{\alpha_{pq}}}} \ , \frac{1}{36\tau\eta\sqrt{2\gamma_1 c_{\alpha_{pq}}}}  \ ,\frac{1}{36\tau\eta\sqrt{\gamma_1 c_{\alpha_{z}}c_{\alpha_{pq}}}} \} \ . \notag 
\end{align}

Similarly, to eliminate $\frac{1}{T}\sum_{t=0}^{T-1} \mathbb{E}[\|\bar{r}_{t+1}- \bar{r}_{t}\|^2]$, we enforce
\begin{align}
  16\gamma_1\alpha_{pq}\eta\big(5\omega_2^2\omega_3^2 + C_r\big)+ 432\tau^2\gamma_2^2\mathcal{M} - \frac{\gamma_2}{2\alpha_r\eta} \leq -\frac{\gamma_2}{4\alpha_{r}\eta} \ . \notag 
\end{align}
Specifically, from $C_r = \frac{4\gamma_2^2}{\mu(\gamma_1-L_\rho)}\big(\omega_3^2+\frac{(1-\alpha_r\eta)^2}{\alpha_r^2\eta^2}\big)$, we enforce
\begin{align}
   16\gamma_1\alpha_{pq}\eta\times5\omega_2^2\omega_3^2 \leq \frac{\gamma_2}{16\alpha_{r}\eta} \ , & \quad \frac{16\gamma_1\alpha_{pq}\eta \times 4\gamma_2^2}{\mu(\gamma_1-L_\rho)}\omega_3^2 \leq \frac{\gamma_2}{16\alpha_{r}\eta} \ , \notag \\
    \frac{16\gamma_1\alpha_{pq}\eta\times 4\gamma_2^2}{\mu(\gamma_1-L_\rho)}\frac{(1-\alpha_r\eta)^2}{\alpha_r^2\eta^2} \leq \frac{\gamma_2}{16\alpha_{r}\eta} \ , & \quad 432\tau^2\gamma_2^2\mathcal{M}  \leq \frac{\gamma_2}{16\alpha_{r}\eta}  \ .  \notag 
\end{align}
To solve the first and second inequality, we obtain
\begin{align}
    \alpha_x \leq \Big\{ \frac{(\gamma_1-L_{\rho})(\gamma_2-L_{\rho})}{16\sqrt{5\gamma_1\gamma_2c_{\alpha_{pq}}c_{\alpha_{r}}}L_g} \ , \frac{\sqrt{\mu(\gamma_1-L_\rho)}(\gamma_2-L_{\rho})}{32\gamma_2\sqrt{\gamma_1\gamma_2c_{\alpha_{pq}}c_{\alpha_{r}}}} \Big\} \ . \notag 
\end{align}
To solve the third inequality, from the third inequality in Eq.~(\ref{eq:nabla_xF_alpha_x}), we derive
\begin{align}
    \frac{16\gamma_1\alpha_{pq}\eta\times 4\gamma_2^2}{\mu(\gamma_1-L_\rho)}\frac{(1-\alpha_r\eta)^2}{\alpha_r^2\eta^2} \leq \frac{\alpha_x\eta(\gamma_1-L_{\rho})^2(\gamma_2-L_{\rho})^2}{72\times24L_g^2} \leq \frac{\gamma_2}{16\alpha_{r}\eta} \ , \notag
\end{align}
thus, we obtain $\alpha_x\leq \frac{6\sqrt{3\gamma_2}L_g}{\sqrt{c_{\alpha_{r}}}(\gamma_1-L_{\rho})(\gamma_2-L_{\rho})}$. 
To solve the last inequality, following Eq.~(\ref{eq:M_alpha_rho}), we obtain
\begin{align}
     \alpha_{x} \leq \min\Big\{ \frac{\sqrt{\beta_{x}K}}{48\sqrt{3\gamma_2 c_{\alpha_{r}}}} \ , \frac{\sqrt{\beta_{y}K}}{48\sqrt{\gamma_2 c_{\alpha_{r}}}} \ ,  \frac{\sqrt{\beta_{z}K}}{24\sqrt{6\gamma_2 c_{\alpha_{z}}c_{\alpha_{r}}}} \ , \frac{1}{24\tau\eta\sqrt{3\gamma_2 c_{\alpha_{r}}}}  \ ,\frac{1}{12\tau\eta\sqrt{6\gamma_2 c_{\alpha_{z}}c_{\alpha_{r}}}} \Big\} \ . \notag
\end{align}

Then, to eliminate $\frac{1}{T}\sum_{t=0}^{T-1}\mathbb{E}[\| \bar{u}_{t}\|^2]$, we enforce
\begin{align}
    4 \alpha_{x}^2\alpha_z\eta^3 L_{g}^2 +  6\alpha_x^2\eta^2\mathcal{M}_1 + 288\tau^2 \alpha_x^2\eta^2L_{\rho,\gamma}^2\mathcal{M} -  \frac{\alpha_x \eta}{4} \leq -\frac{\alpha_x \eta}{8} \ . \notag 
\end{align}
In particular, we enforce
\begin{align}\label{eq:enforce_u}
     4\alpha_{x}^2\alpha_z\eta^3 L_{g}^2 \leq \frac{\alpha_x \eta}{48} \ ,  \quad 6\alpha_x^2\eta^2\mathcal{M}_1  \leq \frac{\alpha_x \eta}{24} \ ,  \quad 288\tau^2 \alpha_x^2\eta^2L_{\rho,\gamma}^2\mathcal{M} \leq \frac{\alpha_x \eta}{24} \ .  
\end{align}
By solving the first inequality, we derive $ \alpha_{x} \leq \frac{1}{8\sqrt{3c_{\alpha_z}}L_{g}}$.
To solve the second inequality, from the definition of $\mathcal{M}_1$ in Eq. (\ref{eq:def_mc}), we enforce:
\begin{align}
    & \frac{6\alpha_x^2\eta^2}{K} (2L_{f}^{2}\rho^2c_1 +2L_{g}^{2}c_2+ 2L_{g}^{2}c_3) \leq \frac{\alpha_x \eta}{72} \ , \notag \\
    & \frac{6\alpha_x^2\eta^2}{K} (2L_{f}^{2}\rho^2c_4 + 2L_{g}^{2}c_5)  \leq \frac{\alpha_x \eta}{72} \ , \qquad \frac{6\alpha_x^2\eta^2}{K} 2L_{g}^{2}c_6 \leq \frac{\alpha_x \eta}{72} \ . \notag
\end{align}

To solve these inequalities, we apply the value of $\{c_i\}_{i=1}^6$ in Eq.~(\ref{eq:c_value}), and obtain
\begin{align}
    \alpha_{x} \leq \min\{ \frac{\sqrt{\beta_{x}K}}{72L_{\rho}} \ , \frac{\sqrt{\beta_{y}K}}{24\sqrt{3}L_{\rho}} \ , \frac{\sqrt{\beta_{z}K}}{36\sqrt{2c_{\alpha_{z}}}L_g} \} \ .  \notag  
\end{align}
Similarly, to solve the third inequality, following Eq.~(\ref{eq:M_alpha_rho}), we obtain
\begin{align}
    \alpha_{x} \leq \{  \frac{\sqrt{\beta_{x}K}}{48\sqrt{3}L_{\rho}} \ , \frac{\sqrt{\beta_{y}K}}{48L_{\rho}} \ , \frac{\sqrt{\beta_{z}K}}{24\sqrt{6c_{\alpha_{z}}}L_g}  \ ,  \frac{1}{24\sqrt{3}\tau\eta L_{\rho}} \ ,\frac{1}{12\sqrt{6c_{\alpha_{z}}}\tau\eta L_g} \} \ .  \notag 
\end{align}
In addition, to eliminate $\frac{1}{T}\sum_{t=0}^{T-1}\mathbb{E}[\| \bar{v}_{t}\|^2]$, we enforce
\begin{align}
    \frac{3\alpha_x\alpha_{y}^2 \eta^3L_{\rho}^2}{2} + 4\alpha_y^2\eta^2\mathcal{M}_2 + 192\tau^2\alpha_y^2\eta^2( L_{f}^2\rho^2 + 3L_g^2+3\gamma^2) \mathcal{M} -  \frac{\alpha_{y}\eta}{4} \leq  - \frac{\alpha_{y}\eta}{8}  \ .  \notag 
\end{align}
Specifically, we enforce
\begin{align}
   & \frac{3\alpha_x\alpha_{y}^2 \eta^3L_{\rho}^2}{2} \leq \frac{\alpha_{y}\eta}{48} \ , \quad   4\alpha_y^2\eta^2\mathcal{M}_2 \leq \frac{\alpha_{y}\eta}{24}\ , \notag \\
   & 192\tau^2\alpha_y^2\eta^2( L_{f}^2\rho^2 + 3L_g^2+3\gamma^2) \mathcal{M} \leq \frac{\alpha_{y}\eta}{24} \ .  \notag 
\end{align}
From the definitions of $\mathcal{M}_2$ and $\mathcal{M}$, together with the condition $\alpha_y = \alpha_x$, and noting that Eq. (\ref{eq:enforce_u}) holds, the last two inequalities above follow directly.

By solving the first inequality above, we derive $\alpha_{x} \leq \frac{1}{6\sqrt{2}L_{\rho}}$.

At last, to eliminate $\frac{1}{T}\sum_{t=0}^{T-1}\mathbb{E}[\| \bar{w}_{t}\|^2]$, we enforce
\begin{align}
    \frac{ \alpha^2_z\eta^2 }{2}(L_{F,2}+2L_d+ 2C_\rho) +  2\alpha_z^2\eta^2\mathcal{M}_3  +  48\tau^2 \alpha_z^2\eta^2L_{g}^2 \mathcal{M} - \frac{\alpha_z\eta }{16} \leq - \frac{\alpha_z\eta }{32} \ .  \notag 
\end{align}
In particular, we enforce
\begin{align}
    \frac{ \alpha^2_z\eta^2 }{2} (L_F + 2L_d ) \leq \frac{\alpha_z\eta }{32\times6} \ , & \quad  \alpha_z^2\eta^2C_\rho  \leq  \frac{\alpha_z\eta }{32\times6} \ , \notag \\
    \quad 2\alpha_z^2\eta^2\mathcal{M}_3 \leq  \frac{\alpha_z\eta }{96} \ , & \quad  48\tau^2 \alpha_z^2\eta^2L_{g}^2 \mathcal{M}  \leq  \frac{\alpha_z\eta }{96}  \ .  \notag 
\end{align}
By solving the first inequality above, we obtain $ \alpha_{x} \leq \frac{1}{24c_{\alpha_{z}}(L_F + 2L_d)}$.

To solve the second inequality, from Eq.~(\ref{eq:nabla_xF_alpha_x}), we obtain $ \alpha_x \leq \frac{1}{2\sqrt{3}c_{\alpha_{z}}(\gamma_2-L_\rho)}$.

To solve the third inequality, from the definition of $\mathcal{M}_3$ in Eq. (\ref{eq:def_mc}), we enforce:
\begin{align}
    \frac{2\alpha_z^2\eta^2}{K}2L_{g}^{2}c_3 \leq  \frac{\alpha_z\eta }{192} \ , \quad \frac{2\alpha_z^2\eta^2}{K}2L_{g}^{2}c_6 \leq  \frac{\alpha_z\eta }{192} \ ,  \notag 
\end{align}
and we obtain
\begin{align}
    \alpha_{x} \leq \min\{ \frac{\sqrt{\beta_{x}K}}{48\sqrt{2c_{\alpha_{z}}}L_g} \ , \frac{\sqrt{\beta_{z}K}}{24\sqrt{2}c_{\alpha_{z}}L_g} \} \ .  \notag 
\end{align}
Similarly, to solve the last inequality, following Eq.~(\ref{eq:M_alpha_rho}), we obtain
\begin{align}
    \alpha_{x} \leq \min \{\frac{\sqrt{\beta_{x}K}}{48\sqrt{2c_{\alpha_{z}}}L_g} \ ,  \frac{\sqrt{\beta_{y}K}}{16\sqrt{6c_{\alpha_{z}}}L_g} \ ,   \frac{\sqrt{\beta_{z}K}}{48c_{\alpha_{z}}L_g} \ ,  \frac{1}{24\sqrt{2c_{\alpha_{z}}}\tau\eta L_g} \ , \frac{1}{24c_{\alpha_{z}}\tau\eta L_g}  \} \ . \notag 
\end{align}

In summary, by setting
\begin{align}\label{eq:upper-bound-beta-x}
    & \alpha_{x} \leq \min \Big\{\frac{9}{80\gamma_1} \ , \frac{\sqrt{\mu(\gamma_1-L_{\rho})^3}(\gamma_2-L_{\rho})^2}{24\times16\sqrt{6\gamma_1c_{\alpha_{pq}}}L_gc_{\alpha_r}\gamma_2^2}  \ , \frac{1}{24c_{\alpha_{pq}}(1 + \omega^2_1)} \ , \frac{1}{8\sqrt{3}c_{\alpha_{pq}}\omega_5\sqrt{5\omega_2^2+\frac{2\gamma_2^2}{\mu(\gamma_1-L_\rho)}}} \notag \\
    & \qquad \frac{\sqrt{\beta_{x}K}}{72\sqrt{2\gamma_1 c_{\alpha_{pq}}}} \ , \frac{\sqrt{\beta_{y}K}}{24\sqrt{6\gamma_1 c_{\alpha_{pq}}}} \ ,  \frac{\sqrt{\beta_{z}K}}{72\sqrt{\gamma_1 c_{\alpha_{z}}c_{\alpha_{pq}}}} \ , \frac{1}{36\tau\eta\sqrt{2\gamma_1 c_{\alpha_{pq}}}}  \ ,\frac{1}{36\tau\eta\sqrt{\gamma_1 c_{\alpha_{z}}c_{\alpha_{pq}}}} \ , \notag \\
    & \qquad \frac{(\gamma_1-L_{\rho})(\gamma_2-L_{\rho})}{16\sqrt{5\gamma_1\gamma_2c_{\alpha_{pq}}c_{\alpha_{r}}}L_g} \ , \frac{\sqrt{\mu(\gamma_1-L_\rho)}(\gamma_2-L_{\rho})}{32\gamma_2\sqrt{\gamma_1\gamma_2c_{\alpha_{pq}}c_{\alpha_{r}}}} \ ,  \frac{6\sqrt{3\gamma_2}L_g}{\sqrt{c_{\alpha_{r}}}(\gamma_1-L_{\rho})(\gamma_2-L_{\rho})} \ , \notag \\
    & \qquad \frac{\sqrt{\beta_{x}K}}{48\sqrt{3\gamma_2 c_{\alpha_{r}}}} \ , \frac{\sqrt{\beta_{y}K}}{48\sqrt{\gamma_2 c_{\alpha_{r}}}} \ ,  \frac{\sqrt{\beta_{z}K}}{24\sqrt{6\gamma_2 c_{\alpha_{z}}c_{\alpha_{r}}}} \ , \frac{1}{24\tau\eta\sqrt{3\gamma_2 c_{\alpha_{r}}}}  \ ,\frac{1}{12\tau\eta\sqrt{6\gamma_2 c_{\alpha_{z}}c_{\alpha_{r}}}} \ , \notag \\
    &  \qquad \frac{1}{8\sqrt{3c_{\alpha_z}}L_{g}} \ , \frac{\sqrt{\beta_{x}K}}{48\sqrt{3}L_{\rho}} \ , \frac{\sqrt{\beta_{y}K}}{48L_{\rho}} \ , \frac{\sqrt{\beta_{z}K}}{24\sqrt{6c_{\alpha_{z}}}L_g} \ ,  \frac{1}{24\sqrt{3}\tau\eta L_{\rho}} \ ,\frac{1}{12\sqrt{6c_{\alpha_{z}}}\tau\eta L_g} \ , \notag \\
    & \qquad \frac{1}{6\sqrt{2}L_{\rho}} \ , \frac{1}{24c_{\alpha_{z}}(L_F + 2L_d)} \ , \frac{1}{2\sqrt{3}c_{\alpha_{z}}(\gamma_2-L_\rho)} \ , \frac{\sqrt{\beta_{x}K}}{48\sqrt{2c_{\alpha_{z}}}L_g} \ ,  \frac{\sqrt{\beta_{y}K}}{16\sqrt{6c_{\alpha_{z}}}L_g} \ ,   \frac{\sqrt{\beta_{z}K}}{48c_{\alpha_{z}}L_g} \ ,  \notag \\
    & \qquad \frac{1}{24\sqrt{2c_{\alpha_{z}}}\tau\eta L_g} \ , \frac{1}{24c_{\alpha_{z}}\tau\eta L_g} \Big\} \ ,  \notag \\
    & \alpha_y = \alpha_x \ , \quad  \alpha_z = \alpha_x\underbrace{\frac{(\gamma_1 - L_{\rho})^2}{72L^2_{g}}}_{c_{\alpha_z}=O(1)} \ , \quad \alpha_{pq} = \alpha_{x}\underbrace{\frac{(\gamma_1-L_{\rho})^4(\gamma_2-L_\rho)^2\mu}{72\times768\gamma_1L_{g}^2(5L_g^2\mu+16\gamma_2^2(\gamma_1-L_\rho))}}_{c_{\alpha_{pq}} = O(1/\kappa)} \ , \notag \\
    & \alpha_r = \alpha_{x}\underbrace{\frac{(\gamma_1-L_{\rho})^4(\gamma_2-L_\rho)^4}{72\times24\times640\gamma_1\gamma_2^2 L_g^4}}_{{c_{\alpha_r}=O(1)}} \ , 
\end{align}
we obtain
\begin{align}
    & \frac{\mathcal{P}_{T} -  \mathcal{P}_{0}}{ T} \leq - \frac{\alpha_x \eta}{8} \frac{1}{T}\sum_{t=0}^{T-1}  \mathbb{E}[\| \nabla_{x} F(\bar{x}_{t}, \bar{y}_{t}, \bar{z}_{t}; \bar{p}_{t}, \bar{q}_{t}, \bar{r}_{t}) \|^2]  \notag \\
    & - \frac{\alpha_{y}\eta}{2} \frac{1}{T}\sum_{t=0}^{T-1} \mathbb{E}[\|\nabla_{y}  F(\bar{x}_{t}, \bar{y}_{t}, \bar{z}_{t}; \bar{p}_{t}, \bar{q}_{t}, \bar{r}_{t}) \|^2] - \frac{\alpha_z\eta}{2} \frac{1}{T}\sum_{t=0}^{T-1} \mathbb{E}[\|\nabla_{z} F(\bar{x}_{t}, \bar{y}_{t}, \bar{z}_{t}; \bar{p}_{t}, \bar{q}_{t}, \bar{r}_{t})\|^2 ]  \notag \\
    & - \frac{\gamma_1}{8\alpha_{pq}\eta} \frac{1}{T}\sum_{t=0}^{T-1} \big( \mathbb{E}[\|\bar{p}_{t+1}- \bar{p}_{t}\|^2] + \mathbb{E}[\|\bar{q}_{t+1}-\bar{q}_{t}\|^2]\big) - \frac{\gamma_2}{4\alpha_{r}\eta}\mathbb{E}[\|\bar{r}_{t+1}-\bar{r}_{t}\|^2] \notag \\
    &  + \mathcal{M}_{0} + 3\tau S\mathcal{M} \ . \notag 
\end{align}
From 
\begin{align}
    & \rho^2\mathbb{E}[\| \nabla_{x} f_{\rho}(\bar{x}_{t}, \bar{y}_{t}, \bar{z}_{t}) \|^2] = 2\mathbb{E}[\| \nabla_{x} F(\bar{x}_{t}, \bar{y}_{t}, \bar{z}_{t}; \bar{p}_{t}, \bar{q}_{t}, \bar{r}_{t}) \|^2] + 2 \gamma_1^2\mathbb{E}[\|\bar{x}_{t}- \bar{p}_{t}\|^2]\ , \notag \\
    & \rho^2\mathbb{E}[\| \nabla_{y} f_{\rho}(\bar{x}_{t}, \bar{y}_{t}, \bar{z}_{t}) \|^2] \leq 2\mathbb{E}[\| \nabla_{y} F(\bar{x}_{t}, \bar{y}_{t}, \bar{z}_{t}; \bar{p}_{t}, \bar{q}_{t}, \bar{r}_{t}) \|^2] + 2\gamma_1^2 \mathbb{E}[\|\bar{y}_{t}- \bar{q}_{t}\|^2] \ , \notag \\
    & \rho^2\mathbb{E}[\| \nabla_{z} f_{\rho}(\bar{x}_{t}, \bar{y}_{t}, \bar{z}_{t}) \|^2] \leq 2\mathbb{E}[\| \nabla_{z} F(\bar{x}_{t}, \bar{y}_{t}, \bar{z}_{t}; \bar{p}_{t}, \bar{q}_{t}, \bar{r}_{t}) \|^2] + 2\gamma_2^2\mathbb{E}[\|\bar{z}_{t}- \bar{r}_{t}\|^2]\ , \notag 
\end{align}
we obtain
\begin{align}\label{eq:final_bound}
    &  \quad  \frac{1}{T}\sum_{t=0}^{T-1} \Bigg(  \rho^2\mathbb{E}[\| \nabla_{x} f_{\rho}(\bar{x}_{t}, \bar{y}_{t}, \bar{z}_{t}) \|^2] +  \rho^2\mathbb{E}[\|\nabla_{y}   f_{\rho}(\bar{x}_{t}, \bar{y}_{t}, \bar{z}_{t}) \|^2] +  \rho^2\kappa\mathbb{E}[\|\nabla_{z}  f_{\rho}(\bar{x}_{t}, \bar{y}_{t}, \bar{z}_{t}) \|^2 ] \Bigg) \notag \\
    & \leq  \frac{1}{T}\sum_{t=0}^{T-1} \Bigg(  2\mathbb{E}[\| \nabla_{x} F(\bar{x}_{t}, \bar{y}_{t}, \bar{z}_{t}; \bar{p}_{t}, \bar{q}_{t}, \bar{r}_{t}) \|^2] +  2\mathbb{E}[\|\nabla_{y}  F(\bar{x}_{t}, \bar{y}_{t}, \bar{z}_{t}; \bar{p}_{t}, \bar{q}_{t}, \bar{r}_{t}) \|^2]   \notag \\
    & \qquad \qquad \quad  +  2\kappa \mathbb{E}[\|\nabla_{z} F(\bar{x}_{t}, \bar{y}_{t}, \bar{z}_{t}; \bar{p}_{t}, \bar{q}_{t}, \bar{r}_{t})\|^2 ]   + 2\gamma_1^2\mathbb{E}[\|\bar{x}_{t}- \bar{p}_{t}\|^2] + 2\gamma_1^2 \mathbb{E}[\|\bar{y}_{t}- \bar{q}_{t}\|^2] \notag \\
    & \qquad \qquad \quad  + 2\gamma_2^2\kappa \mathbb{E}[\|\bar{z}_{t}- \bar{r}_{t}\|^2] \Bigg) \notag \\
	& \leq \max\{\frac{16}{\alpha_{x}\eta}  \ , \frac{16\kappa}{\alpha_{x}\eta c_{\alpha_{z}}} \ , \frac{16\gamma_1}{\alpha_{x}\eta c_{\alpha_{pq}}} \ , \frac{16\kappa\gamma_2}{\alpha_{r}\eta c_{\alpha_{r}}} \} \Big( \frac{\mathcal{P}_{0} -  \mathcal{P}_{T}}{ T} + \mathcal{M}_{0} + 3\tau S\mathcal{M} \Big) \ .
\end{align}
When $t=0$, we have
\begin{align}
    & \mathcal{P}_{0} = \Psi_{0}+ \frac{6\alpha_{x}}{\beta_{x}\eta}\rho^2\mathcal{G}_{u, 1, 0}  +  \frac{6\alpha_{x}}{\beta_{x}\eta}\mathcal{G}_{u, 2, 0}    + \frac{6\alpha_{x}}{\beta_{x}\eta}\mathcal{G}_{u, 3, 0}+ \frac{2\alpha_{y}}{\beta_{y}\eta}\rho^2\mathcal{G}_{v, 1, 0}  + \frac{2\alpha_{y}}{\beta_{y}\eta} \mathcal{G}_{v, 2, 0}  + \frac{6\alpha_{z}}{\beta_{z}\eta} \mathcal{G}_{w, 0}  \notag  \\
    & \leq \Psi_{0} + \frac{6\alpha_{x}}{\beta_{x}\eta}\frac{\sigma^2}{BK}(\rho^2+2) + \frac{2\alpha_{x}}{\beta_{y}\eta}\frac{\sigma^2}{BK}(\rho^2+1) + \frac{3\alpha_{x}c_{\alpha_{z}}}{\beta_{z}\eta}  \frac{\sigma^2}{BK} \ . 
\end{align}
Therefore, taking $O$ notation, the first term in Eq.~(\ref{eq:final_bound}) can be shown as 
\begin{align}
    & \quad \max\{\frac{16}{\alpha_{x}\eta} \ , \frac{16\kappa}{\alpha_{x}\eta c_{\alpha_{z}}} \ ,  \frac{16\gamma_1}{\alpha_{x}\eta c_{\alpha_{pq}}} \ , \frac{16\kappa\gamma_2}{\alpha_{r}\eta c_{\alpha_{r}}} \} \frac{\mathcal{P}_{0} -  \mathcal{P}_{T}}{ T}   \leq   O\Big(\frac{\Psi_{0} \kappa}{\alpha_{x}\eta T}\Big) + O\Big( \frac{\kappa }{\beta \eta^2BKT }\Big) \ . \notag 
\end{align}
Similarly, the second term in Eq.~(\ref{eq:final_bound}) can be shown as
\begin{align}
    & \mathcal{M}_0 = \frac{2\sigma^2}{K}\Big(\beta_x^2\eta^4\frac{6\alpha_{x}}{\beta_{x}\eta}(\rho^2 + 2)  + \beta_y^2\eta^4(\rho^2+1) \frac{2\alpha_{y}}{\beta_{y}\eta}  + \beta_z^2\eta^4\frac{3\alpha_{z}}{\beta_{z}\eta} \Big) \ , \notag \\
    & \max\{\frac{16}{\alpha_{x}\eta} \ , \frac{16\kappa}{\alpha_{x}\eta c_{\alpha_{z}}} \ ,  \frac{16\gamma_1}{\alpha_{x}\eta c_{\alpha_{pq}}} \ , \frac{16\kappa\gamma_2}{\alpha_{r}\eta c_{\alpha_{r}}}\} \mathcal{M}_{0} \leq O\Big( \frac{\kappa\beta\eta^2}{K}\Big) \ . 
\end{align}
At last, the third term in Eq.~(\ref{eq:final_bound}) can be shown as
\begin{align}
     & S = 12\tau (\beta_x^2+\beta_y^2)\eta^4 \delta^2\rho^2 + 8\tau (\beta_x^2+\beta_y^2)\eta^4  \sigma^{2}\rho^2 \notag + 12\tau(4\beta_x^2 + 3\beta_y^2+\beta_z^2)\eta^4\delta^2 \notag \\
     & \quad \quad +8\tau (4\beta_x^2 + 3\beta_y^2+\beta_z^2)\eta^4\sigma^{2}  = O\Big(\tau\beta^2\eta^4\Big) \notag \\
     & \mathcal{M} = \frac{6(\alpha_x^2+\alpha_y^2+\alpha_z^2)\eta^2}{K} (2L_{f}^{2}\rho^2c_1 +2L_{g}^{2}c_2+ 2L_{g}^{2}c_3  +  2L_{f}^{2}\rho^2c_4 + 2L_{g}^{2}c_5 + 2L_{g}^{2}c_6) \notag \\
     & \quad  \qquad + 3 \tau^2(\alpha_x^2+\alpha_y^2+\alpha_z^2)\eta^2 \Big(6\alpha_x\eta(\rho^2L_f^2+2L_g^2)+  2\alpha_{y}\eta(\rho^2L_f^2+L_g^2) + 3\alpha_z\eta L_g^2\Big) \notag \\
     & = O\Big(\frac{\alpha_{x}^3\eta}{\beta K}(\rho^2+1)\Big) + O\Big(\tau^2\alpha_{x}^3\eta^3(\rho^2+1)\Big) \ . 
\end{align}
Therefore, we obtain
\begin{align}
    & \max\{\frac{16}{\alpha_{x}\eta} \ , \frac{16\kappa}{\alpha_{x}\eta c_{\alpha_{z}}} \ , \frac{16\gamma_1}{\alpha_{x}\eta c_{\alpha_{pq}}} \ , \frac{16\kappa\gamma_2}{\alpha_{r}\eta c_{\alpha_{r}}} \} 3\tau S\mathcal{M} \leq O\Big(\frac{\kappa\tau^2\beta\eta^4\alpha_{x}^2 }{K}\Big) + O\Big(\kappa\tau^4\alpha_{x}^2\beta^2\eta^6\Big) \ . 
\end{align}

In summary, we establish the convergence rate
\begin{align}
    &  \quad  \frac{1}{T}\sum_{t=0}^{T-1} \Bigg(  \mathbb{E}[\| \nabla_{x} f_{\rho}(\bar{x}_{t}, \bar{y}_{t}, \bar{z}_{t}) \|^2] + \mathbb{E}[\|\nabla_{y}   f_{\rho}(\bar{x}_{t}, \bar{y}_{t}, \bar{z}_{t}) \|^2] +  \kappa\mathbb{E}[\|\nabla_{z}  f_{\rho}(\bar{x}_{t}, \bar{y}_{t}, \bar{z}_{t}) \|^2 ] \Bigg) \notag \\
    & \leq \frac{1}{\rho^2}\max\{\frac{16}{\alpha_{x}\eta} \ , \frac{16\kappa}{\alpha_{x}\eta c_{\alpha_{z}}} \ , \frac{16\gamma_1}{\alpha_{x}\eta c_{\alpha_{pq}}} \ , \frac{16\kappa\gamma_2}{\alpha_{r}\eta c_{\alpha_{r}}} \} \Big( \frac{\mathcal{P}_{0} -  \mathcal{P}_{T}}{ T} + \mathcal{M}_{0} + 3\tau S\mathcal{M} \Big)  \\
    & \leq O\Big(\frac{  \kappa (\Psi_{0} - \Psi_{T})}{\alpha_{x}\eta \rho^2 T}\Big) + O\Big( \frac{\kappa}{\beta \eta^2\rho^2 BKT }\Big) + O\Big( \frac{\kappa\beta\eta^2}{\rho^2K}\Big) + O\Big(\frac{\kappa\tau^2\beta\eta^4\alpha_{x}^2 }{\rho^2K}\Big) + O\Big(\frac{\kappa\tau^4\alpha_{x}^2\beta^2\eta^6}{\rho^2}\Big)  \ . \notag
\end{align}

Since $F(x, y, z ; p, q, z)  =\rho f (x, y) + (g(x, y)-g(x, z))  + \frac{\gamma_1}{2}\| x - p \|^2 + \frac{\gamma_1}{2}\| y - q\|^2 -  \frac{\gamma_2}{2}\| z - r\|^2$, we need to scale $F(x, y, z ; p, q, r)$ by $\frac{1}{\rho}$ to study the original penalized function $ f (x, y) + \frac{1}{\rho}(g(x, y)-g(x, z))$. Then according to \cite{zhang2025federated}, by initializing $y_0$ and $z_0$ such that $\mathbb{E}[\|{y}_{0} - y_{1/\rho}^{*}({{x}}_{0})\|^2]\leq O(\epsilon^2)$  and $\mathbb{E}[\|{z}_{0} - y^{*}({{x}}_{0})\|^2]\leq O(\epsilon^2)$ where $y_{1/\rho}^{*}({{x}})=\arg\min_{y} f (x, y) + \frac{1}{\rho}g(x, y)$, we can mitigate the influence of $\frac{1}{\rho}$ in the initial upper bound. As a result, we can obtain 
\begin{align}
    &  \quad  \frac{1}{T}\sum_{t=0}^{T-1} \Bigg(  \mathbb{E}[\| \nabla_{x} f_{\rho}(\bar{x}_{t}, \bar{y}_{t}, \bar{z}_{t}) \|^2] + \mathbb{E}[\|\nabla_{y}   f_{\rho}(\bar{x}_{t}, \bar{y}_{t}, \bar{z}_{t}) \|^2] +  \kappa\mathbb{E}[\|\nabla_{z}  f_{\rho}(\bar{x}_{t}, \bar{y}_{t}, \bar{z}_{t}) \|^2 ] \Bigg) \notag \\
    & \leq O\Big(\frac{  \kappa }{\alpha_{x}\eta \rho T}\Big) + O\Big( \frac{\kappa}{\beta \eta^2\rho^2 BKT }\Big) + O\Big( \frac{\kappa\beta\eta^2}{\rho^2K}\Big) + O\Big(\frac{\kappa\tau^2\beta\eta^4\alpha_{x}^2 }{\rho^2K}\Big) + O\Big(\frac{\kappa\tau^4\alpha_{x}^2\beta^2\eta^6}{\rho^2}\Big)  \ .
\end{align}

By setting
\begin{align}
    & \rho = O\left(\frac{\epsilon}{\kappa^3}\right) \ , \quad  \alpha_x = O\left(1\right)  \ , \notag \\
    & \beta_x = O\left(\frac{1}{K}\right) \ , \quad \beta_y = O\left(\frac{1}{K}\right) \ , \quad  \beta_z = O\left(\frac{1}{K}\right) \ , \notag \\
    & \eta = O\left(\frac{K\epsilon^2}{\kappa^{7/2}}\right)\ ,  \tau =  O\left(\frac{\kappa^{7/2}}{K\epsilon^2}\right)  \ , \quad B = O\left(\frac{\kappa^{13/2}}{K\epsilon^3}\right) \ , \quad T = O\left(\frac{\kappa^{15/2}}{K\epsilon^5}\right) \ .  
\end{align}
we obtain
\begin{align}
    &   \frac{1}{T}\sum_{t=0}^{T-1} \Bigg(  \mathbb{E}[\| \nabla_{x} f_{\rho}(\bar{x}_{t}, \bar{y}_{t}, \bar{z}_{t}) \|^2] +  \mathbb{E}[\|\nabla_{y}   f_{\rho}(\bar{x}_{t}, \bar{y}_{t}, \bar{z}_{t}) \|^2] +  \kappa\mathbb{E}[\|\nabla_{z}  f_{\rho}(\bar{x}_{t}, \bar{y}_{t}, \bar{z}_{t}) \|^2 ] \Bigg) \leq O(\epsilon^2) \ . \notag 
\end{align}

\section{Translation between Stationarity Measures}\label{app:translation}

\begin{algorithm}[ht]
	\caption{LocalSGDA~\citep{deng2021local}}
	\label{alg:localsgda-pl}
	\begin{algorithmic}[1]
		\REQUIRE $\eta_{\alpha}>0$, $\eta_{\theta}>0$. 
		Initialization on  worker $k$: 
		$\alpha_{0}^{(k)}=\alpha_{0}$, $\theta_{0}^{(k)}=\theta_{0}$, \\
		\vspace{2mm}
		\FOR{$t=0,\cdots, T-1$,  worker $k$} 
		
		\STATE Update $\alpha$: 
		$\alpha^{(k)}_{t+1} = \alpha^{(k)}_{t} - \eta_{\alpha} \nabla_{1} \tilde{f}_{\rho}^{(k)}(\alpha^{(k)}_{t}, \theta^{(k)}_{t}; \xi^{(k)}_{t}) $ \ , \\
		
		\STATE Update $\theta$: 
		$\theta^{(k)}_{t+1} = \theta^{(k)}_{t} + \eta_{\theta}\nabla_{2} \tilde{f}_{\rho}^{(k)}(\alpha^{(k)}_{t}, \theta^{(k)}_{t}; \xi^{(k)}_{t}) $ \ , \\

		\IF {$\text{mod}(t+1, \tau)==0$}
		\STATE 
		$\alpha^{(k)}_{t+1}=\frac{1}{K}\sum_{k'=1}^{K}\alpha^{(k')}_{ t+1}$\ , \ 
        $\theta^{(k)}_{t+1}=\frac{1}{K}\sum_{k'=1}^{K}\theta^{(k')}_{ t+1}$,  \\
		\ENDIF
		\ENDFOR
	\end{algorithmic}
\end{algorithm}

In this section, we present a detailed treatment of the translation between the stationarity measures defined in Definition~\ref{def:measures}, and complete the proof of Remark~\ref{remark:tran}. To establish this connection in the federated setting, we employ the LocalSGDA method (Algorithm~\ref{alg:localsgda-pl}) on an auxiliary two-sided PL minimax problem given by
\begin{align} \label{eq:loss-translation}
    & \min_{\alpha \triangleq (x, y)}\max_{\theta \triangleq z} \tilde{f}_{\rho}(\alpha, \theta)\triangleq   \frac{1}{K}\sum_{k=1}^{K}\underbrace{f^{(k)}_{\rho}(\alpha, \theta) + L \|\alpha -\tilde{\alpha}\|^2}_{\tilde{f}_{\rho}^{(k)} (\alpha, \theta)} \ , 
\end{align}
where $\tilde{\alpha}$ is the output of by Algorithm~\ref{alg:fedss-bvrfo}.  The resulting objective $\tilde{f}_{\rho}(\alpha, \theta)$ possesses several useful properties: it is $L$-strongly convex in $\alpha$, satisfies the $\mu$-PL condition with respect to $\theta$, and is $\tilde{L}$-smooth with $\tilde{L} = 3L$. We next characterize the convergence behavior of Algorithm~\ref{alg:localsgda-pl} under the strongly-convex--PL assumption and show that its convergence rate is dominated by that of Algorithm~\ref{alg:fedss-bvrfo}; therefore, the overall convergence speed is also dominated by Algorithm~\ref{alg:fedss-bvrfo}.


\begin{theorem}~\label{theorem:two-sided-pl}
 Given the smoothness and strongly-convex-PL assumptions on $\tilde{f}_{\rho}(\alpha, \theta)$, choosing step sizes $\eta_{\alpha} = O(\epsilon^2/\kappa^5)$ and $\eta_{\theta}=128\kappa^2\eta_{\alpha}$, then Algorithm~\ref{alg:localsgda-pl} has a convergence rate of $T=O(\epsilon^{-2}\kappa^6\log(\kappa))$ to achieve the $\epsilon$-accuracy solution: $\mathbb{E}[\|\alpha_{T} -\alpha^*\|^2] \leq  O(\epsilon^2)$, where $0<\epsilon<1$.
 
\end{theorem}

\begin{proof}
    To prove Theorem~\ref{theorem:two-sided-pl}, we introduce the following auxiliary functions:
    \begin{align}
    	& \tilde{\Phi}(\alpha) = \max_{\alpha} \tilde{f}_{\rho}(\alpha, \theta) = \max_{\theta} f_{\rho}(\alpha, \theta) + L \| \alpha -\tilde{\alpha}\|^2 \ , \notag \\
    	& \tilde{\Psi}(\theta) =  \min_{\alpha} \tilde{f}_{\rho}(\alpha, \theta) = \min_{\alpha} f_{\rho}(\alpha, \theta) + L \| \alpha -\tilde{\alpha}\|^2 \ , \quad  \tilde{\Phi}^* = \max_{\theta} \tilde{\Psi}(\theta) \ . \notag 
    \end{align}
    In addition, we denote $\bar{\alpha}_{t} =  \frac{1}{K}\sum_{k=1}^{K} {\alpha}^{(k)}_{t} , \ \bar{\theta}_{t} =  \frac{1}{K}\sum_{k=1}^{K} {\theta}^{(k)}_{t}$, and:
    \begin{align}
        \hat{m}^{(k)}_{\alpha, t} = \nabla_{1} \tilde{f}_{\rho}^{(k)}(\alpha^{(k)}_{t}, \theta^{(k)}_{t}; \xi^{(k)}_{t}) \ , \quad & \hat{m}^{(k)}_{\theta, t} = \nabla_{2} \tilde{f}_{\rho}^{(k)}(\alpha^{(k)}_{t}, \theta^{(k)}_{t}; \xi^{(k)}_{t})  \ , \notag \\
        {m}^{(k)}_{\alpha, t} : = \nabla_{1} \tilde{f}_{\rho}^{(k)}({\alpha}^{(k)}_{t}, {\theta}^{(k)}_{t})  \ , \quad & {m}^{(k)}_{\theta, t} : = \nabla_{2} \tilde{f}_{\rho}^{(k)}({\alpha}^{(k)}_{t}, {\theta}^{(k)}_{t})  \ . \notag 
    \end{align}

    Given that the $\tilde{L}$-smooth of $\tilde{f}_{\rho}(\alpha, \theta)$ with respect to ${\alpha}$, we derive
\begin{align} 
	& \quad  \mathbb{E}[\tilde{f}_{\rho}(\bar{\alpha}_{t}, \bar{\theta}_{t})] -  \mathbb{E}[\tilde{f}_{\rho}(\bar{\alpha}_{t+1}, \bar{\theta}_{t})]  \leq - \mathbb{E}[\langle \nabla_{1} \tilde{f}_{\rho}(\bar{\alpha}_{t}, \bar{\theta}_{t}), \bar{\alpha}_{t+1}- \bar{\alpha}_{t} \rangle]  + \frac{\tilde{L}}{2} \mathbb{E}[\|\bar{\alpha}_{t+1} - \bar{\alpha}_{t}\|^2] \notag \\
	& \leq  \eta_{\alpha} \mathbb{E}[\langle \nabla_{1} \tilde{f}_{\rho}(\bar{\alpha}_{t}, \bar{\theta}_{t}), \frac{1}{K}\sum_{k=1}^{K} {m}^{(k)}_{\alpha, t} \rangle]  + \frac{\eta^2_{\alpha}\tilde{L}}{2} \mathbb{E}[\|\frac{1}{K}\sum_{k=1}^{K} {m}^{(k)}_{\alpha, t} \|^2] +  \frac{\eta^2_{\alpha}\tilde{L}\sigma^2}{2}  \notag \\
	& \leq   \frac{\eta_{\alpha}}{2} \mathbb{E}[\| \nabla_{1} \tilde{f}_{\rho}(\bar{\alpha}_{t}, \bar{\theta}_{t})\|^2] + (\frac{\eta_{\alpha}}{2} + \frac{\eta^2_{\alpha}\tilde{L}}{2} ) \mathbb{E}[\|\frac{1}{K}\sum_{k=1}^{K} {m}^{(k)}_{\alpha, t} \|^2] +  \frac{\eta^2_{\alpha}\tilde{L}\sigma^2}{2}  \ . \notag
\end{align}

    Similarly, given that the $\tilde{L}$-smooth of $\tilde{f}_{\rho}(\alpha, \theta)$ with respect to $\theta$, we derive
\begin{align}
    & \quad \mathbb{E}[\tilde{f}_{\rho}(\bar{\alpha}_{t+1}, \bar{\theta}_{t})]  - \mathbb{E}[\tilde{f}_{\rho}(\bar{\alpha}_{t+1}, \bar{\theta}_{t+1})] \leq  - \mathbb{E}[\langle \nabla_{2} \tilde{f}_{\rho}(\bar{\alpha}_{t+1}, \bar{\theta}_{t}),  \bar{\theta}_{t+1} - \bar{\theta}_{t}\rangle] + \frac{\tilde{L}}{2}\mathbb{E}[\|\bar{\theta}_{t+1} - \bar{\theta}_{t}\|^2] \notag \\
    &  \leq  - \eta_{\theta}\mathbb{E}[\langle \nabla_{2} \tilde{f}_{\rho}(\bar{\alpha}_{t+1}, \bar{\theta}_{t}),  \frac{1}{K}\sum_{k=1}^{K} {m}^{(k)}_{\theta, t} \rangle]  + \frac{\eta^2_{\theta}\tilde{L}}{2}\mathbb{E}[\|\frac{1}{K}\sum_{k=1}^{K} {m}^{(k)}_{\theta, t} \|^2 ] + \frac{\eta^2_{\theta}\tilde{L}\sigma^2}{2}\notag \\
    &  \leq  - \frac{\eta_{\theta}}{2}\mathbb{E}[\| \nabla_{2} \tilde{f}_{\rho}(\bar{\alpha}_{t+1}, \bar{\theta}_{t})\|^2] - \frac{\eta_{\theta}}{2}\mathbb{E}[\| \frac{1}{K}\sum_{k=1}^{K} {m}^{(k)}_{\theta, t} \|^2] +  \eta_{\theta}\mathbb{E}[\| \nabla_{2} \tilde{f}_{\rho}(\bar{\alpha}_{t}, \bar{\theta}_{t})- \frac{1}{K}\sum_{k=1}^{K}{m}^{(k)}_{\theta, t} \|^2] \notag \\
    & \quad + \eta_{\theta}\mathbb{E}[\| \nabla_{2} \tilde{f}_{\rho}(\bar{\alpha}_{t+1}, \bar{\theta}_{t})- \nabla_{2} \tilde{f}_{\rho}(\bar{\alpha}_{t}, \bar{\theta}_{t}) \|^2]  + \frac{\eta^2_{\theta}\tilde{L}}{2}\mathbb{E}[\|\frac{1}{K}\sum_{k=1}^{K} {m}^{(k)}_{\theta, t}\|^2 ] + \frac{\eta^2_{\theta}\tilde{L}\sigma^2}{2}\ . \notag
\end{align}
From $ - \| \nabla_{2} \tilde{f}_{\rho}(\bar{\alpha}_{t+1}, \bar{\theta}_{t})\|^2 \leq -\frac{1}{2}\|  \nabla_{2} \tilde{f}_{\rho}(\bar{\alpha}_{t}, \bar{\theta}_{t})\|^2 + \|\nabla_{2} \tilde{f}_{\rho}(\bar{\alpha}_{t+1}, \bar{\theta}_{t})- \nabla_{2} \tilde{f}_{\rho}(\bar{\alpha}_{t}, \bar{\theta}_{t}) \| $ and $\eta_{\theta} \leq \frac{1}{2\tilde{L}}$, we obtain
\begin{align}
     & \quad \mathbb{E}[\tilde{f}_{\rho}(\bar{\alpha}_{t+1}, \bar{\theta}_{t})]  - \mathbb{E}[\tilde{f}_{\rho}(\bar{\alpha}_{t+1}, \bar{\theta}_{t+1})] \notag \\
     &   \leq  - \frac{\eta_{\theta}}{4}\mathbb{E}[\| \nabla_{2} \tilde{f}_{\rho}(\bar{\alpha}_{t}, \bar{\theta}_{t})\|^2] - \frac{\eta_{\theta}}{4}\mathbb{E}[\| \frac{1}{K}\sum_{k=1}^{K} {m}^{(k)}_{\theta, t} \|^2]  +  \eta_{\theta}\mathbb{E}[\| \nabla_{2} \tilde{f}_{\rho}(\bar{\alpha}_{t}, \bar{\theta}_{t})- \frac{1}{K}\sum_{k=1}^{K}{m}^{(k)}_{\theta, t} \|^2]  \notag \\
    & \quad  + \frac{3\eta_{\theta}}{2}\mathbb{E}[\| \nabla_{2} \tilde{f}_{\rho}(\bar{\alpha}_{t+1}, \bar{\theta}_{t})- \nabla_{2} \tilde{f}_{\rho}(\bar{\alpha}_{t}, \bar{\theta}_{t}) \|^2] + \frac{\eta^2_{\theta}\tilde{L}\sigma^2}{2}\ . \notag
\end{align}
Then, combining 
\begin{align}
    &  \mathbb{E}[\| \nabla_{2} \tilde{f}_{\rho}(\bar{\alpha}_{t}, \bar{\theta}_{t})- \frac{1}{K}\sum_{k=1}^{K} {m}^{(k)}_{\theta, t}  \|^2] \leq  \tilde{L}^2\frac{1}{K}\sum_{k=1}^{K} \mathbb{E}[\| \bar{\alpha}_{t}-  {\alpha}^{(k)}_{t} \|^2] + \tilde{L}^2\frac{1}{K}\sum_{k=1}^{K} \mathbb{E}[\| \bar{\theta}_{t}-  {\theta}^{(k)}_{t} \|^2] \ , \notag
\end{align}
and
\begin{align}
    & \quad \mathbb{E}[\| \nabla_{2} \tilde{f}_{\rho}(\bar{\alpha}_{t+1}, \bar{\theta}_{t})- \nabla_{2} \tilde{f}_{\rho}(\bar{\alpha}_{t}, \bar{\theta}_{t}) \|^2] \leq \eta_{\alpha}^2\tilde{L}^2 \mathbb{E}[\| \frac{1}{K}\sum_{k=1}^{K} \hat{m}^{(k)}_{\alpha, t} \|^2]  \notag \\
    & \leq  \eta_{\alpha}^2\tilde{L}^2\mathbb{E}[\| \frac{1}{K}\sum_{k=1}^{K} {m}^{(k)}_{\alpha, t} \|^2] + \eta_{\alpha}^2\tilde{L}^2\sigma^2 \ , \notag
\end{align}
we obtain
\begin{align}
    &  \quad  \mathbb{E}[\tilde{f}_{\rho}(\bar{\alpha}_{t+1}, \bar{\theta}_{t})] - \mathbb{E}[\tilde{f}_{\rho}(\bar{\alpha}_{t+1}, \bar{\theta}_{t+1})] \notag \\
    & \leq  - \frac{\eta_{\theta}}{4}\mathbb{E}[\| \nabla_{2} \tilde{f}_{\rho}(\bar{\alpha}_{t}, \bar{\theta}_{t})\|^2] - \frac{\eta_{\theta}}{4}\mathbb{E}[\| \frac{1}{K}\sum_{k=1}^{K} {m}^{(k)}_{\theta, t}  \|^2] +  \frac{3\eta_{\theta}\eta_{\alpha}^2\tilde{L}^2}{2}  \mathbb{E}[\| \frac{1}{K}\sum_{k=1}^{K} {m}^{(k)}_{\alpha, t}  \|^2]  \notag \\
    & \quad  + \eta_{\theta} \tilde{L}^2\frac{1}{K}\sum_{k=1}^{K} \mathbb{E}[\| \bar{\alpha}_{t}-  {\alpha}^{(k)}_{t} \|^2]  +\eta_{\theta} \tilde{L}^2\frac{1}{K}\sum_{k=1}^{K} \mathbb{E}[\| \bar{\theta}_{t}-  {\theta}^{(k)}_{t} \|^2] + \frac{3\eta_{\theta}\eta_{\alpha}^2\tilde{L}^2\sigma^2}{2}  + \frac{\eta^2_{\theta}\tilde{L}\sigma^2}{2}  \ . \notag
\end{align}

Combining the upper bound of $\mathbb{E}[\tilde{f}_{\rho}(\bar{\alpha}_{t}, \bar{\theta}_{t})] -  \mathbb{E}[\tilde{f}_{\rho}(\bar{\alpha}_{t+1}, \bar{\theta}_{t})]$ and $\mathbb{E}[\tilde{f}_{\rho}(\bar{\alpha}_{t+1}, \bar{\theta}_{t})] - \mathbb{E}[\tilde{f}_{\rho}(\bar{\alpha}_{t+1}, \bar{\theta}_{t+1})]$  with $\eta_{\theta} \leq \frac{1}{2\tilde{L}}$, we obtain
\begin{align}\label{eq:f_smooth_res}
    &  \mathbb{E}[\tilde{f}_{\rho}(\bar{\alpha}_{t}, \bar{\theta}_{t})] -  \mathbb{E}[\tilde{f}_{\rho}(\bar{\alpha}_{t+1}, \bar{\theta}_{t+1})]  \leq \frac{\eta_{\alpha}}{2} \mathbb{E}[\| \nabla_{1} \tilde{f}_{\rho}(\bar{\alpha}_{t}, \bar{\theta}_{t})\|^2]  - \frac{\eta_{\theta}}{4}\mathbb{E}[\| \nabla_{2} \tilde{f}_{\rho}(\bar{\alpha}_{t}, \bar{\theta}_{t})\|^2]  \notag \\
    &  + (\frac{\eta_{\alpha}}{2} + \frac{5\eta_{\alpha}^2\tilde{L}}{4} )\mathbb{E}[\| \frac{1}{K}\sum_{k=1}^{K} {m}^{(k)}_{\alpha, t} \|^2] - \frac{\eta_{\theta}}{4}\mathbb{E}[\| \frac{1}{K}\sum_{k=1}^{K} {m}^{(k)}_{\theta, t} \|^2]  + \eta_{\theta} \tilde{L}^2\frac{1}{K}\sum_{k=1}^{K} \mathbb{E}[\| \bar{\alpha}_{t}-  {\alpha}^{(k)}_{t} \|^2]  \notag \\
    & +\eta_{\theta} \tilde{L}^2\frac{1}{K}\sum_{k=1}^{K} \mathbb{E}[\| \bar{\theta}_{t}-  {\theta}^{(k)}_{t} \|^2] + \frac{5\eta_{\alpha}^2\tilde{L}\sigma^2}{4}  + \frac{\eta^2_{\theta}\tilde{L}\sigma^2}{2}  \ . 
\end{align}

Moreover, using the $L_{\tilde{\Phi}}$-smoothness of $\tilde{\Phi}({\alpha})$, with $L_{\tilde{\Phi}}=O(\kappa)$ \cite{lin2020gradient}, we derive
\begin{align}\label{eq:phi_smooth_res}
	& \quad  \mathbb{E}[\tilde{\Phi}(\bar{\alpha}_{t+1})] -  \mathbb{E}[\tilde{\Phi}(\bar{\alpha}_{t})] \leq    \mathbb{E}[\langle \nabla \tilde{\Phi}(\bar{\alpha}_{t}), \bar{\alpha}_{t+1} - \bar{\alpha}_{t}  \rangle] + \frac{L_{\tilde{\Phi}}}{2} \mathbb{E}[\| \bar{\alpha}_{t+1} - \bar{\alpha}_{t} \|^2] \notag \\
	& \leq   -\eta_{\alpha}   \mathbb{E}[\langle \nabla \tilde{\Phi}(\bar{\alpha}_{t}), \frac{1}{K}\sum_{k=1}^{K} {m}^{(k)}_{\alpha, t} \rangle] + \frac{\eta^2_{\alpha}L_{\tilde{\Phi}}}{2} \mathbb{E}[\|  \frac{1}{K}\sum_{k=1}^{K}  {m}^{(k)}_{\alpha, t}\|^2]  +\frac{\eta^2_{\alpha}L_{\tilde{\Phi}}\sigma^2}{2} \notag \\
    & \leq 8 \eta_{\alpha}   \mathbb{E}[\| \nabla \tilde{\Phi}(\bar{\alpha}_{t}) -  \nabla_{1} \tilde{f}_{\rho}(\bar{\alpha}_{t}, \bar{\theta}_{t})\|^2]  +  \frac{\eta_{\alpha}}{8}   \mathbb{E}[\|  \frac{1}{K}\sum_{k=1}^{K}  {m}^{(k)}_{\alpha, t} \|^2] \notag \\
    & \quad - \frac{\eta_{\alpha}}{2}   \mathbb{E}[\|  \nabla_{1} \tilde{f}_{\rho}(\bar{\alpha}_{t}, \bar{\theta}_{t}) \|^2] -  \frac{\eta_{\alpha}}{2}   \mathbb{E}[\| \frac{1}{K}\sum_{k=1}^{K} {m}^{(k)}_{\alpha, t} \|^2]  +  \frac{\eta_{\alpha}}{2}   \mathbb{E}[\| \nabla_{1} \tilde{f}_{\rho}(\bar{\alpha}_{t}, \bar{\theta}_{t})  -\frac{1}{K}\sum_{k=1}^{K} {m}^{(k)}_{\alpha, t}  \|^2] \notag \\
    & \quad + \frac{\eta^2_{\alpha}L_{\tilde{\Phi}}}{2} \mathbb{E}[\|  \frac{1}{K}\sum_{k=1}^{K}  {m}^{(k)}_{\alpha, t}\|^2]  +\frac{\eta^2_{\alpha}L_{\tilde{\Phi}}\sigma^2}{2} \notag \\
    & \leq   8 \eta_{\alpha} \frac{\tilde{L}^2}{\mu^2}  \mathbb{E}[\|  \nabla_{2} \tilde{f}_{\rho}(\bar{\alpha}_{t}, \bar{\theta}_{t}) \|^2]  - \frac{\eta_{\alpha}}{2}   \mathbb{E}[\|  \nabla_{1} \tilde{f}_{\rho}(\bar{\alpha}_{t}, \bar{\theta}_{t}) \|^2] + ( \frac{\eta^2_{\alpha}L_{\tilde{\Phi}}}{2}  -  \frac{3\eta_{\alpha}}{8} )  \mathbb{E}[\|  \frac{1}{K}\sum_{k=1}^{K}  {m}^{(k)}_{\alpha, t} \|^2]  \notag \\
	&  \quad  +  \frac{\eta_{\alpha} \tilde{L}^2}{2}  \frac{1}{K}\sum_{k=1}^{K}  \mathbb{E}[\| \bar{\alpha}_{t}-\alpha^{(k)}_{t} \|^2]  +  \frac{\eta_{\alpha} \tilde{L}^2}{2}  \frac{1}{K}\sum_{k=1}^{K}  \mathbb{E}[\|  \bar{\theta}_{t}  - \theta^{(k)}_{t}  \|^2]  +\frac{\eta^2_{\alpha}L_{\tilde{\Phi}}\sigma^2}{2} \ ,
\end{align}
where the last step follows from that $\tilde{f}_{\rho}(\alpha, \theta)$ satisfies the $\mu$-PL condition with respect to $\theta$.

We next define a potential function of the form:
\begin{align}
	& \mathcal{L}_{t}  =  2\tilde{\Phi}(\bar{\alpha}_{t}) - \tilde{\Phi}^*  - \tilde{f}_{\rho}(\bar{\alpha}_{t}, \bar{\theta}_{t}) \ , \notag
\end{align}
By invoking Eq.~(\ref{eq:f_smooth_res}) and Eq.~(\ref{eq:phi_smooth_res}), and choosing $\eta_{\alpha} =\frac{\eta_{\theta}}{128\kappa^2}$ with $\kappa = \frac{\tilde{L}}{\mu} > 1$, and setting $\eta_{\alpha} \leq \frac{1}{2(5\tilde{L} + 4L_{\tilde{\Phi}})}$, we obtain
\begin{align}\label{eq:p_1}
	& \quad \mathcal{L}_{t+1} - \mathcal{L}_{t}  \notag \\
	& = 2(\tilde{\Phi}(\bar{\alpha}_{t+1}) -\tilde{\Phi}(\bar{\alpha}_{t}) )  +  \tilde{f}_{\rho}(\bar{\alpha}_{t}, \bar{\theta}_{t})- \tilde{f}_{\rho}(\bar{\alpha}_{t+1}, \bar{\theta}_{t+1})  \notag \\
	& \leq -  \frac{\eta_{\alpha}}{2} \mathbb{E}[\| \nabla_{1} \tilde{f}_{\rho}(\bar{\alpha}_{t}, \bar{\theta}_{t})\|^2] - \frac{\eta_{\theta}}{8}\mathbb{E}[\| \nabla_{2} \tilde{f}_{\rho}(\bar{\alpha}_{t}, \bar{\theta}_{t})\|^2] - \frac{\eta_{\alpha}}{8} \mathbb{E}[\|  \frac{1}{K}\sum_{k=1}^{K}  {m}^{(k)}_{\alpha, t} \|^2]   \notag \\
	& \quad  - \frac{\eta_{\theta}}{4}\mathbb{E}[\| \frac{1}{K}\sum_{k=1}^{K} {m}^{(k)}_{\theta, t} \|^2]  +  (\eta_{\alpha}+\eta_{\theta}) \tilde{L}^2 \frac{1}{K}\sum_{k=1}^{K}  \mathbb{E}[\| \bar{\alpha}_{t}-\alpha^{(k)}_{t} \|^2]   \notag \\
    & \quad  +  (\eta_{\alpha}+\eta_{\theta})  \tilde{L}^2 \frac{1}{K}\sum_{k=1}^{K}  \mathbb{E}[\|  \bar{\theta}_{t}  - \theta^{(k)}_{t}  \|^2] +\eta^2_{\alpha}L_{\tilde{\Phi}}\sigma^2   +  \frac{5\eta^2_{\alpha}\tilde{L}\sigma^2}{4} + \frac{\eta^2_{\theta}\tilde{L}\sigma^2}{2}  \ . 
\end{align}

Given that $\tilde{f}_{\rho}(\alpha, \theta)$ satisfies the $\mu$-PL condition with respect to $\theta$, we obtain
\begin{align}
	& \quad  \mathbb{E}[\| \nabla \tilde{\Phi}(\bar{\alpha}_{t})\|^2] \leq  2\mathbb{E}[\| \nabla \tilde{\Phi}(\bar{\alpha}_{t}) -  \nabla_{1} \tilde{f}_{\rho}(\bar{\alpha}_{t}, \bar{\theta}_{t})\|^2]   +2 \mathbb{E}[\| \nabla_{1} \tilde{f}_{\rho}(\bar{\alpha}_{t}, \bar{\theta}_{t})\|^2] \notag \\
	& \leq \frac{2\tilde{L}^2}{\mu^2}\mathbb{E}[\|\nabla_2 \tilde{f}_{\rho}(\bar{\alpha}_{t}, \bar{\theta}_{t}) \|^2]   +2 \mathbb{E}[\| \nabla_{1} \tilde{f}_{\rho}(\bar{\alpha}_{t}, \bar{\theta}_{t})\|^2] \ . \notag
\end{align}
Then, by setting $\eta_{\alpha} =\frac{\eta_{\theta}}{128\kappa^2}$ and using the fact that $\tilde{f}_{\rho}(\alpha, \theta)$ satisfies the $\mu$-PL condition with respect to $\theta$, we obtain
\begin{align}
		& \quad  \frac{\eta_{\alpha}}{2} \mathbb{E}[\| \nabla_{1} \tilde{f}_{\rho}(\bar{\alpha}_{t}, \bar{\theta}_{t})\|^2] + \frac{\eta_{\theta}}{8}\mathbb{E}[\| \nabla_{2} \tilde{f}_{\rho}(\bar{\alpha}_{t}, \bar{\theta}_{t})\|^2]  \notag \\
		& \geq \frac{\eta_{\alpha}}{4}\mathbb{E}[\| \nabla \tilde{\Phi}(\bar{\alpha}_{t})\|^2] - \frac{\eta_{\alpha}}{2} \frac{\tilde{L}^2}{\mu^2}\mathbb{E}[\|\nabla_2 \tilde{f}_{\rho}(\bar{\alpha}_{t}, \bar{\theta}_{t}) \|^2]  + \frac{\eta_{\theta}}{8}\mathbb{E}[\| \nabla_{2} \tilde{f}_{\rho}(\bar{\alpha}_{t}, \bar{\theta}_{t})\|^2]  \notag \\
		& \geq \frac{\eta_{\alpha}}{4}\mathbb{E}[\| \nabla \tilde{\Phi}(\bar{\alpha}_{t})\|^2] + \frac{\eta_{\theta}}{16}\mathbb{E}[\| \nabla_{2} \tilde{f}_{\rho}(\bar{\alpha}_{t}, \bar{\theta}_{t})\|^2]  \notag \\
		& \geq \frac{\eta_{\alpha}\mu}{2} (\tilde{\Phi}(\bar{\alpha}_{t}) - \tilde{\Phi}^*) + \frac{\eta_{\theta}\mu}{8} (\tilde{\Phi}({\alpha}_{t})  - \tilde{f}_{\rho}(\bar{\alpha}_{t}, \bar{\theta}_{t}))\notag \\
		& \geq  \frac{\eta_{\alpha}\mu}{2} (\tilde{\Phi}(\bar{\alpha}_{t}) - \tilde{\Phi}^*) + \frac{\eta_{\alpha}\mu}{2} (\tilde{\Phi}({\alpha}_{t})  - \tilde{f}_{\rho}(\bar{\alpha}_{t}, \bar{\theta}_{t}))\notag \\
		& = \frac{\eta_{\alpha}\mu}{2}\mathcal{L}_{t} \ . \notag
\end{align}

Rearranging Eq.~(\ref{eq:p_1}) yields
\begin{align}
    \mathcal{L}_{t+1} & \leq (1 - \frac{\eta_{\alpha}\mu}{2}) \mathcal{L}_{t} +\eta^2_{\alpha}L_{\tilde{\Phi}}\sigma^2   +  \frac{5\eta^2_{\alpha}\tilde{L}\sigma^2}{4} + \frac{\eta^2_{\theta}\tilde{L}\sigma^2}{2} \notag \\
    & \quad +  (\eta_{\alpha}+\eta_{\theta}) \tilde{L}^2 \Big( \frac{1}{K}\sum_{k=1}^{K}  \mathbb{E}[\| \bar{\alpha}_{t}-\alpha^{(k)}_{t} \|^2]  +  \frac{1}{K}\sum_{k=1}^{K}  \mathbb{E}[\|  \bar{\theta}_{t}  - \theta^{(k)}_{t}  \|^2] \Big) \notag \\
    & \leq (1 - \frac{\eta_{\alpha}\mu}{2})^{t+1} \mathcal{L}_{0} + \frac{2\eta_{\alpha}L_{\tilde{\Phi}}\sigma^2}{\mu}   +  \frac{5\eta_{\alpha}\tilde{L}\sigma^2}{2\mu} + \frac{\eta^2_{\theta}\tilde{L}\sigma^2}{\eta_{\alpha}\mu} \notag \\
    & \quad +  \frac{2(\eta_{\alpha}+\eta_{\theta}) }{\eta_{\alpha}\mu}\tilde{L}^2 \sum_{t'=0}^{t}\Big( \frac{1}{K}\sum_{k=1}^{K}  \mathbb{E}[\| \bar{\alpha}_{t'}-\alpha^{(k)}_{t'} \|^2]  +  \frac{1}{K}\sum_{k=1}^{K}  \mathbb{E}[\|  \bar{\theta}_{t'}  - \theta^{(k)}_{t'}  \|^2] \Big)  \  . \notag 
\end{align}

For the final term, we begin by establishing a bound on the consensus error associated with $\alpha$:
\begin{align}
    & \quad \frac{1}{K}\sum_{k=1}^{K}  \mathbb{E}[\| \bar{\alpha}_{t}-\alpha^{(k)}_{t} \|^2]  = \frac{1}{K}\sum_{k=1}^{K}  \mathbb{E}[\| \bar{\alpha}_{s_{t}\tau} - \eta_{\alpha}\sum_{t'=s_{t}\tau}^{t-1}\frac{1}{K}\sum_{k'=1}^{K} \hat{m}^{(k')}_{\alpha, t'} -(\alpha^{(k)}_{s_t\tau}  - \eta_{\alpha}\sum_{t'=s_{t}\tau}^{t-1}\hat{m}^{(k)}_{\alpha, t'}) \|^2] \notag \\
    & = \tau\eta_{\alpha}^2\frac{1}{K}\sum_{k=1}^{K}\sum_{t'=s_{t}\tau}^{t-1}  \mathbb{E}[\|  \frac{1}{K}\sum_{k'=1}^{K}\hat{m}^{(k')}_{\alpha, t'} -  \hat{m}^{(k)}_{\alpha, t'}  \|^2] \notag \\
    & = \tau\eta_{\alpha}^2\frac{1}{K}\sum_{k=1}^{K}\sum_{t'=s_{t}\tau}^{t-1}  \mathbb{E}[\|  \frac{1}{K}\sum_{k'=1}^{K}\hat{m}^{(k')}_{\alpha, t'} - \frac{1}{K}\sum_{k'=1}^{K}{m}^{(k')}_{\alpha, t'} +  \frac{1}{K}\sum_{k'=1}^{K}{m}^{(k')}_{\alpha, t'} -  {m}^{(k)}_{\alpha, t'} + {m}^{(k)}_{\alpha, t'}  -  \hat{m}^{(k)}_{\alpha, t'}  \|^2] \notag \\
    & \leq \tau\eta_{\alpha}^2\frac{1}{K}\sum_{k=1}^{K}\sum_{t'=s_{t}\tau}^{t-1}  \mathbb{E}[\|  \frac{1}{K}\sum_{k'=1}^{K}\hat{m}^{(k')}_{\alpha, t'} - \frac{1}{K}\sum_{k'=1}^{K}{m}^{(k')}_{\alpha, t'}  \|^2] +  \tau\eta_{\alpha}^2\frac{1}{K}\sum_{k=1}^{K}\sum_{t'=s_{t}\tau}^{t-1}  \mathbb{E}[\| {m}^{(k)}_{\alpha, t'} -  \hat{m}^{(k)}_{\alpha, t'} \|^2] \notag \\
    & \quad +3\tau\eta_{\alpha}^2\frac{1}{K}\sum_{k=1}^{K}\sum_{t'=s_{t}\tau}^{t-1}  \mathbb{E}[\| \frac{1}{K}\sum_{k'=1}^{K}{m}^{(k')}_{\alpha, t'} - \frac{1}{K}\sum_{k'=1}^{K}\nabla_{1}\tilde{f}_{\rho}^{(k')}(\bar{\alpha}_{t'}, \bar{\theta}_{t'})  \|^2] \notag \\
    & \quad +3\tau\eta_{\alpha}^2\frac{1}{K}\sum_{k=1}^{K}\sum_{t'=s_{t}\tau}^{t-1}  \mathbb{E}[\| \frac{1}{K}\sum_{k'=1}^{K}\nabla_{1}\tilde{f}_{\rho}^{(k')}(\bar{\alpha}_{t'}, \bar{\theta}_{t'}) -  \nabla_{1}\tilde{f}_{\rho}^{(k)}(\bar{\alpha}_{t'}, \bar{\theta}_{t'})  \|^2] \notag \\
    & \quad  + 3\tau\eta_{\alpha}^2\frac{1}{K}\sum_{k=1}^{K}\sum_{t'=s_{t}\tau}^{t-1}  \mathbb{E}[\| \nabla_{1}\tilde{f}_{\rho}^{(k)}(\bar{\alpha}_{t'}, \bar{\theta}_{t'}) -  {m}^{(k)}_{\alpha, t'}   \|^2]  \notag  \\
    & \leq \sum_{t'=s_{t}\tau}^{t-1}2\tau\eta_{\alpha}^2\sigma^2 + \sum_{t'=s_{t}\tau}^{t-1}3\tau\eta_{\alpha}^2\delta^2 + 6\tau\eta_{\alpha}^2\tilde{L}^2\frac{1}{K}\sum_{k=1}^{K}\sum_{t'=s_{t}\tau}^{t-1}\Big(  \mathbb{E}[\| \bar{\alpha}_{t'}-\alpha^{(k)}_{t'} \|^2] + \mathbb{E}[\| \bar{\theta}_{t'}-\theta^{(k)}_{t'} \|^2] \Big) \notag 
\end{align}
where the second step holds due to $\bar{\alpha}_{s_{t}\tau} = \alpha^{(k)}_{s_t\tau}$, the last step holds due to Assumptions~\ref{assumption:variance}-\ref{assumption:heterogeneous}. 
Similarly, we obtain a bound on the consensus error associated with $\theta$:
\begin{align}
    &  \frac{1}{K}\sum_{k=1}^{K}  \mathbb{E}[\| \bar{\theta}_{t}-\theta^{(k)}_{t} \|^2] \leq 2\sum_{t'=s_{t}\tau}^{t-1}\tau\eta_{\theta}^2\sigma^2 + 3\sum_{t'=s_{t}\tau}^{t-1}\tau\eta_{\theta}^2\sigma^2 \notag \\
    & + 6\tau\eta_{\theta}^2\tilde{L}^2\frac{1}{K}\sum_{k=1}^{K}\sum_{t'=s_{t}\tau}^{t-1}\Big(  \mathbb{E}[\| \bar{\alpha}_{t'}-\alpha^{(k)}_{t'} \|^2] + \mathbb{E}[\| \bar{\theta}_{t'}-\theta^{(k)}_{t'} \|^2] \Big) \ . \notag 
\end{align}
Combining together and by summing up $t'$ from $s_t\tau$ to $(s_t+1)\tau -1$, we derive
\begin{align}
    &  \sum_{t'=s_{t}\tau}^{(s_t+1)\tau-1} \frac{1}{K}\sum_{k=1}^{K}  \Big( \mathbb{E}[\| \bar{\alpha}_{t}-\alpha^{(k)}_{t} \|^2] + \mathbb{E}[\| \bar{\theta}_{t}-\theta^{(k)}_{t} \|^2] \Big) \leq \sum_{t'=s_{t}\tau}^{(s_t+1)\tau-1} \tau(\eta_{\alpha}^2 + \eta_{\theta}^2)(2\sigma^2 + 3\delta^2)\notag \\
    & +  6\tau^2(\eta_{\alpha}^2+ \eta_{\theta}^2)\tilde{L}^2\frac{1}{K}\sum_{k=1}^{K}\sum_{t'=s_{t}\tau}^{(s_t+1)\tau-1}\Big(  \mathbb{E}[\| \bar{\alpha}_{t'}-\alpha^{(k)}_{t'} \|^2] + \mathbb{E}[\| \bar{\theta}_{t'}-\theta^{(k)}_{t'} \|^2] \Big) \ . \notag 
\end{align}
Setting $\eta_{\alpha} \leq \frac{1}{2\sqrt{6}\tau \tilde{L}}$, $\eta_{\alpha} \leq \frac{1}{2\sqrt{6}\tau \tilde{L}}$ yields
\begin{align}
    \sum_{t=0}^{T-1}\frac{1}{K}\sum_{k=1}^{K}  \Big( \mathbb{E}[\| \bar{\alpha}_{t}-\alpha^{(k)}_{t} \|^2] + \mathbb{E}[\| \bar{\theta}_{t}-\theta^{(k)}_{t} \|^2] \Big) \leq \tau^2(\eta_{\alpha}^2 + \eta_{\theta}^2)(4\sigma^2 + 6\sigma^2) \ . \notag 
\end{align}
Plugging this upper bound to above inequality, we obtain
\begin{align}
     & \quad \mathcal{L}_{T}  \leq (1 - \frac{\eta_{\alpha}\mu}{2})^{T} \mathcal{L}_{0} + \frac{2\eta_{\alpha}L_{\tilde{\Phi}}\sigma^2}{\mu}   +  \frac{5\eta_{\alpha}\tilde{L}\sigma^2}{2\mu} + \frac{\eta^2_{\theta}\tilde{L}\sigma^2}{\eta_{\alpha}\mu} \notag \\
    & \qquad \quad +  \frac{2}{\mu}(1+\frac{\eta_{\theta}}{\eta_{\alpha}}) \tilde{L}^2 \sum_{t'=0}^{T-1}\Big( \frac{1}{K}\sum_{k=1}^{K}  \mathbb{E}[\| \bar{\alpha}_{t'}-\alpha^{(k)}_{t'} \|^2]  +  \frac{1}{K}\sum_{k=1}^{K}  \mathbb{E}[\|  \bar{\theta}_{t'}  - \theta^{(k)}_{t'}  \|^2] \Big) \notag \\ 
    & \leq (1 - \frac{\eta_{\alpha}\mu}{2})^{T} \mathcal{L}_{0} + \frac{2\eta_{\alpha}L_{\tilde{\Phi}}\sigma^2}{\mu}   +  \frac{5\eta_{\alpha}\tilde{L}\sigma^2}{2\mu} + \frac{\eta^2_{\theta}\tilde{L}\sigma^2}{\eta_{\alpha}\mu} +  (1+\frac{\eta_{\theta}}{\eta_{\alpha}}) \tilde{L}^2 \frac{2\tau^2(\eta_{\alpha}^2 + \eta_{\theta}^2)(4\sigma^2+6\delta^2)}{\mu}  \ . \notag
\end{align}
Following Eq.~(9) and its following inequality in \cite{yang2022faster}, we can obtain
\begin{align}
	&  \mathcal{L}_{0}\leq  \Tilde{\Phi}^*  - \Tilde{\Psi}(\tilde{\theta})  + \frac{1}{2{L}}\|\nabla_{1} {f}(\tilde{\alpha}, \tilde{\theta}) \|^2 \leq \Tilde{\Phi}^*  - \Tilde{\Psi}(\tilde{\theta})  + \frac{\epsilon^2}{2{L}} \ , \notag \\
    &   \leq  \frac{1}{2\mu} \|\nabla_{2} {f}(\tilde{\alpha}, \tilde{\theta})\|^2 +\frac{1}{2{L}} \|\nabla_{1} {f}(\tilde{\alpha}, \tilde{\theta}) \|^2 + \frac{\epsilon^2}{2{L}} \leq  \frac{\epsilon^2}{{L}} + \frac{\epsilon^2}{2{L}} \leq \frac{2\epsilon^2}{L} \ . \notag
\end{align}

By putting this upper bound into that of $\mathcal{L}_{T}$, we have
\begin{align}
     \mathcal{L}_{T}  & \leq  (1 - \frac{\eta_{\alpha}\mu}{2})^{T}\frac{2\epsilon^2}{L} + \frac{2\eta_{\alpha}L_{\tilde{\Phi}}\sigma^2}{\mu}   +  \frac{5\eta_{\alpha}\tilde{L}\sigma^2}{2\mu} + \frac{\eta^2_{\theta}\tilde{L}\sigma^2}{\eta_{\alpha}\mu} \notag \\
    & \quad +  (1+\frac{\eta_{\theta}}{\eta_{\alpha}}) \tilde{L}^2 \frac{2\tau^2(\eta_{\alpha}^2 + \eta_{\theta}^2)(4\sigma^2+6\delta^2)}{\mu} \ .  \notag
\end{align}
Following the argument in the proof of Proposition 2.1 in \cite{yang2022faster} and choosing $\eta_{\alpha} = O(\eta_{\theta}/\kappa^2)$, one can establish that
\begin{align}
    & \mathbb{E}[\|\alpha_{T} -\alpha^*\|^2]\leq O(\kappa) \mathcal{L}_{T}\leq O(\kappa\epsilon^2)(1-\frac{\eta_{\alpha}\mu}{2})^{T}  +O(\kappa^5\eta_{\alpha})  +O(\kappa^7\eta_{\alpha}^2)  \ .  \notag
\end{align}
Finally, when $\eta_{\alpha} = O(\epsilon^2/\kappa^5)$ and $T = O(\epsilon^{-2}\kappa^6 \log(\kappa))$, we obtain
\begin{align}\label{eq:convergence-rate-pl}
	& \mathbb{E}[\|\alpha_{T} -\alpha^*\|^2]\leq  O(\epsilon^2) \  .
\end{align}
This completes the proof.
\end{proof}

Based on Theorem~\ref{theorem:two-sided-pl}, we combine Theorem~\ref{theorem:main} and Theorem~\ref{theorem:two-sided-pl} to obtain an $O(\epsilon)$-stationary guarantee for the original objective. Defining $\Phi_{1/2L}(\tilde{\alpha}) = \min_{\beta}\Phi(\beta) + L\|\beta-\tilde{\alpha}\|^2$, using the argument of Proposition 2.1 b) in \cite{yang2022faster}, one obtains
\begin{align}
     \|\nabla \Phi(\hat{\alpha})\|&  \leq \|\nabla \Phi_{1/2L}(\tilde{\alpha})\| + 2\kappa L \|\hat{\alpha} - \alpha^*(\tilde{\alpha})\| \notag \\
    & \leq \underbrace{4\|\nabla_1 f_{\rho}(\tilde{\alpha}, \tilde{\theta})\| +  4\kappa\|\nabla_2 f_{\rho}(\tilde{\alpha}, \tilde{\theta})\|}_{\leq O(\epsilon)  \text{\ follows from Theorem~\ref{theorem:main}}} + \underbrace{2\kappa L \|\hat{\alpha} - \alpha^*(\tilde{\alpha})\|}_{\leq O(\epsilon)  \text{\ follows from Theorem~\ref{theorem:two-sided-pl}}}  \ . 
\end{align}
Here $(\tilde{\alpha}, \tilde{\theta}) \triangleq (\tilde{x}, \tilde{y}, \tilde{z})$ is the output of Algorithm~\ref{alg:fedss-bvrfo}, while $\hat{\alpha}$ is obtained by running LocalSGDA initialized at $\tilde{\alpha}$. Then, it is easy to conclude that  $\|\nabla \Phi(\hat{\alpha})\|\leq O(\epsilon)$. 


Finally, we complete the Remark~\ref{remark:tran}. The main algorithm requires $T = O\big(\frac{\kappa^{15/2}}{K\epsilon^5}\big)$ iterations, whereas the translation phase incurs $T = O\big(\epsilon^{-2}\kappa^6 \log(\kappa)\big)$. After comparing these two number of iterations, it is easy to conclude that the overall complexity for obtaining an $O(\epsilon)$-stationary solution is $T = O\big(\frac{\kappa^{15/2}}{K\epsilon^5}\big)$.



\end{document}